\documentclass[a4paper]{article}

\usepackage[english]{babel}

\usepackage[utf8]{inputenc}
\usepackage[utf8]{inputenc}
\usepackage{amsmath}
\usepackage{graphicx}
\usepackage{amssymb}
\usepackage{amsthm}
\usepackage{tikz-cd}
\usepackage{mathrsfs}
\usepackage[colorinlistoftodos]{todonotes}
\usepackage{enumitem}
\usepackage{yfonts}
\usepackage{ dsfont }
\usepackage{MnSymbol}
\usepackage{slashed}
\usepackage{bbm}
\usepackage[dvipsnames]{xcolor}

\title{{A Variational Analysis of Kernel Learning with Learnable Linear Transformations}}

\author{Yang Li\footnote{Cambridge University, Department of Pure Mathematics and Mathematical Statistics},~~~Feng Ruan\footnote{Northwestern University, Department of Statistics and Data Science}}

\date{\today}
\newtheorem{thm}{Theorem}[section]
\newtheorem{lem}[thm]{Lemma}

\theoremstyle{definition}
\newtheorem{eg}[thm]{Example}

\newtheorem{cor}[thm]{Corollary}

\newtheorem{rmk}{Remark}
\newtheorem{prop}[thm]{Proposition}
\newtheorem{Def}[thm]{Definition}
\newtheorem{Question}{Question}
\newtheorem*{Notation}{Notation}

\newtheorem*{Acknowledgement}{Acknowledgement}

\newcommand{\ie}{\emph{i.e.} }
\newcommand{\cf}{\emph{cf.} }

\newcommand{\R}{\mathbb{R}}
\newcommand{\C}{\mathbb{C}}

\newcommand{\norm}[1]{\left\lVert#1\right\rVert}

\newcommand{\E}{\mathbb{E}}

\newcommand{\1}{\mathbbm{1}}

\DeclareMathOperator{\Hom}{Hom}
\DeclareMathOperator{\End}{End}

\newcommand{\red}[1]{{\color{black}{#1}}}
\newcommand{\blue}[1]{{\color{black}{#1}}}
\begin{document}
	\maketitle

	\abstract{The classical kernel ridge regression problem aims to find the best fit for the output $Y$ as a function of the input data $X\in \R^d$, with a fixed choice of regularization term imposed by a given choice of a reproducing kernel Hilbert space, such as a Sobolev space. Here we consider a generalization of the kernel ridge regression problem, by introducing an extra matrix parameter $U$, which aims to detect the scale parameters and the feature variables in the data, and thereby improve the efficiency of kernel ridge regression. This naturally leads to a nonlinear variational problem to optimize the choice of $U$. 
		
		\blue{We study various foundational mathematical aspects of this variational problem, including its Euler-Lagrange equation, continuity and first variation, limiting behavior under degenerate or diverging transformations, and the structure of its local minimizers. Particular attention is given to two data-distribution settings, namely multi-scale and multi-index models, where the learned transformation $U$ encodes intrinsic scale parameters and the essential low-dimensional feature variables, respectively.}

		\vspace{.5em} 
		
		
		\vspace{.5em} 
		
		
		\vspace{.5em}

		\section{Introduction}

		\subsection{Kernel learning problem}

		The classical version of \emph{kernel ridge regression} \cite{Smale}\cite[Chapter 7]{Bach}\cite{Wahba90} is the following problem.
		Let $X$ be a random variable valued in a Euclidean space $V\simeq \R^d$, and let $Y$ be a random variable with $\E[|Y|^2]<+\infty$. One aims to find the minimizer $f$ of the loss function
		\[
		\frac{1}{2} \E[|Y-f(X)|^2 ]+ \frac{\lambda}{2} \norm{f}_{\mathcal{H}}^2,
		\]
		where $0<\lambda<1$ is a fixed small parameter, and $\mathcal{H}$ is a fixed a priori choice of \emph{reproducing kernel Hilbert space} (RKHS) induced by a kernel function $K$  \cite{Aronszajn}. The idea is that we wish to find the optimal function $f(X)$ which fits $Y$, among a given nice class of functions. The expectation term measures the accuracy of approximation, while the regularization term $\frac{\lambda}{2} \norm{f}_{\mathcal{H}}^2$ imposes a certain amount of smoothness, which helps to achieve good generalization behavior for new test data.

		The choice of $\mathcal{H}$ reflects an \emph{a priori} guess on the function class for $f$, and the effectiveness of this method depends crucially on selecting an appropriate kernel \cite{ScholkopfSm18}. In classical kernel ridge regression this kernel $K$ is fixed in advance, which is problematic when the relevant structure of the data is not known beforehand. This limitation has motivated \emph{kernel learning}, in which one searches over a parametrized family of kernels $\{K_\theta\}$ to better adapt to the data distribution \cite{LanckrietCrBaGhJo04}
		\cite{FukumizuBaJo09, Radhakrishnan24, ZhuDaDrFa25, CLLR, HuangLaDePoRo25} (see Section \ref{sec:comparison-to-prior-art} below).

		This paper studies the special case of 
		optimizing among the family of kernels
		$
		K_U(x,x') = K(Ux, Ux'),
		$
		parametrized by a learnable linear transformation 
		\[
		\red{ U\in \End(V):= \{  U: V\to V \text{ linear map}  \},}
		\]
		from the \emph{variational calculus} perspective. Equivalently, for each $U \in \End(V)$ 
		we consider
		\[
		I(f,U,\lambda)
		= \frac{1}{2} \E\big[ |Y - f(UX)|^2 \big]
		+ \frac{\lambda}{2} \norm{f}_{\mathcal{H}}^2,
		\]
		minimized over $f \in \mathcal{H}$, and we denote the minimal value by
		\[
		J(U,\lambda)= \min_{f\in \mathcal{H}} I(f,U,\lambda).
		\]
		Although kernel ridge regression is linear in $f$ for fixed $U$, the dependence of $J(U,\lambda)$ on $U$ is highly nonlinear. Consequently, optimizing over $U$ leads to a nonlinear variational problem whose structure is far from obvious. Our goal is to analyze this dependence of $U$, and in particular, how to \emph{optimize} the choice of $U$, thereby shedding light on which $U$ are preferred by the objective $J$, and how they extract latent structure in the data. 
		
		A particularly nice case is when the RKHS $\mathcal{H}$ is rotationally invariant.
		In this case, $J(U,\lambda)$ depends only on the inner product $\Sigma=U^TU$, and the kernel ridge regression with parameter admits an equivalent formulation \red{(\cf Section \ref{sect:rotationallyinvkernels} below)}. One minimizes the loss function
		\[
		\mathcal{I}(F, \Sigma,\lambda)=  \frac{1}{2} \E [|Y-F(X)|^2] + \frac{\lambda}{2} \norm{F}_{\mathcal{H}_\Sigma}^2,
		\]
		where \emph{the choice of the Hilbert norm varies with $\Sigma$}. Then
		\[
		J(U,\lambda)= \mathcal{J}(\Sigma;\lambda):=   \min_{F\in \mathcal{H}_\Sigma} \mathcal{I}(F,\Sigma,\lambda).
		\]
		We denote $Sym^2_+$ as the space of inner products on $V$, which is an open subset of the space of positive semi-definite forms  $Sym^2_{\geq 0}$, and we refer to the local minimizers $\Sigma$ of $\mathcal{J}$ as \emph{vacua} in analogy with physics, where the vacua sometimes refer to local energy minimizing configurations. These vacua correspond to representations favored by the model. One of our principal objectives is to understand the landscape of vacua.

		\begin{rmk}
			Recent works in machine learning have further drawn connections between analogous kernel learning formulations and two-layer neural network models, see, e.g.,
			\cite{Radhakrishnan24, ZhuDaDrFa25, CLLR, FollainBa24, HuangLaDePoRo25}. While our analysis does not rely on this perspective, the variational framework developed here is compatible with such viewpoints. For readers interested in discussion of these connections—and for an account of how our results interact with this perspective—we refer the reader to the postscript of the companion paper ``Gradient Flow in the Kernel Learning Problem''~\cite{RuanLi}. Readers interested in these aspects are strongly encouraged to consult those discussion immediately after reading this introduction.
		\end{rmk}

		\subsection{Scale adaptation and variable selection}\label{sect:Motivation}
		We now present a sequence of examples highlighting two basic mechanisms in the kernel learning problem: \emph{scale adaptation} and \emph{variable selection}.

		\blue{To make the role of scale adaptation explicit, we begin by illustrating the intrinsic tension between approximation and regularization in a \emph{fixed} RKHS, thereby motivating the need to learn an extra parameter $U$.}

		\begin{eg}
			
			Let $X\sim N(0,\sigma^2)$ be a normal variable with variance $\sigma^2$, and $Y= \sin (\sigma^{-1} X)$.
			We first consider the classical kernel ridge regression problem, where $\mathcal{H}$ is a fixed Sobolev space $H_s$ with $s>\frac{1}{2}$. 
			
			If we want to achieve good approximation, then $f\approx \sin(x/\sigma)$  for $|x|\lesssim \sigma$. Now if $\sigma$ is very large, then this implies that the $L^2$-norm of $f$ is very large, of the order $\sigma^{1/2}$. For a fixed Sobolev space, $\norm{f}_{L^2}\lesssim \norm{f}_{\mathcal{H}}$, so  $\norm{f}_{\mathcal{H}} \gtrsim \sigma^{1/2}$ is also very large. For the minimizer $f$ of the kernel ridge regression problem, 
			\[
			\frac{1}{2} \E[|Y-f(X)|^2] + \frac{\lambda}{2} \norm{f}^2_\mathcal{H} \leq \frac{1}{2}\E[|Y|^2] \lesssim 1,
			\]
			so this can only happen when $\lambda\lesssim\sigma^{-1}$.

			If on the other hand $\sigma$ is very small, then $\sin(x/\sigma)$ has very high Fourier frequency $|\omega|\sim \sigma^{-1}$, so if the minimizer $f\approx \sin(x/\sigma)$ on $|x|\lesssim \sigma$, we will have
			\[
			\norm{f}_{\mathcal{H}}^2\sim \int (1+  |\omega|^2)^s |\hat{f}|^2 d\omega \gtrsim  \sigma^{-2s} \norm{f}_{L^2}^2 \sim \sigma^{1-2s}. 
			\]
			This would force $\lambda\lesssim \sigma^{2s-1}$. The upshot is that as long as $\sigma$ is too big or too small, then for fixed regularization parameter $\lambda$, it would be impossible to achieve good approximation.

			On the other hand, if we \emph{vary the choice of the scale parameter} $U$, or  equivalently from the viewpoint of $\Sigma=U^T U$, we \emph{vary the choice of the Hilbert inner product}, then for $U\sim \sigma^{-1}$, we can simply take $f= \sin(x)$, which would give a good approximation, at no major cost to the regularization parameter $\lambda$. The upshot is that \emph{in order to break the approximation-regularization trade-off, we should aim to learn the appropriate scale parameter, or equivalently to learn the appropriate Hilbert norm}. Such a phenomenon offers the first motivation of the kernel learning problem, which is the hope that the scale parameters are reflected in the global minima or vacua of $\mathcal{J}$. 
		\end{eg}

		The previous example involves a single characteristic scale $\sigma$. A natural question is what happens when the data
		exhibits several distinct scales. In such cases, a fixed RKHS faces an even sharper tension: no single choice of norm
		can simultaneously resolve all scales, and any learned representation must implicitly decide which scales to emphasize.
		We record two models illustrating this multi-scale structure.

		\begin{eg}\label{eg:twoscale}
			(Two scale problem) Let $V=\R$, and $X_0\in \{ 0,1,\ldots m \} $ be a discrete variable, which is \emph{not directly observable}, and we prescribe the probabilities $\mathbb{P}(X_0=i)>0$. We have $m$ distinct points $a_1,\ldots a_m\in \R$. Let $X\sim N(0,1)$ conditional on $X_0=0$, and $X\sim N(a_i, \sigma^2)$ conditional on $X_0=i$ for $i=1,\ldots m$, where $0<\sigma\ll 1$. 
			Let
			\[
			Y= \sin(X)   + \sum_1^m \sin( \frac{X-a_i}{\sigma})  e^{-|X-a_i|^2/\sigma^2}. 
			\]
			In this problem there are two inherent scales, of order $1,\sigma$ respectively.  
		\end{eg}

		\begin{eg}\label{eg:multiscale}
			(Multi-scaled problem)
			Let $V=\R$, and $X_0\in \{ 1,\ldots m \} $ be a discrete variable, which is \emph{not directly accessible to the observer}, and we prescribe the probabilities $\mathbb{P}(X_0=i)>0$. Let  $a_1,\ldots a_m\in \R$ be distinct points. Let  $X\sim N(a_i, \sigma_i^2)$ conditional on $X_0=i$ for $i=1,\ldots m$, where  $0<\sigma_1\ll \ldots \ll \sigma_m\lesssim 1$ have very different orders of magnitudes. 
			Let
			\[
			Y=  \sum_1^m \sin( \frac{X-a_i}{\sigma_i})  e^{-|X-a_i|^2/\sigma_i^2}. 
			\]
			The task is to learn $Y$ based on observations of $(X,Y)$, without observing $X_0$. In this problem we have $m$ very distinct inherent scales, of order $\sigma_1,\ldots \sigma_m$.  \blue{For these multi-scale problems, the central question is to understand the \emph{mechanism of scale adaptation}: in what sense do the minimizers $U$ of the kernel learning objective reflect the intrinsic scales of the data.}
		\end{eg}

		Finally, we turn to the mechanism of \emph{variable selection}.
		
		\begin{eg}
			\label{eg:variable-selection}
			(Variable selection and multi-index problem)
			Suppose $V=W\oplus W'$ is high dimensional, while the subspace $W$ is low dimensional. We write $X=(X_W,X_{W'})$, and assume that $X_{W'}$ is independent of both $X_W$ and $Y$. Then the prediction follows a multi-index model~\cite{CookLi02}, namely,
			\begin{equation*}
				\E[Y|X] = \E[Y|X_W].
			\end{equation*}
			One then would hope to fit $Y$ by $f(UX)$, where the image of $U\in \End(V)$ lands in $W$, so that $U$ captures the essential degrees of freedom $X_W$ for the prediction of $Y$. If so, one would also expect $U$ to be low rank.
		\end{eg}

		In light of these examples, it is natural to seek a complete description of the \emph{vacua} $\Sigma$ (resp. $U$). \blue{Of particular interest are those vacua whose associated RKHS $\mathcal{H}_\Sigma$ are adapted to the intrinsic scales of the data, as in the multi-scale problem, and to its predictive feature variables, as in the multi-index problem.} A fundamental challenge is that $\mathcal{J}(\Sigma;\lambda)$ (resp. $J(U,\lambda)$) is in general \emph{non-convex} in $\Sigma$ (resp. $U$), so that one is confronted with a rich landscape of stationary points rather than a single canonical minimizer. Thus one is faced with two main problems:
		
		\begin{itemize}
			\item The \emph{static problem} is to understand how $\mathcal{J}(\Sigma;\lambda)$ depends on $\Sigma$, and in particular to understand the landscape of the vacua under a priori structural assumptions about the distribution of $X,Y$. As $\mathcal{J}$ is non-convex, in general the vacua are not unique. The main contribution of this paper is to make a foundational study on this static problem. 

			\item  The \emph{dynamical problem} is to design a gradient flow on the space $Sym^2_+$ of $\Sigma$ which converges to the vacua of $\mathcal{J}$. This depends on the choice of a Riemannian metric on $Sym^2_+$, and we will identify a canonical choice in the companion paper \cite{RuanLi} with a number of  good properties, such as subsequential convergence to stationary points, and monotone reduction of Gaussian noise effect. On the other hand, there are still many challenges ahead to make this into an efficient numerical optimization scheme.
		\end{itemize}

		\subsection{Overview of this paper}
		\label{sec:overview}

		For the sake of exposition we focus on the rotationally invariant case, even though many results in this paper apply to more general RKHSs. This paper studies how $\mathcal{J}(\Sigma;\lambda)$ depends on $\Sigma$. In particular, we are interested in the following questions:

		\begin{Question}
			What is the limiting behavior of $\mathcal{J}(\Sigma;\lambda)$ when $\Sigma$ becomes degenerate, or when $\Sigma$ tends to infinity in the space of inner products? More formally speaking, we want to find a natural \emph{partial compactification}, \red{namely adding in certain limit points for $\Sigma$}, such that $\mathcal{J}$ extends to a continuous functional. 
			
			\red{The degenerate $\Sigma$ turns out to be related to variable selection or dimension redction, whereas $\Sigma\to +\infty$ is related to scale detection.}
			
		\end{Question}

		\begin{Question}
			How do the vacua of $\mathcal{J}(\Sigma;\lambda)$ reflect the structural properties of the distribution of $(X,Y)$, such as the scale parameters and the essential degrees of freedom of how $Y$ depends on $X$? When should there be several vacua of $\mathcal{J}$?  
		\end{Question}

		\begin{rmk} 
			The kernel learning problem is analogous to the 2-layer neural network: the function $f$ is analogous to the outer layer weights, and minimizing with respect to $f$ for fixed $U$ (resp. minimizing with respect to the outer layer weights for fixed inner layer weights) amounts to a linear regression problem; the parameter $U$ (or $\Sigma$) is analogous to the inner layer weights, and the minimization in $U$ or $\Sigma$ (resp. the inner layer weights) is a genuinely nonlinear problem. We can idealize the 2-layer neural network training as a two time-scale process, with a faster time scale which minimizes the loss function with respect to the outer layer weights while fixing the inner layer weights, and a slower time scale which minimizes with respect to the inner layer weights. The analogous problem is to minimize $I(f,U,\lambda)$ first with respect to $f$, so we obtain $J(U,\lambda)$ (or equivalently $\mathcal{J}(\Sigma;\lambda)$), and then finding the local minimizers(=vacua) of $J(U,\lambda)$ with respect to $U$. 
			
			To summarize, our main question is the analogue of describing the \emph{long time limit} of the 2-layer neural network training process (with \emph{infinite sample data}, so the difference between finite sample empirical expectation and the population level expectation can be ignored), and in particular about \emph{how the optimal inner layer weights reflect properties of the data distribution} $(X,Y)$. This question is significantly different from, and complementary to, the  statistical question of `how many finite samples does one need to achieve a certain training/test accuracy'. It is worth emphasizing that $\mathcal{J}(\Sigma;\lambda)$ is in general \emph{far from convex}, and its vacua \emph{far from unique} (as we shall prove in some examples). The main point of our results can be summarized in the slogan that `\emph{the significant vacua $\Sigma$ of $\mathcal{J}$ knows about the essential degrees of freedom, scale parameters and clustering behavior in $X$}'.

		\end{rmk}

		We now summarize the contents of the paper. Further conceptual discussions are given at the ends of several sections.
		
		\begin{enumerate}
			\item Section \ref{sect:kernelridgeregression} starts with a quick introduction of the kernel ridge regression problem where the RKHS is defined by a translation invariant kernel. Then we characterize the minimizer in terms of an \emph{integral equation} derived from the Euler-Lagrange equation (Proposition  \ref{prop:integraleqn1}), and then prove a quantitative estimate on the \emph{modulus of continuity} of $\mathcal{J}(\Sigma;\lambda)$ and the minimizer function  (Proposition  \ref{lem:continuityJ}, Lemma \ref{lem:Jminimizercontinuity1}) using the coercivity estimate of the integral equation; this implies in particular the \emph{continuous extension} of the functional $\mathcal{J}$ to the space of positive semi-definite form $Sym^2_{\geq 0}$. We recall the representer formula for the minimizer of the kernel ridge regression problem when $X$ takes only finitely many values (Proposition  \ref{prop:representer}). Using the \emph{law of large numbers}, the general distribution case can be approximated with arbitrarily high probability by the case of empirical distributions (Proposition  \ref{prop:lawoflargenumbers}).

			\item Section \ref{Translationinvkernel} studies the limiting behavior of $\mathcal{J}(\Sigma;\lambda)$ when $\Sigma$ \emph{tends to infinity} in the space of inner products, while remaining bounded on a (possibly empty) proper subspace $W\subset V$. The main theorem is Theorem  \ref{thm:asymptoticvalue}.  Perhaps surprisingly, we discover that the answer is sensitive to whether the marginal distribution of $X$ on the quotient space $V/W$ contains any \emph{discrete} part. If there is no discrete part, for instance if $X$ is a continuous variable, then $\lim \mathcal{J}(\Sigma;\lambda)= \frac{1}{2} \E[|Y|^2]$, which is the global supremum of $\mathcal{J}$. If there is a nontrivial discrete part, then we find that the kernel ridge regression problem decouples to a number of dimensionally reduced problems in the asymptotic limit, whose minimizers become orthogonal in the RKHS asymptotically. The proof technique relies on relating the kernel ridge regression problem to an \emph{optimal interpolation problem}, and uses a little Fourier analysis.

			\item 
			
			Section \ref{sect:rotationallyinvkernels} focuses on the \emph{rotationally invariant} kernel case. The main result is the \emph{first variation formula} of $\mathcal{J}$ in Theorem \ref{thm:firstvariation}, which generalizes \cite[Lemma 3.6]{CLLR} by relaxing the condition on the kernel function. We provide two arguments under slightly different conditions, one using integral equation techniques, the other using finite measure approximation via the law of large numbers. We then discuss when the first variation formula can be continuously extended to the boundary of $Sym^2_{\geq 0}$ (Proposition  \ref{prop:continuousextensionJ}), and introduce some formal definitions on the partial compactification (Section \ref{sect:partialcompactification}).

			Notable examples of rotationally invariant kernels come from \emph{completely monotone functions} according to Schoenberg's  theorem (Section
			\ref{sect:completelymontonekernel}), including Gaussian kernels, Laplace kernels and Sobolev kernels.

			\item 
			
			The \emph{vacua} are formally defined as the local minima of $\mathcal{J}$ on the partial compactification. There are two kinds of limiting configurations: vacua on the boundary of $Sym^2_{\geq 0}$ represented by degenerate inner products, and vacua at infinity.

			The boundary vacua are intimately related to the problem of variable selection (Section \ref{sect:boundaryvacuavariableselection}); \blue{e.g., Example~\ref{eg:variable-selection} where $Y$ depends on certain components of $X$}.   Applying the first variation formula to the boundary points $\Sigma$, we derive some criterion to test the local minimizing property (Section \ref{sect:boundaryofSym}), and identify two opposite effects for and against the local minimizing property (Lemma \ref{lem:boundarystability}, \ref{lem:linearsignal}).

			In Section \ref{sect:DiscreteI} we revisit the case where $X$ is a discrete variable taking a finite number of values, and extract the sub-leading effect when $\Sigma\to \infty$. A particular consequence is to exhibit examples for the vacua at infinity.

			\item

			Section \ref{sect:clusterinteractions} takes a more \emph{operator theoretic} approach to the Euler-Lagrange equation, with an eye towards analyzing \emph{multi-scaled} problems as in Examples \ref{eg:twoscale}, \ref{eg:multiscale}, where the probability distribution of $X$ contains several `clusters'. We prove various operator norm estimates (Proposition  \ref{prop:operatorT}) and their variants associated to clusters, to the effect that if two clusters are far separated spatially, or if their scale parameters are very different, then the operator products describing the interactions between the clusters have small norm. One can then compare the minimizer of the kernel ridge regression, versus the sum of the minimizers for the kernel ridge regression problems associated to the clusters. Theorems \ref{thm:noninteraction1}, \ref{thm:noninteraction3} and Corollary \ref{cor:noninteraction2} quantify their deviation, which is small when 
			the clusters are far separated.

			In the opposite direction, if some component of $X$ concentrates near its conditional expectation value with respect to the inner product $\Sigma$, then the minimizer $f_U$ is well approximated by the minimizer for the kernel ridge regression problem associated to a simpler probability distribution of $(X,Y)$, depending on \emph{fewer degrees of freedom} (Section \ref{sect:dimreductionmechanism}).

			\item 
			
			Section \ref{sect:Landscape2} formulates the notion of `\emph{scale detectors}', which we consider to be the most important vacua. The task is to relate scale detectors to the information of $(X,Y)$.

			In the simplest case where $X$ is a continuous variable whose marginal distributions satisfy a priori bounds on the probability density, and if the global minimum of $\mathcal{J}$ is strictly lower than the infimum of $\mathcal{J}$ on the boundary of $Sym^2_{\geq 0}$ with a quantitative gap, then one can show that the global minimizing vacuum $\Sigma$ is a scale detector , and all scale detectors in the interior of $Sym^2_+$ are uniformly equivalent to each other up to estimable constants (Theorem \ref{thm:scaledetectorglobalmin}).

			To exhibit a multi-scaled phenomenon we need to allow the probability density of $X$ to concentrate. Using the machinery developed in this paper, we construct a class of one-dimensional examples with \emph{multiple vacua} (typical cases being Example \ref{eg:multiscale}), which turn out to reflect the \emph{scale parameters} in the distribution of $X$ (Theorem \ref{thm:multiplevacua}).


		\end{enumerate}

		\begin{rmk}
			The exposition in this paper is aimed primarily at a general mathematical 
			readership familiar with analysis and linear algebra, and does not assume any 
			prior background in machine learning. Some results in this 
			paper are also reproduced from \cite{CLLR}, often in a slightly more general 
			setting. \blue{Readers interested in machine learning may wish to consult Appendix~\ref{appendix:discussion}, which interprets the main results through the lens of representation learning and discusses their possible implications for two-layer neural network models.}
		\end{rmk}

		\subsection{Related Work}
		\label{sec:comparison-to-prior-art}

		The general problem of learning a kernel from data was first discussed in \cite{LanckrietCrBaGhJo04} and has since been developed extensively; see, for example, \cite{GonenAl11}. Within this broad literature, a more recent and increasingly active line of work focuses on kernel learning through learnable linear transformations of the input, in which one introduces a matrix $U$ and considers kernels related to the form $K(Ux, Ux')$.
		
		This class of models has attracted increasing attention in recent years, in part because it admits a natural interpretation as a two-layer architecture: the linear transformation $x \mapsto Ux$ plays the role of the inner (feature) layer, while the function $f$ corresponds to the outer layer, in the sense that, for fixed $U$, optimization with respect to $f$ reduces to a linear regression problem. From this perspective, learning the transformation $U$ induces representation learning in kernel models and bears a close structural analogy to two-layer neural network architectures~\cite{Radhakrishnan24, FollainBa24, CLLR,  HuangLaDePoRo25, ZhuDaDrFa25, RuanLi}

		
		Existing work within such kernel-based representation learning has led to a variety of formulations: some works prefer to impose explicit constraints on $U$ (e.g., low-rank), others extend to combinations of multiple transformations $\{U_i\}$. These studies typically focus on finite-sample generalization guarantees, statistical recovery of predictive subspaces, or the analysis of specific algorithmic procedures. 
		
		The present paper is situated within this line of work, but adopts a different perspective: rather than analyzing finite-sample behavior or particular algorithms, we study the population-level variational landscape underlying these formulations. We next review representative works in this area and explain how our framework connects to---and differs from---these formulations in scope, perspective, and results.

		A foundational line of research by Fukumizu, Bach, and Jordan
		\cite{FukumizuBaJo04, FukumizuBaJo09} introduced low-rank kernel learning
		under the hard constraint $\mathrm{rank}(U)=r$ among others. They proved that, with a
		vanishing regularization parameter, the empirical \emph{global minimizer}
		recovers the correct predictive subspace of dimension $r$ in the large-sample
		limit.  More recently, Chen et al.~\cite{CLLR} removed the rank constraint and
		studied the unconstrained kernel learning model, which is the finite sample
		analogue of the problem considered in the present paper. They established
		that the empirical \emph{global minimizer} is statistically consistent in
		recovering both the predictive subspace and its intrinsic dimension, and
		demonstrated strong empirical performance relative to classical kernel
		ridge regression. Related theoretical developments have also shown improved
		generalization guarantees for \emph{global minimizers} when an explicit low-rank
		constraint is imposed on $U$~\cite{HuangLaDePoRo25}. 
		
		Beyond these single-transformation models,
		Follain and Bach~\cite{FollainBa24} propose a multiple kernel learning formulation to learn kernels of the form
		$\sum_i \eta_i K(U_i x, U_i x')$, which can be interpreted as learning multiple candidate representations than a single $U$; their main
		theoretical results provide high-probability generalization bounds for the
		empirical \emph{global minimizer} associated with that kernel.

		In a parallel algorithmic direction, Radhakrishnan et al.~\cite{Radhakrishnan24} study kernel-based feature learning through the AGOP (Average Gradient Outer Product) scheme, which adopts the same parametrization $f(UX)$ considered in the present paper.  AGOP can be formally viewed as an iterative procedure on the same kernel learning objective $J$. It has been analyzed theoretically, leading to improved generalization error bounds over standard kernel methods~\cite{ZhuDaDrFa25}, and has also demonstrated strong empirical performance on benchmark datasets~\cite{Radhakrishnan24}. However, AGOP is not derived from an explicit variational principle, and its precise relationship to the optimization landscape of $J$ remains largely unexplored, presenting an interesting direction for future research~\cite{Radhakrishnan24, RadhakrishnanBeDr25}.
		
		Complementary to these existing developments, our work studies the kernel learning objective $J$ entirely at the \emph{population level} and \emph{independently of any specific algorithmic procedure}. \blue{We develop a variational analysis of the population functional $J(U,\lambda)$, establishing its continuity properties, first-variation formula, and limiting behavior under degenerate and diverging transformations $U$ for various choices of data distributions $(X, Y)$ and RKHS $\mathcal{H}$}. This analysis yields a detailed description of the landscape of $J$, including both its global minimizers and its local minimizers (vacua), thereby revealing \emph{which representations are intrinsically preferred by this variational formulation in the infinite-sample limit}. \blue{Particular attention is given to two data-distribution settings of special interest, namely multi-index and multi-scale models. In the former, the learned transformation $U$ recovers the essential predictive feature variables, while in the latter it captures intrinsic scale information in the underlying data distribution---a representation-learning phenomenon that, to the best of our knowledge, has not been studied in the existing literature.} Together, these results demonstrate how the variational structure of the kernel learning problem can favor representations that capture multi-index and multi-scale structure in the data. \blue{At the same time, several natural variational questions remain open, including a more complete higher-order characterization of the variational landscape and extensions of the present theory to broader classes of data distributions and representation-learning phenomena (see the remarks throughout the paper).}


		\blue{Finally, beyond its \emph{static} variational implications, the framework developed here also provides a natural starting point for studying associated \emph{dynamical} questions, such as gradient-flow dynamics~\cite{RuanLi}, whose formulation fundamentally relies on the first-variation analysis developed in this paper. The study of such dynamical questions, however, lies beyond the scope of the present work. }

		\section{Kernel ridge regression problem}\label{sect:kernelridgeregression}

		\subsection{Translation invariant RKHS}

		We now give a quick review of translation invariant reproducing kernel Hilbert spaces (RKHS); more general references on RKHS can be found in \cite{Aronszajn}\cite{Smale}.


		Let $\mathcal{H}$ be a Hilbert space of complex valued functions on a $d$-dimensional Euclidean space $V\simeq \mathbb{R}^d$, associated to a translation invariant kernel. The space $\mathcal{H}$ is defined in terms of the Fourier transform by
		\begin{equation}
			\mathcal{H}=\{  f:  \norm{f}_{\mathcal{H}}^2= \int_{V^*} \frac{ |\hat{f}|^2(\omega)}{ k_V(\omega) }d\omega<+\infty\} ,  
		\end{equation}
		where we require the function $k_V$ to be strictly positive, and satisfies the \emph{$L^1$-integrability condition}
		\begin{equation}\label{L1integrablekV}
			\int_{V^*} k_V(\omega) d\omega=1.
		\end{equation}
		At this generality, we do not need to assume that $k_V$ is smooth or rotationally invariant, although such assumptions are frequently used in practice. The point of working in this generality is that one is forced to think through what are the essential structures.

		Clearly the Hilbert space is preserved under the translation $f\mapsto f(\cdot +a)$ for any $a\in V$.
		
		\begin{Notation}
			Our convention for the Fourier transform pair is
			\begin{equation}\label{Fouriertransform}
				\hat{f}(\omega)=\int_{\R^N} f(x) e^{-2\pi i \langle x, \omega\rangle} dx,\quad f(x)=\int_{(\R^N)^*} \hat{f}(\omega) e^{2\pi i \langle x, \omega\rangle} d\omega. 
			\end{equation}
			Notice the Fourier variables naturally live in the dual space.
			The reason for the $2\pi i$ choice is that it makes the least appearance of dimensional constants.

		\end{Notation}

		The integrability condition on $k_V$ ensures the following  Sobolev embedding properties on the function space $\mathcal{H}$. 
		
		\begin{lem}\label{lem:decayatinfinity}
			(Sobolev embedding) For any $f\in \mathcal{H}$, 
			we have $  \norm{f}_{L^\infty} \leq\norm{f}_{\mathcal{H}}$, and $f$ is continuous, and $|f|\to 0$ at infinity.
		\end{lem}

		\begin{proof}
			We note by Cauchy-Schwarz that
			\begin{equation}\label{Sobolevembeddingeqn1}
				\norm{f}_{L^\infty}^2 \leq \norm{\hat{f}}_{L^1}^2 \leq (\int_{V^*} \frac{  |\hat{f}|^2}{k_V} d\omega )(\int_{V^*} k_V d\omega )\leq \norm{f}_{\mathcal{H}}^2. 
			\end{equation}
			Since $\hat{f}$ is $L^1$, the continuity of $f$ follows from dominated convergence. By the Riemann-Lebesgue lemma $|f|\to 0$ at infinity.
		\end{proof}

		In particular, there is a continuous embedding of $\mathcal{H}$ into the space of continuous functions $C^0$, which means that $\mathcal{H}$ is a \emph{reproducing kernel Hilbert space} (RKHS).
		We define the kernel function
		\begin{equation}
			K(x)= \int_{V^*} e^{2\pi i \langle x, \omega\rangle} k_V(\omega) d\omega,
		\end{equation}
		and let $K(x, y)= K(x-y)$, so in particular $K(x,y)=\overline{K(y,x)}$. The continuous functions $K(x,\cdot)$ lie in $\mathcal{H}$, and have Fourier transform $k_V(\omega) e^{-2\pi i \langle \omega, x\rangle}$. By Riemann-Lebesgue, $K(x)\to 0$ as $|x|\to \infty$. These functions are known as reproducing kernels. 
		Each $f\in \mathcal{H}$ can be represented as 
		\begin{equation}\label{reproducingformula}
			f(x)= ( f, K(x,\cdot))_{\mathcal{H}}.
		\end{equation}

		\begin{lem}\label{lem:Rellich}
			The functions in the unit ball of $\mathcal{H}$ have uniform modulus of continuity.
		\end{lem}
		
		\begin{proof}
			By the reproducing kernel property, 
			\[
			|f(x)-f(y)|= |( f, K(x,\cdot)- K(y,\cdot) )_{\mathcal{H}}|\leq \norm{f}_{\mathcal{H}} \norm{  K(x,\cdot)- K(y,\cdot)   }_{\mathcal{H}}    ,
			\]
			and as $|y-x|\to 0$,
			\[
			\norm{  K(x,\cdot)- K(y,\cdot)   }_{\mathcal{H}}^2= K(x,x)+K(y,y)-2\text{Re}(K(x,y))=2-2\text{Re} K(x-y)\to 0.
			\]
			Hence the functions $f$ are equicontinuous.
		\end{proof}

		\subsection{Minimization problem}\label{sect:minimizationproblem}

		Let $X$ be a random variable taking value in the Euclidean space $V\simeq \R^d$, and $Y$ be a complex valued random variable with $\E[|Y|^2]<+\infty$. Following \cite[section 6]{Smale},
		we consider the following minimization problem, which we will refer to as \emph{kernel ridge regression}. The \emph{loss function} is
		\begin{equation}
			I(f, U, \lambda)= \frac{1}{2} \E |Y-f(UX)|^2 + \frac{\lambda}{2} \norm{f}_{\mathcal{H}}^2,
		\end{equation}
		depending on the parameters $U\in \End(V)$ and $\lambda\in (0, +\infty)$, and we seek the minimizer for
		\[
		J(U, \lambda)=  \inf_{f\in \mathcal{H}} I(f,U,\lambda).
		\]
		Notice that due to the Sobolev embedding, every $f\in \mathcal{H}$ is a continuous function on $V$, so $I(f,U,\lambda)$ is always well defined. The existence of a unique minimizer follows from the Riesz representation theorem \cite[Proposition 7]{Smale}. 

		\begin{lem}\label{lem:trivialupperbound}
			(Trivial upper bound) We always have $J(U,\lambda)\leq \frac{1}{2} \E[|Y|^2]$. If $\E[Y|UX]$ is nonzero, then the strict inequality holds.
		\end{lem}

		\begin{proof}
			Plug in the test function $f=0$, then $I(0, U, \lambda)\geq J(U,\lambda)$. The equality is achieved only when $f=0$ is the minimizer, which by the Euler-Lagrange equation implies that $\E[Y|UX]=0$.
		\end{proof}

		\begin{lem}\label{lem:triviallowerbound}
			(Trivial lower bound) 
			\[
			I(f, U,\lambda) =\frac{1}{2}\E[\text{\rm Var}(Y |UX)] + \frac{1}{2}\E[ |\E[  Y|UX]- f(UX)|^2   ] + \frac{\lambda}{2} \norm{f}_{\mathcal{H}}^2.
			\]
			In particular
			$
			J(U,\lambda) \geq  \frac{1}{2}\E[\text{\rm Var}(Y |UX)]. 
			$
		\end{lem}

		\subsection{Euler-Lagrange equation}
		
		\begin{lem}\label{EulerLagrange}
			(Euler-Lagrange) The minimizer $f_U$ for $I(f, U,\lambda)$ satisfies that
			\[
			\E[ (Y-f_U(UX) )\bar{g}(UX)   ]= \lambda ( f_U, g)_{\mathcal{H}},\quad \forall g\in \mathcal{H}.
			\]
		\end{lem}

		We plug in the kernel function $K(x, \cdot)$ as the test function $g$ in the Euler-Lagrange equation, to deduce the integral equation
		\begin{equation}\label{eqn:EulerLagrange1}
			f_U(x)= ( f_U, K(x,\cdot) )_{\mathcal{H}} = \frac{1}{\lambda} \E[ (Y-f_U(UX))\overline{K(x,UX) }  ].
		\end{equation}
		By plugging the test function $f_U$ into the Euler-Lagrange equation, we deduce
		\[
		\E[ Y\overline{f_U(UX) }  ]= \E[|f_U(UX)|^2]+\lambda ( f_U, f_U)_{\mathcal{H}}
		\]
		whence
		\begin{equation}\label{eqn:Jalternativeformula}
			J(U,\lambda)= \frac{1}{2} \E[|Y|^2]- \frac{1}{2}\E[ Y\overline{f_U(UX) }]= \frac{1}{2} \E[|Y|^2]- \frac{1}{2}\E[ \bar{Y}f_U(UX)   ].
		\end{equation}


		\begin{Notation}
			We denote the (semi)-norm $\norm{f}_\E^2= \E[|f(X)|^2]$ for functions $f$ defined on the support of $X$. 
		\end{Notation}

		\begin{prop}\label{prop:integraleqn1}
			The function $f_U\circ U$ is characterized as the unique finite norm solution to the integral equation
			\[
			F(x)= \frac{1}{\lambda}  \E[ (Y-F(X))\overline{K(U(x-X)) }  ]. 
			\]
		\end{prop}

		\begin{proof}
			For later application, we consider the more general integral equation
			\[
			f(x)+ \frac{1}{\lambda}  \E[f(X)\overline{K(U(x-X)) }  ]= \E[h(x,X)],
			\]
			where $\E[ |h(X',X)|^2 ]<+\infty $, and $X'$ is an independent copy of $X$. We claim that the equation has a unique finite norm solution, and it satisfies the \emph{coercivity estimate}
			\begin{equation}\label{eqn:coercivity}
				\norm{f}_\E^2 \leq \E[|h(X',X)|^2]. 
			\end{equation}
			We notice that for any $f$ with $\norm{f}_\E<+\infty$,
			\begin{equation}
				\begin{split}
					& \E[ f(X)\overline{K(U(x-X)) } \overline{f(X')} ]
					\\
					= &  \int_{V^*} \E[ f(X) e^{-2\pi i\langle U(X-X'),\omega\rangle}
					\overline{f(X')} ] k_V(\omega)d\omega
					\\
					=&  \int_{V^*} |\E[ f(X) e^{-2\pi i\langle UX,\omega\rangle}]|^2
					k_V(\omega)d\omega \geq 0,
				\end{split}
			\end{equation}
			and the linear functional
			$
			g\mapsto \E[ h(X',X) g(X') ]
			$
			has norm bounded by $\E[ |h(X',X)|^2]^{1/2}$ with respect to the Hilbert inner product
			\[
			\norm{f}_\E^2\leq \norm{f}_\E^2+ \E[ f(X)\overline{K(U(x-X)) } \overline{f(X')} ],
			\]
			so the claim follows from Riesz representation theorem.

			In our setting, the integral equation on $f_U\circ U$ follows from Euler-Lagrange (\ref{eqn:EulerLagrange1}). The coercivity estimate applied to the homogeneous equation implies uniqueness.
		\end{proof}
		
		\blue{The integral equation admits a natural interpretation. The quantity $Y-F(X)$ represents the prediction residual. The value of $F(x)$ is obtained by averaging these residuals against the kernel profile centered at $x$.}

		\subsection{Quantitative continuity in parameter}

		We now prove some quantitative continuity in $U$ for the minimal value $J(U,\lambda)$ and minimizer $f_U$ of $I(f,U,\lambda)$.

		\begin{lem}\label{lem:Jminimizercontinuity1}
			For any $U,\tilde{U}$, there is a uniform bound
			\begin{equation}\label{eqn:Jminimizercontinuity1}
				\E[ | f_U(UX)- f_{\tilde{U}}(\tilde{U}X)|^2   ] \leq  \frac{1}{\lambda^2} \E[|Y|^2 ]  \E[ |K(U(X'-X))- K( \tilde{U}(X'-X)) |^2 ] ,    
			\end{equation}
			where $X'$ denotes an independent copy of $X$. 
		\end{lem}
		
		\begin{proof}
			By the integral equation for $f_U\circ U$ and $f_{\tilde{U}}\circ \tilde{U}$, 
			the difference $f(x)=f_{\tilde{U}}(\tilde{U}x)-f_U(Ux) $ satisfies
			\begin{equation}\label{eqn:integraleqndifferentiability}
				f(x)+  \frac{1}{\lambda}  \E[ f(X)\overline{K(U(x-X)) }  ] = \frac{1}{\lambda}  \E[ (Y-f_{\tilde{U}}(\tilde{U}X) )\overline{g(x-X)}  ],
			\end{equation}
			for $g(x)= K(\tilde{U}x) -K(Ux)$. By applying the coercivity estimate (\ref{eqn:coercivity}),
			\[
			\norm{f}_\E^2 \leq \frac{1}{\lambda^2} \E[ |Y- f_{\tilde{U}}(\tilde{U}X) |^2   ] \E[ |g(X'-X)|^2 ] \leq \frac{1}{\lambda^2} \E[ |Y |^2   ] \E[ |g(X'-X)|^2 ] ,
			\]
			where the second inequality uses the trivial upper bound Lemma \ref{lem:trivialupperbound}. 
		\end{proof}

		\begin{prop}\label{lem:continuityJ}
			(Quantitative continuity of $J$)
			For any  $U,\tilde{U}$, there is a uniform bound
			\[
			|J(U,\lambda)-J(\tilde{U},\lambda)|\leq \frac{1}{2\lambda}\E[|Y|^2 ]  \E[ |K(U(X'-X))- K( \tilde{U}(X'-X)) |^2 ]^{1/2}, 
			\]
			where $X'$ denotes an independent copy of $X$.
		\end{prop}

		\begin{proof}
			Using the formula (\ref{eqn:Jalternativeformula}),
			and Cauchy-Schwarz, 
			\[
			|J(U,\lambda)-J(\tilde{U},\lambda)|\leq \frac{1}{2} \E[|Y|^2]^{1/2} \E[ | f_U(UX)- f_{\tilde{U}}(\tilde{U}X)|^2   ]^{1/2},
			\]
			so the Proposition follows from (\ref{eqn:Jminimizercontinuity1}).
		\end{proof}

		\begin{cor}\label{cor:Jminimizercontinuity}
			The minimizer $f_U$ depends continuously on $U$, in the Hilbert space $\mathcal{H}$-norm. 
		\end{cor}

		\begin{proof}
			Since $I(f, U,
			\lambda)$ is quadratic in $f$ with 
			minimizer $f_U$, we have
			\[
			I(f,U,\lambda)= \mathcal{Q}(f-f_{U})+ J(U,\lambda),
			\]
			where $\mathcal{Q}$ is the quadratic term, bounded below by $\frac{\lambda}{2} \norm{f-f_{U}}_{\mathcal{H}}^2$. The same holds with $U$ replaced by any other $\tilde{U}$. In particular, when $\tilde{U}\to U$, 
			\[
			J(\tilde{U},\lambda)+ \frac{\lambda}{2} \norm{f_U-f_{\tilde{U}}}^2_{\mathcal{H}}\leq I(f_U, \tilde{U}, \lambda) \to I(f_U, U,\lambda)=J(U,\lambda),
			\]
			Hence
			\[
			\limsup_{\tilde{U}\to U} \frac{\lambda}{2} \norm{f_U-f_{\tilde{U}}}^2_{\mathcal{H}} \leq \lim_{ \tilde{U}\to U } (J(U ,  \lambda) - J(\tilde{U} ,  \lambda))=0,
			\]
			so $  \norm{f_U-f_{\tilde{U}}}_{\mathcal{H}}\to 0 $.
		\end{proof}

		\subsection{Representer theorem for finite atomic measures}\label{sect:representer}

		In the simplest case $X$ takes value in a \emph{finite number of points} $\{ a_1,\ldots a_m\} \subset V $, with $\mathbb{P}(X=a_i)=p_i$, and $Y=y_i$ conditional on $X=a_i$. Then the representer theorem \cite[Section 7.2]{Bach}\cite[Proposition 8]{Smale} gives an explicit formula for $J(U,\lambda)$ in terms of the kernel functions.

		We write $M=(M_{ij})$ as the $m\times m$ inner product matrix with entries
		\[
		M_{ij}=K(Ua_i, Ua_j)= K( U(a_i-a_j)) 
		\]
		\blue{and let $M^{-1}$ denote its inverse matrix whenever $M$ is invertible, which is the case when the points $\{U a_i\}_1^m$ are distinct from each other}. We further write $P=\text{diag}(p_1,\ldots p_m)$ encoding the probability of the atoms at $a_i$.

		\begin{prop}\label{prop:representer} \cite[Proposition 8]{Smale}
			(Representer theorem) \blue{Assume that $M$ is invertible. Then}
			\begin{equation}
				J(U,\lambda)= \frac{1}{2}\sum_1^m p_i |y_i|^2-\frac{1}{2} \sum_{i,j=1}^m(\lambda M^{-1}+P)^{-1}_{ij} p_i y_i p_j\bar{y}_j.
			\end{equation}
		\end{prop}

		\begin{proof}
			We minimize the functional
			\[
			I(f, U,\lambda)= \frac{1}{2} \sum_1^m p_i |y_i- f(U a_i)|^2 +  \frac{\lambda}{2} \norm{f}_{\mathcal{H}}^2.
			\]
			We first fix the $m$ numbers $f(Ua_i)$, and minimize $\norm{f}_{\mathcal{H}}^2$. This is an optimal interpolation problem, and the solution is
			\[
			f= \sum_{i,j=1}^m (M^{-1})_{ji} f(Ua_j) K( Ua_i, \cdot),\quad \norm{f}_{\mathcal{H}}^2= \sum_{i,j=1}^m (M^{-1})_{ij} f(U a_i) \overline{f(U a_j)}. 
			\]
			In the second step, we minimize with respect to  $f(Ua_i)$; this is a finite dimensional linear regression problem. The minimizer satisfies
			\begin{equation}\label{eqn:representerminimizer}
				f(U a_i)=\sum_{j=1}^m (\lambda M^{-1}+P)^{-1}_{ji} p_j y_j,
			\end{equation}
			whence by (\ref{eqn:Jalternativeformula})
			\[
			J(U,\lambda)= \min_f I(f,U,\lambda)=\frac{1}{2}\sum p_i |y_i|^2- \frac{1}{2}\sum (\lambda M^{-1}+P)^{-1}_{ij} p_i y_i p_j\bar{y}_j,
			\]
			as required.
		\end{proof}

		\subsection{Law of large numbers}\label{sect:lawoflargenumbers}

		Suppose we take $m$ i.i.d. observations of $(X,Y)$, then the empirical expectation is defined by
		\[
		\E_m [ F(X,Y)]= \frac{1}{m} \sum_1^m F(X_{(i)}, Y_{(i)} ).
		\]
		Suppose that $\E[|F(X,Y)|]<+\infty$, then the law of large numbers implies that almost surely
		\[
		\lim_{m\to +\infty} \E_m [ F(X,Y)] = \E[ F(X,Y)]. 
		\]
		By replacing $\E$ with the empirical expectation $\E_m$, we can define the empirical version $I_m(f,U,\lambda)$ of $I(f,U,\lambda)$, whose minimizer is $f_{U,m}$, and the minimal value is $J_m(U,\lambda)$. 
		
		\red{We first show the almost sure convergence of the empirical $I_m(f,U,\lambda)$ to $I(f, U,\lambda)$.}
		\begin{lem}\label{lemma:lawoflargenumbers}
			Suppose $\E[|Y|^2]<+\infty$. For every constant $M < \infty$, almost surely
			\[
			\lim_{m\to+\infty}\sup_{f\in \mathcal{H}: \norm{f}_\mathcal{H}^2\leq M} \sup_{U: |U|\leq M } |\E_m[|f(UX)- Y|^2 ]- \E[|Y-f(UX)|^2] | = 0.
			\]
			As a consequence, $I_m(f, U, \lambda) \to I(f, U, \lambda)$ converge uniformly in $f \in \{f: \norm{f}_\mathcal{H}^2\leq M\}$ and $U \in \{U: |U| \le M\}$.
		\end{lem}
		\begin{proof}
			Let $\epsilon > 0$ be a small number. For the random variable $X$ we can first find some large $R$ depending on $\epsilon$, such that $\lim_{m \to \infty} \E_m [\1_{|X|\geq R}] = \E[\1_{|X| \geq R}] \leq \epsilon^2$. We then notice that the class of functions $f
			\circ U$ with a given bound $M$ is equicontinuous and uniformly bounded by $M^{1/2}$, so by Arzela-Ascoli, there is a finite $\epsilon$-net $g_1,\ldots g_N$, such that any $f\circ U$ in this class is $\epsilon$-close in $C^0$-norm to some $g_i$ on the compact set $\{ 
			|x|\leq R\}$. By the law of large numbers applied to $|Y-g_i(X)|^2$, almost surely we have
			\[
			\lim_{m \to \infty} \max_{1\leq i\leq N} |\E_m|Y-g_i(X)|^2 - \E[|Y-g_i(X)|^2]| = 0.
			\]
			We note that $f, g_i$ are all uniformly bounded in $C^0$ norm by $M^{1/2}$. By separately estimating the contribution for $|X|\leq R$ and $|X|\geq R$, we see that almost surely,
			\[
			\begin{split}
				& |\E_m[ |f(UX)- Y|^2 ]- \E_m [|g_i(UX)-Y|^2]|
				\\
				&
				\leq C\epsilon (\E_m[|Y|+ |f(UX)|+|g_i(UX)|] ) + C (\E_m[(|Y|+ |f(UX)|+| g_i(UX)| )\1_{ |X|\geq R }] 
				\\
				& \leq C \epsilon + C\epsilon \E_m[|Y|] + C (\E_m[|Y|^2+ |f(UX)|^2+| g_i(UX)|^2 ])^{1/2} \E_m[\1_{ |X|\geq R }]^{1/2}
				\\
				&\leq C\epsilon + C\epsilon (\E_m[|Y|^2])^{1/2},
			\end{split}
			\]
			where $C < \infty$ depends on $M$ but not on $f, g_i$. 
			A similar estimate holds with $\E_m$ replaced by 
			$\E$. Since $\epsilon > 0$ is arbitrary, and $\E_m[|Y|^2] \to \E[|Y|^2] < \infty$, the lemma follows.
		\end{proof}

		\red{We next show the almost sure convergence of the empirical minimizers.}
		\begin{prop}\label{prop:lawoflargenumbers}
			Suppose $\E[|Y|^2]<+\infty$. For any compact subset $\mathcal{U}$ of $U$, the following event holds almost surely: as $m \to +\infty$, the minimizer $f_{U,m}\in \mathcal{H}$ (resp. $J_m(U,\lambda)$) converges uniformly in $U$ to $f_U\in \mathcal{H}$ (resp. $J(U,\lambda)$). 
		\end{prop}

		\begin{proof}
			By the law of large numbers, almost surely
			\[
			\norm{f_{U,m}}_{\mathcal{H}}^2\leq \lambda^{-1}\E_m[|Y|^2]\to \lambda^{-1}\E[|Y|^2]<+\infty.
			\]
			Hence, for any $\delta > 0$, we can select $M < \infty$ such that with probability at least $\geq 1-\delta$, $M$ serves as a uniform bound on $\norm{f_{U,m}}_{\mathcal{H}}^2$ and $|U|$ in the compact set $\mathcal{U}$. By applying Lemma~\ref{lemma:lawoflargenumbers} for this $M$, we deduce that 
			\[
			\sup_{U \in \mathcal{U}} |I_m(f_{U, m}, U, \lambda) - I(f_{U, m}, U, \lambda)| \to 0  
			~~~ (m\to \infty)   
			\] 
			occurs with probability at least $1-\delta$. 
			Since $\delta > 0$ is arbitrary, this means that it also occurs almost surely. Similarly, we can deduce that almost surely,
			\[
			\sup_{U \in \mathcal{U}} |I_m(f_{U}, U, \lambda) - I(f_{U}, U, \lambda)| \to 0  
			~~~ (m\to \infty). 
			\] 
			
			Since $I_m(f, U,
			\lambda)$ is quadratic in $f$ with 
			minimizer $f_{U,m}$, we have
			\[
			I_m(f,U,\lambda)= \mathcal{Q}(f-f_{U,m})+ J_m(U,\lambda),
			\]
			where $\mathcal{Q}$ is the quadratic term, bounded below by $\frac{\lambda}{2} \norm{f-f_{U,m}}_{\mathcal{H}}^2$. This leads to the following inequality: 
			\[
			J_m(U,\lambda)+ \frac{\lambda}{2} \norm{f_U-f_{U,m}}^2_{\mathcal{H}}\leq I_m(f_U, U, \lambda) \leq I(f_U, U,\lambda)+|(I_m-I)(f_U, U, \lambda)|. 
			\]
			Since $I(f_U, U, \lambda) = J(U,\lambda)$, this becomes:
			\[
			J_m(U,\lambda)+ \frac{\lambda}{2} \norm{f_U-f_{U,m}}^2_{\mathcal{H}}\le J(U, \lambda)
			+|(I_m-I)(f_U, U, \lambda)|.
			\]
			Similarly we deduce the reverse inequality
			\[
			\begin{split}
				J(U,\lambda)+ \frac{\lambda}{2}\norm{f_U-f_{U,m}}^2_{\mathcal{H}} 
				\leq J_m(U,\lambda) + |(I_m-I)(f_{U, m}, U, \lambda)|. \\
			\end{split}
			\]
			
			Recall that $|(I_m-I)(f_U, U, \lambda)| $ and $|(I_m-I)(f_{U, m}, U, \lambda)|$ converge uniformly to zero in $U \in \mathcal{U}$ as $m \to \infty$. Combining these bounds, we obtain that
			\[
			\sup_{U \in \mathcal{U}} \big\{|J_m(U,\lambda)- J(U,\lambda)|+\frac{\lambda}{2}\norm{f_U-f_{U,m}}^2_{\mathcal{H}}\big\} \to 0~~~(m \to \infty).
			\]
			This shows the uniform convergence of both the minimizers and the minimum values in $U$.
		\end{proof}

		\red{	
			\begin{rmk}
				As a caveat, the uniform convergence of the empirical objective to the population objective assumes $U$ lies in a bounded subset. We shall see later that the behavior of $J(U,\lambda)$ for very large $U$ is sensitive to the distinction between continuous and discrete distribution for $X$, hence the empirical vs. population minimizers may behave in qualitatively different ways.
			\end{rmk}	
		}

		\section{Asymptotic regime}\label{Translationinvkernel}
		
		We will be interested in the following question:
		\begin{Question}
			For fixed $\lambda$, how does the minimum and the minimizer depend on $U$ as $U$ tends to infinity?
		\end{Question}

		More precisely, we fix an orthogonal direct sum decomposition $V=W\oplus W'$  with respect to $|\cdot|_V^2$,  
		which induces a direct sum decomposition of the dual space
		$
		V^*= W^*\oplus W'^*.
		$
		Thus we can write $x\in V$ and $\omega\in V^*$ as
		\[
		x=(x_W, x_{W'}),\quad \omega=(\omega_W, \omega_{W'}),\quad \langle x,\omega\rangle=\langle x_W, \omega_W\rangle+\langle x_{W'},\omega_{W'}\rangle. 
		\]

		Then a general endomorphism $U\in \End(V)$ has 4 components, which can be written in matrix form as
		\begin{equation}\label{eqn:Umatrix}
			U= \begin{bmatrix}
				U_{11}, & U_{12}\\
				U_{21}  , &U_{22}
			\end{bmatrix} \in 
			\begin{bmatrix}
				\End(W), & \Hom(W', W)\\
				\Hom(W, W'), & \End(W')
			\end{bmatrix}. 
		\end{equation}
		The action of $U$ on $x=(x_W, x_{W'})$ is $Ux=( U_{11}x_W+U_{12}x_{W'}, U_{21}x_W+ U_{22}x_{W'})$.
		We will consider a \emph{sequence} of $U$ depending on some index $k$ which we will usually suppress in the notation. We assume that 
		\begin{enumerate}
			\item  The entry $U_{11}$ remains uniformly bounded and converges to a finite limit as $k\to +\infty$, denoted as $U_{11}^\infty$.
			
			\item The entry $U_{21}=0$, and $U_{12}$ is arbitrary.
			
			\item  There are lower bounds $|U_{22} x_{w'} |\geq \Lambda_k |x_{w'}|$ for any $x_{w'}\in W'$, where $\Lambda_k\to +\infty$ as $k\to +\infty$.   
		\end{enumerate}

		\red{Intuitively, this limit $U\to +\infty$ means that the kernel becomes increasingly sensitive to differences in the $W'$ direction, so points with different $W'$-coordinates become asymptotically separated.}

		The random variable $X$ can also be decomposed as $X=(X_W, X_{W'})$. In general, the marginal distribution of $X_{W'}$ is the sum of a discrete part and a non-atomic part (which could have a singular continuous part in general),
		\begin{equation}\label{atomic}
			\mu_{X_{W'}} \sim \sum_1^\infty p_i \delta_{a_i}+ \mu_{NA},\quad p_i> 0, \sum p_i \leq 1,\quad a_i\in W'.
		\end{equation}
		If $X_{W'}$ is a continuous variable, then the discrete part does not appear. It will turn out that the asymptotic behavior of $J(U,\lambda)$ as $U\to +\infty$ for \emph{continuous variables} vis \`a vis \emph{discrete variables} will be drastically different. \red{Roughly, in the large $U$ limit, the kernel interaction $K(U(X-X'))$ vanishes unless two samples have the same $W'$ component. For a non-atomic distribution, two independent samples almost surely have different $W'$-components, so the kernel interactions disappear and the limiting contribution is trivial. By contrast, atoms $a_i$ satisfy $\mathbb{P}(X_{W'}=a_i)>0$, so samples within the same atom continue to interact and give rise to lower-dimensional kernel ridge regression problems on $W$.}

		\subsection{Dimensional reduction of the RKHS}\label{sect:dimreductionRKHS}

		We can define a new RKHS $\mathcal{H}_W$ for a class of functions $f$ on $W$ with the following norm:
		\begin{equation}
			\norm{f}_{ \mathcal{H}_W }^2= \int_{W^*} \frac{ |\hat{f}(\omega_W)|^2 }{ k_W  } d\omega_{W},
		\end{equation}
		where
		\[
		k_W(\omega_W)= \int_{ W'^*} k_V(\omega_W, \omega_{W'}) d\omega_{W'}.
		\]
		From the integrability assumptions on $k_V$, Fubini theorem implies
		\[
		\int_{W^*} k_W d\omega_W=  \int_{V^*} k_V(\omega) d\omega=1,
		\]
		Thus $k_W$ inherits the same kind of integrability conditions.

		The relation between $\mathcal{H}$ and $\mathcal{H}_W$ arises from the optimal extension  problem. 
		
		\begin{lem}\label{InterpretationHW}
			(Interpretation of $\mathcal{H}_W$) If $f\in \mathcal{H}$, then its restriction to the subspace $W\subset V$ lies in $\mathcal{H}_W$. Moreover, for any $g\in \mathcal{H}_W$,
			\[
			\norm{g}_{ \mathcal{H}_W } = \min_{f\in \mathcal{H}}\{ \norm{ f    }_{\mathcal{H}} : f|_W=g  \}. 
			\]
		\end{lem}

		\begin{proof}
			We observe that $f|_W=g$ if and only if 
			\[
			\begin{split}
				g(x_W)& = f(x_W, 0)= \int_{V^*} \hat{f}(\omega_W, \omega_{W'}) e^{2\pi i \langle x_W, \omega_W\rangle } d\omega
				\\
				& =   \int_{W^*}e^{2\pi i \langle x_W, \omega_W\rangle } d\omega_W  \int_{W'^* } \hat{f}(\omega_W, \omega_{W'}) d\omega_{W'},
			\end{split}
			\]
			if and only if 
			\[
			\begin{split}
				& \hat{g}(\omega_W)= \int_W g(x_W) e^{-2\pi i \langle x_W, \omega_W\rangle} dx_W 
				=  \int_{ W'^* } \hat{f}(\omega_W, \omega_{W'}) d\omega_{W'}. 
			\end{split}
			\]
			This implies
			\[
			\begin{split}
				\norm{g}_{\mathcal{H}_W}^2 & = \int_{W^*}  d\omega_W |\hat{g}|^2 k_W^{-1} = \int_{W^*}  d\omega_W 
				| \int_{W'^*} \hat{f}(\omega_W, \omega_{W'}) d\omega_{W'}    |^2 k_W^{-1}
				\\
				& \leq  \int_{W^*}  d\omega_W 
				(\int_{W'^*} |  \hat{f}|^2 k_V^{-1} d\omega_{W'} )  
				\\
				& = \norm{f}_{ \mathcal{H} }^2. 
			\end{split}
			\]
			where the second line uses Cauchy-Schwarz applied to the integrand of the outer integral. This proves $\norm{g}_{\mathcal{H}_W} \leq \norm{f}_{\mathcal{H}}$, and in particular the restriction of $f$ to $W$ defines a function in $\mathcal{H}_W$.

			Moreover, for any given $g\in \mathcal{H}_W$, the equality is achieved by the $f\in \mathcal{H}$ whose Fourier transform is defined by
			\[
			\hat{f}(\omega)= \hat{g}(\omega_W) \frac{k_V( \omega) }{ k_W(\omega_W) }.
			\]
			This implies that 
			$
			\norm{g}_{ \mathcal{H}_W } = \min_{f\in \mathcal{H}}\{ \norm{ f    }_{\mathcal{H}} : f|_W=g  \}. 
			$
		\end{proof}

		We shall later need a generalization of Lemma \ref{InterpretationHW}. The \emph{optimal interpolation problem} asks the following:
		
		\begin{Question}
			Given $m$ distinct points $y_i\in W'$  and functions $g_i\in \mathcal{H}_W$ for $i=1,\ldots m$, what is the $f\in \mathcal{H}$ with the smallest $\mathcal{H}$-norm, such that its restrictions to the affine subspaces $y_i+ W$ are prescribed by $f(x_W, y_i)= g_i(x_W) $ for $i=1,\ldots m$?  
		\end{Question}

		\begin{lem}\label{Optimalinterpolationlem}
			(Optimal interpolation)
			Suppose there exists some $f\in \mathcal{H}$ with $f(x_W, y_i)= g_i(x_W) $ for $i=1,\ldots m$. Then 
			the formula for the minimizer $f_*$ of the optimal interpolation problem can be written in Fourier transform as
			\begin{equation}\label{optimalinterpolationminimizer}
				\hat{f}_*( \omega)= \sum (M^{-1})_{ji} \hat{g}_j(\omega_W) \frac{ k_V(\omega) } { k_W(\omega_W) } e^{ - 2\pi i \langle \omega_{W'}, y_i\rangle  },
			\end{equation}
			where $M^{-1}$ is the inverse matrix to the following Hermitian matrix depending on $\omega_W$,
			\begin{equation}\label{optimalinterpolationK}
				M_{ij}(\omega_W)= k_W(\omega_W)^{-1} \int_{ W'^* } k_V(\omega) e^{-2\pi i \langle \omega_{W'}, y_i-y_j\rangle}  d\omega_{W'}. 
			\end{equation}
			The norm of the minimizer is given by
			\begin{equation}\label{optimalinterpolationminimum}
				\norm{f_*}_{\mathcal{H}}^2=  \int_{ W^* } \sum k_W^{-1}(M^{-1})_{ij} \hat{g}_i \overline{ \hat{g}}_j d\omega_W. 
			\end{equation}
		\end{lem}

		\begin{proof}
			For any $f\in\mathcal{H}$, we recall that $\norm{\hat{f}}_{L^1}\leq \norm{f}_\mathcal{H}$ by Cauchy-Schwarz, and in particular $\hat{f}$ is $L^1$. 
			We first reformulate the restriction conditions in terms of the Fourier transform. The requirement is
			\[
			\begin{split}
				g_i(x_W)& = f(x_W, y_i)= \int_{W^*} \hat{f}(\omega_W, \omega_{W'}) e^{2\pi i (\langle x_W, \omega_W\rangle +\langle y_i, \omega_{W'}\rangle )}d\omega
				\\
				& =   \int_{W^*} \int_{ W'^*} \hat{f}(\omega_W, \omega_{W'}) e^{2\pi i \langle x_W, \omega_W\rangle } e^{2\pi i \langle y_i, \omega_{W'}\rangle }d\omega_W d\omega_{W'},
			\end{split}
			\]
			which is equivalent to
			\begin{equation}\label{optimalinterpolationeqn1}
				\hat{g}_i(\omega_W)= \int_{ W'^*} \hat{f}(\omega_W, \omega_{W'})  e^{2\pi i \langle y_i, \omega_{W'}\rangle } d\omega_{W'}, \quad i=1,\ldots m.
			\end{equation}
			Our goal is to find $\hat{f}_*$ minimizing the integral 
			\[
			\int_{V^* }  |\hat{f}|^2 k_V^{-1} d\omega= \int_{W^*} d\omega_W \int_{ (W')^* } |\hat{f}|^2 k_V^{-1} d\omega_{ W' }. 
			\]
			The problem decouples to an infinite collection of independent minimization problems: for each fixed value of $\omega_W$, we view $\hat{f}(\omega_W, \cdot)$ as an unknown function of $\omega_{W'}$, so (\ref{optimalinterpolationeqn1}) prescribes the value of $m$ linear functionals, and we want to minimize the unknown function with respect to the Hilbert inner product
			\[
			\int_{ W'^* } |\hat{f}|^2 k_V^{-1}(\omega_W, \omega_{W'}) d\omega_{ W' }.
			\]
			The point is that the different $\omega_W$ do not interact with each other.

			There is a standard procedure to solve this minimization problem. First, we use the Riesz representation theorem to identify the $m$ linear functionals as $m$ elements of the Hilbert spaces, and the minimizer $\hat{f}(\omega_W,\cdot)$ must lie in the $\C$-span of these $m$ elements. This shows that
			\begin{equation}\label{optimalinterpolationeqn2}
				\hat{f}(\omega_W,\cdot)=  \sum_1^m a_i k_V(\omega_W,\cdot) e^{ -2\pi i \langle \cdot, y_i\rangle} ,
			\end{equation}
			where $a_i$ are $m$ coefficients depending on $\omega_W$ to be determined.

			The Hilbert inner product evaluated on the spanned of the $m$ elements $k_V(\omega_W,\cdot) e^{ -2\pi i \langle \cdot, y_i\rangle}$ gives rise to the Hermitian inner product matrix
			\[
			\int_{W'^*} k_V^{-1} k_V(\omega_W,\cdot) e^{ -2\pi i \langle \cdot, y_i\rangle} k_V(\omega_W,\cdot) e^{ 2\pi i \langle \cdot, y_j\rangle} d\omega_{W'}= k_W(\omega_W) M_{ij}(\omega_W). 
			\]
			In particular this shows $(M_{ij})$ is  positive semi-definite. Since $k_V>0$, we see that the $m$ functions $k_V(\omega_W,\cdot) e^{ -2\pi i \langle \cdot, y_i\rangle}$ are linearly independent, so the inner product matrix is also non-degenerate. 
			As a formality check, 
			\[
			\int_{W^*} k_W |M_{ij}| d\omega_W\leq  \int_{W^*} d\omega_W  \int_{ W'^* } k_V(\omega)  d\omega_{W'}<+\infty, \quad \forall i,j,
			\]
			so $(M_{ij})$ is indeed well defined for almost every $\omega_W$.

			It remains to determine the coefficient $a_i$. For this we plug (\ref{optimalinterpolationeqn2}) into the linear constraints (\ref{optimalinterpolationeqn1}) to get
			\[
			\hat{g}_i(\omega_W)= \int_{ (W')^*} \hat{f}_*(\omega_W, \omega_{W'})  e^{2\pi i \langle y_i, \omega_{W'}\rangle } d\omega_{W'}= k_W\sum a_j M_{ji}.
			\]
			Solving this linear system gives
			\[
			a_i= k_W^{-1} (M^{-1})_{ji} \hat{g}_j(\omega_W).
			\]
			This gives the formula for the minimizer (\ref{optimalinterpolationminimizer}). We can then compute the norm by evaluating the inner product matrix for $f_*$,
			\[
			\norm{f_*}_{\mathcal{H} }^2=  \int_{W^*} \sum k_W M_{ij} a_i \bar{a}_j d\omega_{W}= \int_{W^*} \sum k_W^{-1} M^{-1}_{ij} \hat{g}_i \bar{ \hat{g}  }_jd\omega_{W}
			\]
			as required.
		\end{proof}
		
		\blue{
			The matrix $M$ encodes the coupling between the interpolation conditions imposed on the affine slices $y_i+W$. When these slices are well separated, the off-diagonal entries of $M$ become small, the coupling weakens and the interpolation problem approximately decouples into \emph{independent} problems on the individual slices. This \emph{decoupling phenomenon} will recur throughout the paper, leading to various orthogonality, decoupling, and dimensional-reduction effects. The following corollary~\ref{cor:asymptoticorthogonalityoptimalinterpolation} provides a first illustration of this phenomenon in the setting of optimal interpolation.
		}

		\begin{cor}\label{cor:asymptoticorthogonalityoptimalinterpolation}
			(\emph{Asymptotic orthogonality in the optimal interpolation problem}) We fix $W, W'$, and consider a sequence of data $(y_i)_1^m$ and $(g_i)_1^m$ depending on some sequence index $k$ that we suppress in the notation.  We suppose that $\norm{g_i}_{\mathcal{H}_W}$ are uniformly bounded, and $\min_{i\neq j} |y_i-y_j|\to +\infty$ in the $k\to +\infty$ limit, then the norm of the minimizers satisfies
			\[
			\liminf_{k\to \infty} \left(  \norm{f_*}_{\mathcal{H}}^2- \sum_1^m \norm{g_i}_{ \mathcal{H}_W  }^2    \right) \geq 0.
			\]
		\end{cor}

		\begin{proof}
			As noted in the above proof, $(M_{ij})$ is defined for almost every $\omega_W$, and is a positive definite Hermitian matrix. By the formula (\ref{optimalinterpolationK}), the diagonal entries $M_{ii}=1$ by definition. For $i\neq j$, we have 
			$|y_i-y_j|\to +\infty$, so the \emph{Riemann-Lebesgue lemma} implies that for all the $\omega_W$ such that $k_W(\omega_W)$ is finite, we have
			\[
			\lim_{k\to \infty} M_{ij}(\omega_W)=0.
			\]
			To summarize $M_{ij}\to \delta_{ij}$ for almost every $\omega_W$.

			Consequently, the inverse matrix $M^{-1}$ also converges to the identity a.e. $\omega_W$, so that near the limit $M^{-1}\geq (1-\epsilon)I$ for any small given $\epsilon$ and fixed $\omega_W$, whence
			\[
			\liminf_{k\to +\infty} (M^{-1}_{ij} \hat{g}_i  \bar{\hat{g}}_j -\sum_1^m |\hat{g}_i|^2 (1-\epsilon)) \geq 0,\quad a.e.  \quad  \omega_W. 
			\]
			By Fatou's lemma
			\[
			\liminf_{k\to +\infty} \int_{W^*} k_W^{-1} (M^{-1}_{ij} \hat{g}_i  \bar{\hat{g}}_j -\sum_1^m |\hat{g}_i|^2 (1-\epsilon)) d\omega_W \geq 0,
			\]
			and by the norm formula for the minimizer (\ref{optimalinterpolationminimum}), we have
			\[
			\liminf_{k\to +\infty} \left(  \norm{f_*}_{\mathcal{H}}^2- (1-\epsilon)\sum_1^m \norm{g_i}_{ \mathcal{H}_W  }^2    \right)  \geq 0.
			\]
			Since $\norm{g_i}_{\mathcal{H}_W}$ are uniformly bounded in the sequence, and $\epsilon$ is arbitrarily small, we deduce the claim.
		\end{proof}

		\begin{rmk}
			The reason we only prove a one-sided inequality, rather than 
			\[
			\lim_{k\to \infty} \left(  \norm{f_*}_{\mathcal{H}}^2- \sum_1^m \norm{g_i}_{ \mathcal{H}_W  }^2    \right) = 0,
			\]
			is that $M^{-1}$ may be a priori very large for a very small subset of $\omega_W$, even though $M_{ij}$ is always uniformly bounded. This stronger conclusion can be deduced by the stronger hypothesis that $M_{ij}\to \delta_{ij}$ \emph{uniformly} for all $\omega_W$, rather than a.e $\omega_W$. Only the one-sided version will be used later.

		\end{rmk}

		\subsection{Dimensionally reduced minimization problems}

		Recall from (\ref{atomic}) that the marginal distribution may have atoms at $a_i\in W'$. We shall consider the minimization problem for functions on the affine subspaces $\{ x_{W'}=a_i\}\subset V$. Let
		\begin{equation}
			I_i(g, U_{11},\lambda)=\frac{1}{2}\E[   |Y- g(U_{11}X_W)|^2 \1_{X_{W'}=a_i} ] +\frac{1}{2}\lambda \norm{g}_{\mathcal{H}_W }^2,
		\end{equation}
		and $J_i(U_{11}, \lambda)= \inf_{g\in \mathcal{H}_W} I_i(g,U_{11},\lambda)$. As we observed before, $k_W$ inherits the same integrability as $k_V$, so in particular the Sobolev embedding into $C^0$ holds for functions in $\mathcal{H}_W$, and therefore the functional $I_i$ is well defined on the Hilbert space $\mathcal{H}_W$, and Riesz representation theorem  produces a minimizer, which we denote as $g_i^*$.

		\begin{lem}\label{lem:trivialupperboundJi}
			(Trivial upper bound) We have $0\leq J_i(U_{11},\lambda) \leq \frac{1}{2}\E[   |Y|^2 \1_{X_{W'}=a_i} ]$. 
		\end{lem}

		\begin{proof}
			Testing by the function $g=0$, we get $I_i(0, U, \lambda)\geq J_i$.
		\end{proof}

		If the equality is achieved, then $\E[Y|U_{11} X_W]=0$
		conditional on $X_{W'}=a_i$. 
		The point is that it is rare for the trivial upper bound to be achieved.

		\begin{cor} We have
			\[
			0\leq \sum_1^\infty J_i(U_{11},\lambda)\leq  \frac{1}{2}\sum_1^\infty \E[   |Y|^2 \1_{X_{W'}=a_i} ] \leq  \frac{1}{2} \E[   |Y|^2 ].
			\]
			
		\end{cor}

		The main theorem of this chapter will be
		
		\begin{thm}\label{thm:asymptoticvalue}
			Consider the sequence of $U$ tending to infinity as in section \ref{sect:minimizationproblem}. Then
			\[
			\lim_{k\to +\infty} J( U,\lambda)= \frac{1}{2}\E[|Y|^2 
			\1_{ X_{W'} \notin \{a_i\} }] +\sum_1^\infty J_i( U_{11}^{\infty},\lambda).
			\]
			where $U_{11}^\infty$ is the limit of $U_{11}$ as $k\to +\infty$.
		\end{thm}
		
		\blue{
			Theorem~\ref{thm:asymptoticvalue} shows that, in the large-$U$ limit, the limiting objective \emph{decouples} into independent contributions from the non-atomic part and from each individual atom of the marginal distribution of $X_{W'}$. The non-atomic part contributes the trivial term $\frac{1}{2}\E[|Y|^2 \1_{ X_{W'} \notin \{a_i\} }]$, whereas each atom $a_i$ contributes independently through a separate lower-dimensional kernel ridge regression problem, with contribution $J_i(U_{11}^{\infty},\lambda)$.
		}

		\subsection{Upper bound}

		In this section we prove an upper bound for the upper limit:

		\begin{thm}\label{thm:upperlimit}
			Consider the sequence of $U$ tending to infinity as in section \ref{sect:minimizationproblem}. Then
			\begin{equation}
				\limsup_{k\to +\infty} J( U,\lambda)\leq  \frac{1}{2}\E[|Y|^2 
				\1_{ X_{W'} \notin \{a_i\}_1^\infty }] +\sum_1^\infty J_i( U_{11}^\infty,\lambda).
			\end{equation}
		\end{thm}

		\begin{proof}
			For any large $m$, we construct a \emph{test function}
			\begin{equation}\label{testfunctionfm}
				f^{(m)}(x)= \sum_1^m f_i( x_W- U_{12}a_i, x_{W'}- U_{22} a_i), 
			\end{equation}
			where the function $f_i$ is defined via its Fourier transform
			\[
			\hat{f}_i(\omega)= \hat{g}_i^*(\omega_W) k_V(\omega) k_W^{-1} (\omega_W),
			\]
			and $g_i^*\in \mathcal{H}_W$ is the minimizer for the dimensionally reduced problem $J_i(U_{11}^\infty,\lambda)$. By the proof of Lemma \ref{InterpretationHW}, we know $\norm{f_i}_{\mathcal{H}}=\norm{g_i^*}_{\mathcal{H}_W}$, and in particular $f_i\in \mathcal{H}$. Since $\mathcal{H}$ is a translation invariant RKHS, this implies $f^{(m)}\in \mathcal{H}$, and we have a coarse uniform bound independent of the sequence index $k$,
			\[
			\norm{f^{(m)}}_{\mathcal{H}} \leq \sum_1^m \norm{g_i^*}_{\mathcal{H}_W}.
			\]
			
			\red{We will separately deal with the expectation part and the $\mathcal{H}$-norm parts of $I(f^{(m)}, U,\lambda)$.}
			
			\begin{lem}
				As the sequence index $k\to +\infty$,
				\[
				\E[ |f^{(m)}(UX)- Y|^2 \1_{X_{W'}=a_i}  ]\to \E[ |g_i^*(U_{11}^\infty X_W)- Y|^2 \1_{X_{W'}=a_i}  ],\quad i=1,\ldots m.
				\]
			\end{lem}
			
			\begin{proof}
				Conditional on $X_{W'}=a_i$, then
				\[
				f^{(m)}(UX)= \sum_1^m f_j( UX- (U_{12}a_j, U_{22}a_j)) = \sum_1^m f_j( U_{11}X_W +U_{12}(a_i-a_j), U_{22}(a_i-a_j)).
				\]
				In the $j=i$ case, by the proof of Lemma \ref{InterpretationHW}, we get the term
				\[
				f_i( U_{11}X_W, 0)= g_i^*(U_{11}X_W)\to g_i^*(U_{11}^\infty X_W).
				\]
				Meanwhile for $j\neq i$, by assumption $|U_{22}(a_i-a_j)|\geq \Lambda_k |a_i-a_j|\to +\infty$ as $k\to +\infty$. By Lemma \ref{lem:decayatinfinity}, for any fixed finite value of $X_W$, we get pointwise convergence
				\[
				f_j( U_{11}X_W +U_{12}(a_i-a_j), U_{22}(a_i-a_j))\to 0. 
				\]
				Thus conditional on $X_{W'}=a_i$, we get pointwise convergence as $k\to +\infty$,
				\[
				f^{(m)}(UX)\to g_i^*(U_{11}X_W).
				\]
				Moreover $\norm{f^{(m)}}_{L^\infty} \leq \norm{f^{(m)}}_{\mathcal{H}}$ is uniformly bounded, and $\E|Y|^2<+\infty$, so the dominated convergence theorem implies the lemma.
			\end{proof}

			By a very similar argument,
			
			\begin{lem}
				As the sequence index $k\to +\infty$,
				\[
				\E[ |f^{(m)}(UX)- Y|^2 \1_{X_{W'}\notin \{ a_i\} }  ]\to \E[ | Y|^2 \1_{X_{W'}\notin \{ a_1,\ldots a_m\}}  ].
				\]
			\end{lem}

			\begin{proof}
				The key difference from the previous lemma is that conditional on $X_{W'}\notin \{a_1,\ldots a_m\}$, we get pointwise convergence as $k\to +\infty$,
				\[
				f^{(m)}(UX)\to 0.
				\]
				so one concludes by dominated convergence again that
				\begin{equation}\label{trivialfitting1}
					\lim_{k\to +\infty} \E[ |f^{(m)}(UX)|^2 \1_{X_{W'}\notin \{ a_i\}_1^m }  ]= 0,
				\end{equation}
				and
				\[
				\E[ |f^{(m)}(UX)- Y|^2 \1_{X_{W'}\notin \{ a_i\}_1^m }  ]\to \E[ | Y|^2 \1_{X_{W'}\notin \{ a_1,\ldots a_m\}}  ]
				\]
			\end{proof}
			
			We now deal with the $\mathcal{H}$-norm.

			\begin{lem}
				(Asymptotic orthogonality for the test function)
				As the sequence index $k\to +\infty$,
				\[
				\norm{ f^{(m)}}_{\mathcal{H}}^2\to \sum_1^m \norm{g_i^*}_{\mathcal{H}_W}^2.
				\]
			\end{lem}

			\begin{proof}
				We compute the Fourier transform
				\[
				\hat{ f}^{(m)}= \sum_1^m \hat{f}_i e^{-2\pi i  \langle \omega, U_{12}a_i + U_{22} a_i\rangle }= \sum_1^m \hat{g}_i^* k_V(\omega) k_W^{-1}(\omega_W)  e^{-2\pi i  \langle \omega, U_{12}a_i + U_{22} a_i\rangle }.
				\]
				We now compute the $m\times m$ inner product matrix when we restrict $\mathcal{H}$ to the span of the $m$ functions $\hat{g}_i^* k_V(\omega) k_W^{-1}(\omega_W)  e^{-2\pi i  \langle \omega, U_{12}a_i + U_{22} a_i\rangle }$. The diagonal entries are just $\norm{f_i}_{\mathcal{H}}^2= \norm{g_i^*}_{\mathcal{H}_W}^2$. The off diagonal entry is
				\[
				\begin{split}
					& \int_{V^*} k_V^{-1}\hat{g}_i^* k_V(\omega) k_W^{-1}(\omega_W)  \overline{ \hat{g}_j^* k_V(\omega) k_W^{-1}(\omega_W) }   e^{-2\pi i  \langle \omega, U_{12}(a_i-a_j) + U_{22} (a_i-a_j)\rangle }   d\omega
					\\
					=& \int_{V^*} k_V k_W^{-2} \hat{g}_i^* \overline{ \hat{g}_j^*  }   e^{2\pi i  \langle \omega, U_{12}(a_i-a_j) + U_{22} (a_i-a_j)\rangle }  d\omega.
				\end{split}
				\]
				Notice by Cauchy-Schwarz,
				\[
				\int_{V^*} k_V k_W^{-2} |\hat{g}_i^* \overline{ \hat{g}_j^*  } | \leq  (\int_{V^*} k_V k_W^{-2} |\hat{g}_i^*|^2 )^{1/2} (\int_{V^*} k_V k_W^{-2} |\hat{g}_j^*|^2)^{1/2}= \norm{g_i^*}_{\mathcal{H}_W}\norm{g_j^*}_{\mathcal{H}_W},
				\]
				and in particular the function $k_V k_W^{-2} \hat{g}_i^* \overline{ \hat{g}_j^*  }$ is $L^1$. By Riemann-Lebesgue, and the fact that as $k\to +\infty$
				\[
				|U_{22}(a_i-a_j)|\geq \Lambda_k |a_i-a_j|\to +\infty ,
				\]
				we deduce that the off diagonal entries converge to zero, which means the $m$ functions are asymptotically $\mathcal{H}$-orthogonal. This proves the lemma.
			\end{proof}
			
			\red{We return to the proof of the theorem.}		Combining the three lemmas,
			\[
			I(f^{(m)}, U, \lambda) \to \frac{1}{2}\sum_1^m ( \E[ |g_i^*(U_{11}^\infty X_W)- Y|^2 \1_{X_{W'}=a_i}  ]+\lambda \norm{g_i^*}_{\mathcal{H}_W}^2)+ \frac{1}{2}\E[ | Y|^2 \1_{X_{W'}\notin \{ a_i\}}  ],
			\]
			namely
			\begin{equation}\label{eqn:upperlimitproof}
				I(f^{(m)}, U, \lambda) \to \frac{1}{2}\E[|Y|^2 
				\1_{ X_{W'} \notin \{a_1,\ldots a_m\} }] +\sum_1^m J_i( U_{11}^\infty,\lambda).
			\end{equation}
			Since $J(U,\lambda)\leq I(f^{(m)}, U, \lambda)$, and $J_i\geq 0$, we deduce
			\[
			\begin{split}
				&
				\limsup J(U,\lambda) \leq \frac{1}{2}\E[|Y|^2 
				\1_{ X_{W'} \notin \{a_1,\ldots a_m\} }] +\sum_1^m J_i( U_{11}^\infty,\lambda) 
				\\
				& \leq \frac{1}{2}\E[|Y|^2 
				\1_{ X_{W'} \notin \{a_1,\ldots a_m\} }] +\sum_1^\infty J_i( U_{11}^\infty,\lambda).   
			\end{split}
			\]
			Since $\E|Y|^2<+\infty$, we can take the limit $m\to +\infty$ to deduce
			\begin{equation*}
				\limsup J(U,\lambda) \leq \frac{1}{2}\E[|Y|^2 
				\1_{ X_{W'} \notin \{a_i\}_1^\infty }] +\sum_1^\infty J_i( U_{11}^\infty,\lambda),
			\end{equation*}
			as required.
		\end{proof}

		The point of this theorem is that when there is no discrete part, then we just recover the trivial bound $\frac{1}{2} \E[|Y|^2]$, but when there is a discrete part, then we get an \emph{improved upper bound}, which crucially relies on some \emph{asymptotic orthogonality} property when the points $U_{22}a_i$ are far separated in $W'$.

		\subsection{Lower bound}

		In this section we prove a lower bound for the lower limit:

		\begin{thm}
			Consider the sequence of $U$ tending to infinity as in section \ref{sect:minimizationproblem}. Then
			\begin{equation}
				\liminf_{k\to +\infty} J( U,\lambda)\geq  \frac{1}{2}\E[|Y|^2 
				\1_{ X_{W'} \notin \{a_i\}_1^\infty }] +\sum_1^\infty J_i( U_{11}^\infty,\lambda).
			\end{equation}
		\end{thm}

		Since the upper bound and the lower bound match exactly, we deduce that the limit exists and is given by the RHS expression. This would prove Theorem \ref{thm:asymptoticvalue}.

		\begin{proof}
			For each $U$ in the sequence, recall $f_U\in \mathcal{H}$ is the minimizer for $I(f,U,\lambda)$. By the trivial upper bound Lemma  \ref{lem:trivialupperbound} and Sobolev embedding,
			\[
			\lambda\norm{f_U}_{C^0}^2 \leq \E[|Y-f_U(UX)|^2]+ \lambda\norm{f_U}_{\mathcal{H}}^2\leq \E[|Y|^2].
			\]

			We will need to deal with the discrete part and the non-atomic part. For any $m$, we denote $\chi_m= \1_{ X_{W'}\in \{ a_1,\ldots a_m\} }$, and $\chi=\1_{ X_{W'}\in \{ a_i\}_1^\infty 
			}$. We notice
			\[
			\E[ |Y|^2 \chi_m  ] \to   \E[|Y|^2\chi] ,\quad m\to +\infty. 
			\]
			The characteristic function $\chi_m$ allows us to focus on the \emph{atoms with significant $L^2$-mass contributions}.

			\begin{lem}
				For fixed $m$, as the sequence index $k\to +\infty$,
				\[
				\liminf \frac{1}{2}\E[ |f_U(UX)-Y|^2 \chi_m ]+ \frac{1}{2}\lambda \norm{f_U}_{\mathcal{H}}^2 \geq \sum_1^m J_i(U_{11}^\infty,\lambda).
				\]
			\end{lem}

			\begin{proof}
				We define the restrictions $g_i(x_W) = f_U(x_W, U_{22}a_i)$. By Lemma \ref{InterpretationHW}, and the translation invariance of the RKHS, we obtain the uniform bounds
				\[
				\norm{g_i}_{\mathcal{H}_W}  \leq \norm{f_U}_{\mathcal{H}} \leq \lambda^{-1/2} \E[|Y|^2]^{1/2}. 
				\]
				By definition $f_U\in \mathcal{H}$ interpolates the functions $g_i$ at the affine subspaces $U_{22}a_i+W$, so its $\mathcal{H} $-norm is bounded below by the optimal interpolation solution in Lemma \ref{Optimalinterpolationlem}. We observe that the distances between the $m$ interpolation points $U_{22}a_i\in W'$ are given by $|U_{22}(a_i-a_j)|\to +\infty$ in the limit, so by 
				the asymptotic orthogonality statement in Corollary \ref{cor:asymptoticorthogonalityoptimalinterpolation}, 
				\begin{equation}
					\liminf \left( \norm{f_U}_{ \mathcal{H} }^2 -\sum_1^m \norm{g_i}_{\mathcal{H}_W}^2  \right) 
					\geq \liminf \left( \norm{f_*}_{ \mathcal{H} }^2 -\sum_1^m \norm{g_i}_{\mathcal{H}_W}^2  \right) \geq 0.
				\end{equation}
				On the other hand, conditional on $X_{W'}=a_i$,
				\[
				f_U(UX)= f_U(U_{11}X_W+ U_{12} a_i, U_{22} a_i)= g_i( U_{11}X_W+ U_{12} a_i) .
				\]
				Thus
				\[
				\begin{split}
					& \liminf   \frac{1}{2}  \E[ |f_U(UX)-Y|^2 \chi_m ]+ \frac{1}{2}   \lambda \norm{f_U}_{\mathcal{H}}^2 
					\\
					= & \liminf  \frac{1}{2}  \sum_1^m \E[ |g_i(U_{11}X+U_{12}a_i)-Y|^2 \1_{X_{W'}=a_i} ]+ \frac{1}{2}   \lambda \norm{f_U}_{\mathcal{H}}^2
					\\
					\geq & \liminf  \frac{1}{2}  \sum_1^m \left(\E[ |g_i(U_{11}X+U_{12}a_i)-Y|^2 \1_{X_{W'}=a_i} ]+  \frac{1}{2}  \lambda \norm{g_i}_{\mathcal{H}_W}^2 \right)
					\\
					= & \liminf \sum_1^m J_i(U_{11}, \lambda).
				\end{split}
				\]
				Here the last equality uses the definition of the dimensionally reduced problems $J_i$, and the translation invariance of the $\mathcal{H}_W$ norm.

				By applying Lemma \ref{lem:continuityJ} to the dimensionally reduced minimization problem $J_i$, we see 
				\[
				\liminf  J_i(U_{11}, \lambda) = J_i(U_{11}^\infty, \lambda).
				\]
				Combining the above proves the lemma.
			\end{proof}

			We turn the attention to the non-atomic part.
			
			\begin{lem}
				As the sequence index $k\to +\infty$,
				\[
				\liminf \E[ |f_U(UX)-Y|^2 (1-\chi) ] \geq  \E[ | Y|^2 (1-\chi) ].
				\]
			\end{lem}

			\begin{proof}
				It suffices to prove that the cross term
				\[
				\lim \text{Re} \E[ f_U(UX) \bar{Y} (1-\chi)]=0.
				\]
				By the reproducing kernel property (\ref{reproducingformula}),
				\[
				f_U(UX)= ( f_U, K(UX,\cdot))_{\mathcal{H}},
				\]
				hence by Fubini theorem
				\[
				\E[ f_U(UX) \bar{Y} (1-\chi)]= \E[ ( f_U, K(UX,\cdot))_{\mathcal{H}} \bar{Y} (1-\chi)]=( f_U, \E[  K(UX,\cdot) \bar{Y} (1-\chi)])_{\mathcal{H}}.
				\]
				Here we notice that $\E[  K(UX,\cdot) \bar{Y} (1-\chi)]$ is a weighted average of kernel functions, and in particular defines an element of $\mathcal{H}$. 
				By Cauchy-Schwarz, and the fact that $\norm{f}_{\mathcal{H}}\leq \lambda^{-1/2} \E[|Y|^2]^{1/2}$, we get
				\begin{equation}
					|\E[ f_U(UX) Z]|\leq \lambda^{-1/2} \E[|Y|^2]^{1/2}\norm{ \E[  K(UX,\cdot) \bar{Y} (1-\chi)]   }_{ \mathcal{H} }.
				\end{equation}
				where we denote $Z=\bar{Y} (1-\chi)$.

				We now evaluate the RHS norm. We  take an independent double copy $(X',Z')$ which has the same law as $(X,Z)$. Then by Fubini theorem again, 
				\[
				\begin{split}
					& \norm{ \E[  K(UX,\cdot) Z]   }_{ \mathcal{H} }^2= ( \E[  K(UX,\cdot) Z], \E[  K(UX',\cdot) Z'] )_{\mathcal{H}}
					\\
					= & \E[   Z \bar{Z}'    ( K(UX,\cdot),   K(UX',\cdot)  )_{\mathcal{H}} ]
					\\
					=&  \E[   Z \bar{Z}'     K(UX,UX') ]= \E[   Z \bar{Z}'  K(U(X-X')) ].
				\end{split}
				\]
				Here the last line uses the reproducing kernel property.

				By our assumptions on the matrix $U$,
				\[
				\begin{split}
					&  |U(X-X')|=|(U_{11}(X_W-X_{W}') + U_{12}(X_W-X_W') , U_{22}(X_{W'}-X_{W'}') )|    
					\\
					\geq & |U_{22}(X_{W'}-X_{W'}') )| \geq \Lambda_k| X_{W'}-X_{W'}'|,\quad \Lambda_k\to +\infty. 
				\end{split}
				\]
				Thus conditional on the \emph{off diagonal} condition $X_{W'}\neq X_{W'}'$,  by the decay of the kernel function $K$ at infinity, we have the pointwise convergence
				\[
				K(U(X-X'))\to 0,\quad k\to +\infty.
				\]
				By the dominated convergence theorem,
				\[
				\lim \E[   Z \bar{Z}'  K(U(X-X')) \1_{ X_{W'}\neq X_{W'}' } ] =0. 
				\]
				On the diagonal ${ X_{W'}= X_{W'}' } $ of the product measure space, the measure is supported on the atoms of the marginal distribution of $X_{W'}$. But the factor $1-\chi$ is precisely engineered so that $Z=Z'=0$ on the atoms, so 
				\[
				\lim \E[   Z \bar{Z}'  K(U(X-X')) ] =0. 
				\]
				whence
				\[
				|\E[ f(UX) Z]|\leq \lambda^{-1/2} \E[|Y|^2]^{1/2}\norm{ \E[  K(UX,\cdot) \bar{Y} (1-\chi)]   }_{ \mathcal{H} }\to 0,\quad k\to +\infty. 
				\]
				This proves the lemma.
			\end{proof}

			Combining the two lemmas, for arbitrarily large fixed $m$,
			\[
			\liminf I(f,U, \lambda) \geq \sum_1^m J_i(U_{11}^\infty) + \blue{\frac{1}{2}} \E[ | Y|^2 (1-\chi) ]. 
			\]
			But by the trivial upper bound lemma \ref{lem:trivialupperboundJi},
			\[
			0\leq \sum_{m+1}^\infty J_i(U_{11}^\infty,\lambda) \leq \frac{1}{2} \E[ |Y|^2(\chi-\chi_m)  ] \to 0,\quad m\to +\infty. 
			\]
			Thus taking the $m\to +\infty$ limit, we see
			\[
			\liminf_{k\to +\infty} I(f,U, \lambda) \geq \sum_1^\infty J_i(U_{11}^\infty) +  \blue{\frac{1}{2}}  \E[ | Y|^2 (1-\chi) ],
			\]
			so the theorem is proved.
		\end{proof}

		\begin{rmk}\label{rmk:lowerlimitthm}
			It is instructive to reflect on how to make the convergence into an effective estimate for finite but large $U$, once an explicit kernel function is given. The number $m$ comes from selecting the atoms where $Y$ has \emph{significant $L^2$-mass} contributions, and the rest of the atoms are negligible as long as they carry a small total $L^2$-mass for $Y$. The \emph{asymptotic orthogonality} of the optimal interpolation problem is perhaps the most delicate ingredient, and the essential estimate comes from the decay rate of the off-diagonal terms in the matrix $M_{ij}$ in Lemma \ref{Optimalinterpolationlem}. Finally, to control the term  $\E[   Z \bar{Z}'  K(U(X-X')) ]$, we need to compute the \emph{decay rate of the kernel function} $K$ so that the integral is localized to a small neighbourhood of the diagonal $X_{W'}=X'_{W'}$. Moreover, we also need the $L^2$ norm of $Z$ to be small in this neighbourhood. The main enemy has to do with the \emph{local concentration of $L^2$ mass} for $Z$, or in other words the measure $\E[|Z|^2  |X_{W'}] \mu_{X_{W'}}$ behaves like an approximate delta function on $W'$.

		\end{rmk}

		\subsection{Asymptotic nature of the minimizer}

		Recall we constructed a test function $f^{(m)}$ (\cf (\ref{testfunctionfm})) in the proof of Theorem \ref{thm:upperlimit}. We can in fact deduce that it is a \emph{good approximation to the true minimizer}, here denoted as $f_U$:

		\begin{thm}\label{thm:asymptoticapproximateminimizer}
			(Approximate minimizer)
			We have
			\[ 
			\lim_{k\to +\infty} \E[|(f^{(m)}-f_U)(UX)|^2] +  \lambda \norm{ f^{(m)}-f_U  }_{\mathcal{H}}^2 \leq \E[|Y|^2 
			\1_{ X_{W'} \notin \{a_i\}_{m+1}^\infty }] .
			\]
			In particular 
			\[
			\lim_{m\to +\infty}\lim_{k\to +\infty} \E[|(f^{(m)}-f_U)(UX)|^2] +  \lambda \norm{ f^{(m)}-f_U  }_{\mathcal{H}}^2 =0.
			\]
		\end{thm}

		\begin{proof}
			The functional $I(f,U,\lambda)$ is quadratic in $f$ and minimized at $f_U$, with second order term in $f$ being
			\[
			Q(f)=\frac{1}{2} \E[|f(UX)|^2] + \frac{1}{2} \lambda \norm{f}_{\mathcal{H}}^2. 
			\]
			By completing the squares,  
			this implies that
			\[
			I(f,U,\lambda)= Q( f- f_U) + I(f_U, U,\lambda)=  Q( f- f_U) + J(U,\lambda).
			\]
			We recall from (\ref{eqn:upperlimitproof}) that as $k\to +\infty$, 
			\[
			I(f^{(m)}, U, \lambda) \to \frac{1}{2}\E[|Y|^2 
			\1_{ X_{W'} \notin \{a_1,\ldots a_m\} }] +\sum_1^m J_i( U_{11}^\infty,\lambda).
			\]
			By Theorem \ref{thm:asymptoticvalue},
			\[
			J(U,\lambda)\to  \frac{1}{2}\E[|Y|^2 
			\1_{ X_{W'} \notin \{a_i\}_1^\infty }] +\sum_1^\infty J_i( U_{11}^\infty,\lambda).
			\]
			Thus
			\[
			Q(f^{(m)}-f_U)= I(f^{(m)}, U, \lambda)-J(U,\lambda) \to \frac{1}{2}\E[|Y|^2 
			\1_{ X_{W'} \notin \{a_i\}_{m+1}^\infty }] -\sum_{m+1}^\infty J_i( U_{11}^\infty,\lambda).
			\]
			In particular
			\[
			\lim_{k\to +\infty} Q(f^{(m)}-f_U) \leq \frac{1}{2}\E[|Y|^2 
			\1_{ X_{W'} \notin \{a_i\}_{m+1}^\infty }]
			\]
			as required.
		\end{proof}

		\begin{rmk}
			In particular, if $X$ is a continuous variable, then asymptotically as $U\to \infty$, 
			\[
			\lim_{U\to \infty} (\E[|f_U(UX)|^2] +  \lambda \norm{ f_U  }_{\mathcal{H}}^2) =0 .
			\]
			We shall recover this result by functional analytic methods in Corollary \ref{Cor:quantitativenonfitting} below.
		\end{rmk}

		\begin{rmk}
			Recall that $f^{(m)}$ is built from $m$ local contributions coming from the minimizers of the dimensionally reduced problems associated to the distinct points $a_i\in W$, which become $\mathcal{H}$-orthogonal in the limit. The point of the theorem is that in the limit the minimizer essentially decouples into the solutions of these dimensionally reduced problems. 
		\end{rmk}


		\begin{cor}\label{cor:asymptotictrivialfitting}
			(Asymptotic trivial fitting for the non-atomic part)
			We have
			\[
			\lim_{k\to +\infty} \E[|f_U(UX)|^2 \1_{X_{W'} \notin \{a_i\}_{1}^\infty }  ]  =0.
			\]
		\end{cor}

		\begin{proof}
			The first statement in Theorem \ref{thm:asymptoticapproximateminimizer} implies
			\[
			\lim_{k\to +\infty} \E[|(f^{(m)}-f_U)(UX)|^2  \1_{ X_{W'} \notin \{a_i\}_{1}^\infty } ]  \leq \E[|Y|^2 
			\1_{ X_{W'} \notin \{a_i\}_{m+1}^\infty }] .
			\]
			Recall from (\ref{trivialfitting1}) that 
			\[
			\lim_{k\to +\infty} \E[ |f^{(m)}(UX)|^2 \1_{X_{W'}\notin \{ a_i\}_1^m }  ]= 0,
			\]
			hence for any fixed $m$,
			\[
			\lim_{k\to +\infty} \E[|f_U(UX)|^2  \1_{ X_{W'} \notin \{a_i\}_{1}^\infty } ]  \leq \E[|Y|^2 
			\1_{ X_{W'} \notin \{a_i\}_{m+1}^\infty }] .
			\]
			Let $m\to +\infty$, we deduce the claim.
		\end{proof}

		\subsection{Discussions}

		We offer a few  interpretations on the Theorems.

		\begin{enumerate}
			\item   Perhaps the most striking aspect of Theorem \ref{thm:asymptoticapproximateminimizer} is that there is a sharp distinction between the \emph{discrete} and the \emph{continuous variables} $X$ in the large $U$ limit.

			This has a very intuitive explanation. Suppose $Y$ concerns some question which involves several different cases, such as the earth vs. some other planet, then one would expect that there should be some \emph{discrete variable} that distinguishes these cases very sharply; in  other words, the two cases fall into two very distinct \emph{conceptual categories}.

			\item 
			
			In Theorem \ref{thm:asymptoticapproximateminimizer}, we saw that in the asymptotic limit, the minimizer decouples into dimensionally reduced problems with $\mathcal{H}$-orthogonality property.

			When humans approach a question involving several very different cases, they would naturally first make a discrete classification, and then analyze the problem separately in each case. This is reflected in the \emph{dimensional reduction} procedure, where one first looks for the discrete signal variables (in our case the points $a_i\in W'$), and then solve the sub-problem for each fixed $a_i$, which could involve more discrete or continuous variables.

			In two very different cases, the answers to a given question are intuitively expected to be \emph{independent}. The \emph{orthogonality} property may be seen as a mathematical manifestation of this independence.

			\item Our limit as $U\to +\infty$ is an idealization. Along unit vectors in $W$, the magnitude $|Ux|$ stays bounded, while for unit vectors in $W'$, the magnitude $|Ux|$ goes to infinity.

			In practice, if one \emph{optimize over all the choices of $U$}, the minimizer for $\min_U J(U,\lambda)$ may fall in some finite region in the space of matrices $\End(V)$, and there is no general reason for the minimizer to be unique. What one could more realistically hope for, is that $|Ux|$ is much smaller for unit directions $x$ in some subspace $W$ of $V$, compared to unit directions along some subspace $W'$ transverse to $W$. Such a situation may be called \emph{scale separation}.
			The intuition is that in a given problem, some variables are more important than others.

			\item 
			Suppose the reproducing kernel is \emph{rotationally invariant}, namely $K(x)$ is a function of $|x|$, which holds for most kernels used in practice. Then
			\[
			I(f, U, \lambda)= I(f\circ A, A^{-1}U, \lambda),\quad \forall A\in O(V)\simeq O(d),
			\]
			so $J(U, \lambda)=J(A^{-1}U, \lambda)$ depends only on the group invariant quantity $\Sigma= U^TU$ and $\lambda$, where $U^T$ is the adjoint of $U$ with respect to the inner product on $V$. Now $\Sigma$ defines a self-adjoint operator on $V$, so one can find an orthonormal \emph{eigenvector basis}, with eigenvalues $\lambda_i$ for $i=1,\ldots d$.

			The scale separation phenomenon above then has the more concrete meaning that some eigenvalues may be much smaller than the others. Intuitively, the large eigenvalues suggest that the eigenspaces are likely to carry the more important information. Collecting together the small vs. large eigenvalue subspaces would give a reasonable way to identify $W, W'$ in the scale separation phenomenon. It is worth noting that our standing assumption $U_{21}=0$ holds automatically if $W,W'$ are orthogonal, which follows from the eigenspace decomposition.

			Moreover, when $d$ is large, it may happen that the eigenvalues exhibit some \emph{multi-scale hierarchy}. This means that the eigenvalues can be partitioned into subsets, such that eigenvalues in the same subset are of comparable magnitudes, but between different subsets the eigenvalues are separated by many orders of magnitudes:
			\[
			\lambda_1\leq \ldots \leq \lambda_{n_1} \ll \lambda_{n_1+1} \leq \ldots \leq \lambda_{n_2} \ll \lambda_{n_2+1} \leq \ldots \leq \lambda_d. 
			\]
			Intuitively, this reflects the common sense that conceptual categories can contain subcategories and exhibit hierarchical structure.

			\item  
			The discrete/non-atomic dichotomy is only sharp in the $U\to +\infty$ limit. Suppose instead that $U$ is finite, but exhibits scale separation, and we want to understand $J(U,\lambda)$. Then we expect that approximate delta functions would have similar effect to the delta functions. In other words, the substitute for the atoms, is \emph{local concentration} of the 
			probability density of the marginal distribution of $X$ on $W'\simeq V/W$.


			\item 
			We have imposed the \emph{translation invariant RKHS} as the starting point. In fact Corollary \ref{cor:asymptotictrivialfitting} gives a hint why this \emph{may not always produce a good global fit} for the random variable $Y$: for $k\to +\infty$, the entire $L^2$-mass of $Y \1_{ X_{W'}\notin \{ a_i\}} $ is invisible to the minimizer $f_U$.

			The relevance of this issue can be seen from the following thought experiment involving two distinct scales:

			\begin{eg}\label{eg:multiscale1}
				Let $V=\R$, and $X_0\in \{ \pm 1\} $ be a Bernoulli variable, which is not directly observable. Let $X\sim N(0,1)$ conditional on $X_0=1$, and $X\sim N(0, \sigma^2)$ conditional on $X_0=-1$, where $\sigma\ll 1$. Let 
				\[
				Y= \sin(X)   + \sin( \frac{X}{\sigma})  e^{-|X|^2/\sigma^2}. 
				\]
				The density of the measure $\E[|Y|^2|X]\mu_X$ is the superposition of a smooth part  and a concentrated cluster around the origin.
				We suppose $\lambda$ is small but fixed, and $\mathcal{H}$ is one of the standard Sobolev spaces $W^{k,2}$ where $k\geq 1$.

				The point is that if we want to fit the  $\sin( \frac{X}{\sigma}) e^{- ( \frac{X}{\sigma} )^2}$ part using a function of the form $f(UX)$ where $\norm{f}_{\mathcal{H}}\lesssim 1$ is not very large, then the parameter $U$ had better be of order comparable to $\sigma^{-1}\gg 1$. But then Corollary \ref{cor:asymptotictrivialfitting} suggests that the information contained in the $\sin(X)$ part would be difficult to detect. 
				
				Instead, a more reasonable hope is to decompose $Y$ into a sum by multiplying by a suitable partition of unity sensitive to the scale parameters inherent in the problem, and try to find a good fit for each summand of $Y$ \emph{separately}. The parameters $U$ in each sub-problem could be very different.

			\end{eg}
			
		\end{enumerate}

		\section{Rotationally invariant kernels}\label{sect:rotationallyinvkernels}

		\subsection{Rotationally invariant kernels}

		
		Given a fixed inner product $|\cdot|_V$ on $V$,  each matrix $U\in \End(V)$ has an adjoint operator $U^T$, and $\Sigma=U^TU$ defines a (positive semi-definite) form on $V$, and $\Sigma^{-1}$ defines the \emph{dual norm} on $V^*$, namely
		\[
		|x|_\Sigma= |Ux|_V=(x, U^T Ux)^{1/2},\quad 
		|\omega|_{\Sigma^{-1}}=\max_{ |x|_\Sigma=1 } \langle \omega, x\rangle= |(U^{-1})^T \omega|.
		\]
		We will be interested in varying the choice of $\Sigma$.
		We now require the RKHS kernel to be \emph{rotationally invariant} with respect to the fixed inner product $|\cdot|_V$ on $V$, meaning that 
		$
		k_V( \omega)= k(|\omega|_V^2) 
		$
		for some one-variable function $k$, and  correspondingly $K(x)= \mathcal{K}(|x|^2_V)$ for some one-variable function $\mathcal{K}:\R_{\geq 0}\to \R_{>0}$. Thus $J(U,\lambda)$ is invariant under the rotation group action, and depends only on $\Sigma, \lambda$, instead of depending on $U$.

		\begin{lem}\label{lem:Jrotationalinv}
			When $\Sigma$ is nondegenerate, we define the functional
			\[
			\mathcal{I}(F, \Sigma, \lambda)=  \frac{1}{2} \E[|Y-F(X)|^2] + \frac{\lambda}{2} \norm{F}_{\mathcal{H}_\Sigma}^2,
			\]
			and
			\begin{equation}\label{minimizerJSigma}
				\mathcal{J}(\Sigma; \lambda):=\min_{F\in \mathcal{H}_\Sigma }  \frac{1}{2} \E[|Y-F(X)|^2] + \frac{\lambda}{2} \norm{F}_{\mathcal{H}_\Sigma}^2,
			\end{equation}
			where 
			\begin{equation}
				\quad \norm{F}_{\mathcal{H}_\Sigma}^2
				=\sqrt{\det\Sigma}\int_{V^*} \frac{|\hat{F}|^2}{ k(|\omega|_{\Sigma^{-1}}^2) } d\omega.
			\end{equation}
			Given the relation $\Sigma=U^T U$, then $J(U,\lambda)=\mathcal{J}(\Sigma; \lambda)$, and the minimizer $F_\Sigma$ of $\mathcal{I}(F,\Sigma,\lambda)$ satisfies $F_\Sigma(x)=f_U(Ux)$.

		\end{lem}

		\begin{proof}
			By the scaling property of Fourier transforms
			\[
			\widehat{f\circ U}(\omega)= \frac{1}{|\det U|} \hat{f}((U^{-1})^T \omega), 
			\]
			whence one can deduce by a change of variables that
			\[
			\norm{f}_{\mathcal{H}}^2= \norm{f\circ U}_{\mathcal{H}_\Sigma}^2.
			\]
			Thus
			\[
			I(f, U, \lambda)= \frac{1}{2} \E[|Y-(f\circ U)(X)|^2] + \frac{\lambda}{2} \norm{f\circ U}_{\mathcal{H}_\Sigma}^2.
			\]
			minimizing over all $f\in \mathcal{H}$ is equivalent to minimizing over all $F=f\circ U\in \mathcal{H}_\Sigma$, hence the result.
		\end{proof}

		\begin{rmk}
			The definition of $\mathcal{J}$ only uses the inner product $\Sigma$ and $\lambda$, but not the background inner product $|\cdot|_V$. This auxiliary background inner product $|\cdot|_V$ is only needed when we try to convert between $U$ and $\Sigma$. 
		\end{rmk}

		\subsection{First variation formula}

		\begin{Notation}
			We denote $Sym^2_+\subset Sym^2 V^*$ as the open subset consisting of positive definite inner products on $V$. 
		\end{Notation}
		
		In this section we study the first variation of $\mathcal{J}$ as $\Sigma$ varies in $Sym^2_+$. This is a 1-form on $Sym^2_+$, which can be identified with a function valued in $Sym^2 V$. The goal is to prove the following formula, which is \cite[Lemma 2.8]{CLLR} with more relaxed assumptions.

		\begin{thm}(First variation)\label{thm:firstvariation}
			We denote $r_\Sigma(X,Y)=Y-F_\Sigma(X)= Y-f_U(UX)$, and let $(X',Y')$ be an independent copy of $(X,Y)$. 
			We suppose $\E[|Y|^2]<+\infty$, and the kernel function is $K(x)=\mathcal{K}(|x|^2_V)$ for some radial function
			$\mathcal{K}:\R_{\geq 0}\to \R_+$, such that 
			\begin{enumerate}
				\item The derivative $\mathcal{K}'(r)$ is continuous on $\R_+$, and $r\mathcal{K}'(r)\to 0$ as $r\to 0$;
				
				\item One of the following conditions hold: either $|\mathcal{K}'(r)|r$ is bounded as $r\to +\infty$, or $|\mathcal{K}'(r)|$ is monotone decreasing for large enough $r$; 
				
				\item The integrability hypothesis holds: for all $\Sigma$ within any given compact subset of $Sym^2_+$, we have
				\[
				\sup_\Sigma \E[ |\mathcal{K}'(|X-X'|^2_\Sigma)|^2 |X-X'|_V^4   ] <+\infty.
				\]
			\end{enumerate}
			Then at any $\Sigma\in Sym^2_+$, 
			\begin{equation}\label{eqn:Jfirstvariation3}
				\begin{split}
					D_\Sigma \mathcal{J}(\Sigma, \lambda)= &   -\frac{1}{2\lambda}   \E[  r_\Sigma(X,Y) \overline{ r_\Sigma(X',Y')}  D_\Sigma \mathcal{K}(|X-X'|_{\Sigma}^2) ] 
					\\
					= &   -\frac{1}{2\lambda}   \E[  r_\Sigma(X,Y) \overline{r_\Sigma(X',Y')}  \mathcal{K}'(|X-X'|_{\Sigma}^2) (X-X')\otimes (X-X') ] .
				\end{split}
			\end{equation}
			
		\end{thm}

		\red{This formula shows that the variation of $\mathcal{J}$ with respect to $\Sigma$ is governed by pairwise residual correlations weighted by $D_\Sigma\mathcal{K}$. In particular, when the residual $r_\Sigma$ is small, namely $F_\Sigma(X)$ is a good fit for $Y$, then one expects the first variation of $\mathcal{J}$ to be small, as $\mathcal{J}$ is close to its minimum. }

		\begin{rmk}
			We note that the underlying function spaces $\mathcal{H}_\Sigma$ are the same for all $\Sigma$, if and only if the weight function satisfies
			\[
			C(\Sigma)^{-1} k_V(|\omega|^2_V)\leq k_V(|\omega|^2_\Sigma) \leq C(\Sigma) k_V(|\omega|^2_V)
			\]
			with constant uniform in $\omega\in V^*$. This condition holds for the Sobolev kernel, but not for the Gaussian kernel 
			(See Section \ref{sect:completelymontonekernel}). The first variation formula holds for both kernels.
			
		\end{rmk}

		\begin{rmk}
			The first variation is naturally a 1-form. To define the gradient of $\mathcal{J}$ as a vector field on $Sym^2_+$, there needs to be an additional choice of a Riemannian metric on $Sym^2_+$. This issue is fundamental to the companion paper \cite{RuanLi}.
		\end{rmk}

		\begin{rmk}\label{rmk:K'integrability}
			Since
			$
			D_\Sigma \mathcal{K}(|x|^2_\Sigma)= \mathcal{K}'(|x|^2_\Sigma) x\otimes x,
			$
			the hypothesis that $\mathcal{K}'$ is $C^1$ and $r\mathcal{K}'(r)=o(1)$ at the origin, ensures that $D_\Sigma \mathcal{K}(|x|^2_\Sigma)$ is continuous at all $x\in V$. 
		\end{rmk}
		
		\red{
			\begin{rmk}
				\blue{It is an interesting open question to develop a second-order variational analysis of the population objective. In particular, it would be desirable to find an interpretable criterion, in terms of the distribution of $(X,Y)$, for determining when a critical point $\Sigma$ has a strictly positive definite Hessian.}
			\end{rmk}
		}

		\subsubsection{First proof: integral equation}

		We shall prove Theorem \ref{thm:firstvariation} under an additional technical hypothesis: at any $\Sigma\in Sym^2_+$, for $\Sigma'$ close to $\Sigma$, we have
		\[
		\E[ |\mathcal{K}(|X-X'|^2_{\Sigma'})-\mathcal{K}(|X-X'|^2_{\Sigma}) - \langle D_\Sigma \mathcal{K}( |X-X'|^2_{\Sigma}), \Sigma'-\Sigma\rangle |^2   ]= o(|\Sigma'-\Sigma|).
		\]
		This is slightly stronger than the integrability hypothesis in Theorem  \ref{thm:firstvariation}.

		\begin{prop}\label{prop:differentiabilityminimizer}
			(Differentiability in $\Sigma$ for the minimizer) We assume the setting of Theorem \ref{thm:firstvariation}. Then for $\Sigma\in Sym^2_+$ and $\Sigma'$ close to $\Sigma$,
			\[
			\norm{ F_{\Sigma'}- F_\Sigma- \langle D_\Sigma F_\Sigma , \Sigma'-\Sigma \rangle }_\E =o(|\Sigma'-\Sigma|).
			\]
			Here $
			\norm{f}_\E= \E[|f(X)|^2]^{1/2}$, and 
			the differential $D_\Sigma F_\Sigma$ is a function valued in $Sym^2 V$, characterized as the unique finite norm solution to the integral equation
			\begin{equation}\label{eqn:minimizerdifferentiability2}
				f(x)+  \frac{1}{\lambda}  \E[ f(X)\mathcal{K}(|X-x|_\Sigma^2)   ] = \frac{1}{\lambda}  \E[ r_\Sigma(X,Y) D_\Sigma \mathcal{K}(|x-X|_{\Sigma}^2) ].
			\end{equation}

		\end{prop}

		\begin{proof}
			The above integral equation (\ref{eqn:minimizerdifferentiability2}) has a unique finite solution by the integrability hypothesis on $\mathcal{K}$, and the claim in the proof of Proposition  \ref{prop:integraleqn1}. We denote the solution as $D_\Sigma F_\Sigma$.

			We recall from (\ref{eqn:integraleqndifferentiability}) that the following integral equation holds:
			\[
			(F_{\Sigma'}- F_{\Sigma})(x)+  \frac{1}{\lambda}  \E[ (F_{\Sigma'}- F_{\Sigma})(X) \mathcal{K}(|X-x|_\Sigma^2)   ] = \frac{1}{\lambda}  \E[ (Y-F_{\Sigma'}(X) )\overline{g(x-X)}  ],
			\]
			where $g(x)= \mathcal{K}(|x|_{\Sigma'}^2)- \mathcal{K}(|x|_\Sigma^2)$ is a real valued function. This implies an integral equation
			\[
			f(x)+  \frac{1}{\lambda}  \E[ f(X)\mathcal{K}(|X-x|_\Sigma^2)   ] = \frac{1}{\lambda}  \E[ h(x,X)  ],
			\]
			where 
			\[
			\begin{cases}
				f= & F_{\Sigma'}- F_\Sigma- \langle D_\Sigma F_\Sigma ,
				\Sigma'-\Sigma \rangle , 
				\\
				h(x,X)= & r_\Sigma(X,Y) ( \mathcal{K}(|X-x|^2_{\Sigma'})-\mathcal{K}(|X-x|^2_{\Sigma}) - \langle D_\Sigma \mathcal{K}( |X-x|^2_{\Sigma}), \Sigma'-\Sigma\rangle)\\
				& + (F_\Sigma- F_{\Sigma'})(X) \overline{g(x-X)} .
			\end{cases}
			\]
			Using the technical hypothesis and the quantitative continuity of $F_\Sigma$ in Lemma \ref{lem:Jminimizercontinuity1}, we deduce $\E[ |h(X',X)|^2]=o(|\Sigma'-\Sigma|)$, so the result follows from the coercivity
			estimate (\ref{eqn:coercivity}).
		\end{proof}

		We now prove Theorem \ref{thm:firstvariation}. Recall from (\ref{eqn:Jalternativeformula}) the formula
		\[
		\mathcal{J}(\Sigma;\lambda)=  \frac{1}{2} \E[|Y|^2]- \frac{1}{2}\E[ \bar{Y}F_\Sigma(X) ],
		\]
		so $\mathcal{J}(\Sigma;\lambda)$ is differentiable at $\Sigma\in Sym^2_+$, with
		\[
		D_\Sigma\mathcal{J}(\Sigma;\lambda) = - \frac{1}{2}\E[ \bar{Y} D_\Sigma F_\Sigma(X) ].
		\]

		On the other hand, we can start from (\ref{eqn:minimizerdifferentiability2}), set $x=X'$, multiply by $\overline{r_\Sigma(X',Y')}$, and take expectation, to deduce
		\[
		\begin{split}
			& \frac{1}{\lambda}  \E[ r_\Sigma (X,Y) \overline{r_\Sigma(X',Y')}D_\Sigma \mathcal{K}(|X'-X|_{\Sigma}^2)]
			\\
			= & \E[ D_\Sigma F_\Sigma (X',Y')\overline{r_\Sigma(X',Y')}] +  \frac{1}{\lambda}  \E[ D_\Sigma F_\Sigma(X,Y) \mathcal{K}(|X-X'|_\Sigma^2)  \overline{r_\Sigma(X',Y')} ] 
			\\
			= & \E[ D_\Sigma F_\Sigma(X,Y) \overline{r_\Sigma(X,Y)}] +  \frac{1}{\lambda}  \E[ D_\Sigma F_\Sigma   \overline{F_\Sigma(X)} ] 
			\\
			=&  \E[ D_\Sigma F_\Sigma(X,Y) \overline{Y}] .
		\end{split}
		\]
		Here the third line uses the Euler-Lagrange equation (\ref{eqn:EulerLagrange1}), and the fact that $(X',Y')$ is an independent copy of $(X,Y)$. Comparing the above, we verify the first variation formula.

		\subsubsection{Second proof: law of large numbers}

		\begin{itemize}
			\item Step 1: Finite atomic measure case.
		\end{itemize}
		
		We work first in the setting of Section \ref{sect:representer}. The representer theorem (Proposition  \ref{prop:representer}) gives an explicit formula
		\[
		\mathcal{J}(\Sigma;\lambda)= J(U,\lambda)= \frac{1}{2}\sum_1^m p_i |y_i|^2-\frac{1}{2} \sum_{i,j=1}^m(\lambda M^{-1}+P)^{-1}_{ij} p_i y_i p_j\bar{y}_j,
		\]
		where $M_{ij}= K(U(a_i-a_j))= \mathcal{K} (|U(a_i-a_j)|_V^2)= \mathcal{K}(|a_i-a_j|^2_\Sigma)$ is a positive definite real symmetric matrix. By (\ref{eqn:representerminimizer}), the minimizer satisfies
		\[
		f_U(Ua_i)=\sum_{j=1}^m (\lambda M^{-1}+P)^{-1}_{ji} p_j y_j.
		\]
		However, the Euler-Lagrange equation implies (\cf (\ref{eqn:EulerLagrange1}))
		\[
		f_U(Ua_i)= \frac{1}{\lambda} \E[ r_\Sigma(X,Y) \overline{K(Ua_i- UX) }  ]= \frac{1}{\lambda} \E[ r_\Sigma(X,Y){\mathcal{K}(|X-a_i|^2_\Sigma) }  ].
		\]
		Upon differentiating the representer theorem formula,
		\[
		\begin{split}
			D_\Sigma\mathcal{J}(\Sigma;\lambda) = &\frac{1}{2} \sum  p_i y_i(\lambda M^{-1}+P)^{-1}_{il} \lambda D_\Sigma (M^{-1})_{ln} (\lambda M^{-1}+P)^{-1}_{nj}p_j\bar{y}_j
			\\
			= &   \frac{\lambda }{2} \sum f(Ua_l) D_\Sigma (M^{-1})_{ln} \overline{f(U a_n)}
			\\
			= &  \frac{1 }{2\lambda} \E[ r_\Sigma(X,Y) \mathcal{K}(|X-a_l|^2_\Sigma)  D_\Sigma (M^{-1})_{ln} \overline{r_\Sigma(X',Y')} \mathcal{K}(|X'-a_n|^2_\Sigma)    ]
			\\
			= &  -\frac{1 }{2\lambda} \E[ r_\Sigma(X,Y)\overline{r_\Sigma(X',Y')}D_\Sigma \mathcal{K}(|X-X'|_\Sigma^2],
		\end{split}
		\]
		which verifies the first variation formula. The last line uses the derivative formula for inverse matrix $D_\Sigma (M^{-1})= -M^{-1} (D_\Sigma M) M^{-1}$.

		\begin{itemize}
			\item Step 2: Law of large numbers.
		\end{itemize}

		The idea is to approximate the general distribution $(X,Y)$ by the finite atomic measures using the law of large numbers. We use the setup of Section \ref{sect:lawoflargenumbers}. The first variation formula holds for the empirical measures, and to deduce the result for $(X,Y)$, it suffices to prove

		\begin{prop}\label{prop:lawoflargenumbergradient}
			In the setting of Theorem \ref{thm:firstvariation}, the following holds with probability one. Given any compact subset of $Sym^2_+$, as $m\to +\infty$ we have the uniform convergence in $\Sigma$:
			$
			\mathcal{J}_m(\Sigma, \lambda)\to \mathcal{J}(\Sigma;\lambda)$
			and 
			\[
			D_\Sigma \mathcal{J}_m(\Sigma, \lambda)\to  -\frac{1}{2\lambda}   \E[  r_\Sigma(X,Y) \overline{ r_\Sigma(X',Y')}  \mathcal{K}'(|X-X'|_{\Sigma}^2) (X-X')\otimes (X-X') ] .
			\]
		\end{prop}

		\begin{proof}
			We have the almost sure uniform convergence $
			\mathcal{J}_m(\Sigma, \lambda)\to \mathcal{J}(\Sigma;\lambda)$ by Proposition  \ref{prop:lawoflargenumbers}. We now claim that for $\Sigma$ in a fixed compact subset in $Sym^2_+$ (so there is the uniform equivalence $\Lambda^{-1}|\cdot|_V^2\leq \Sigma\leq \Lambda|\cdot|_V^2$), we almost surely have
			\[
			\sup_m \sup_\Sigma \E_m[ |  \mathcal{K}'(|X-X'|_{\Sigma}^2) |^2 |X-X'|^4  ] <+\infty.
			\]
			In the first alternative in Condition 2 of Theorem \ref{thm:firstvariation}, $|\mathcal{K}'(r)|r$ is bounded, so the above is clearly bounded. In the second alternative, $|\mathcal{K}'|$ is monotone decreasing for $r\geq r_0$. Using the condition that $|\mathcal{K}'(r)|r$ remains bounded for bounded $r$, we have
			\[
			\begin{split}
				& \sup_m \sup_\Sigma\E_m[ |  \mathcal{K}'(|X-X'|_{\Sigma}^2) |^2 |X-X'|^4  ]  
				\\
				\leq & C+\sup_m\sup_\Sigma\E_m [ |\mathcal{K}'(|X-X'|_{\Sigma}^2) |^2 |X-X'|^4 \1_{\Lambda^{-1}|X-X'|^2\geq r_0}]
				\\
				\leq  & C+\sup_m \E_m [ |\mathcal{K}'(\Lambda^{-1}|X-X'|_V^2) |^2 |X-X'|^4]
			\end{split},
			\]
			which is almost surely finite since $\E[|\mathcal{K}'(\Lambda^{-1}|X-X'|_V^2) |^2 |X-X'|^4]|<+\infty$.

			



			We need to show that with probability one, the first variation
			\[
			-\frac{1}{2\lambda}   \E_m[  r_{\Sigma,m}(X,Y)  \overline{r_{\Sigma,m}(X',Y')}  \mathcal{K}'(|X-X'|_{\Sigma}^2) (X-X')\otimes (X-X') ] .
			\] 
			converges to the limiting version. This is based on two main ingredients. The first is that by Proposition  \ref{prop:lawoflargenumbers}, 
			with probability one, we know the uniform in $U$ convergence  $f_{U,m}\to f_U$ in $\mathcal{H}$, and therefore in $C^0$. Using Cauchy-Schwarz and the above claim,  up to an error term which uniformly converges to zero, the above expression can be replaced by the empirical expectation of a deterministic function:
			\[
			-\frac{1}{2\lambda}   \E_m[  r_{\Sigma}(X,Y) \overline{  r_{\Sigma}(X',Y')}  \mathcal{K}'(|X-X'|_{\Sigma}^2) (X-X')\otimes (X-X')].  
			\]

			\red{
				Secondly, we need a uniform law of large numbers to show that with  probability one, the above expression converges to the limiting version uniformly in $\Sigma$. For a start, the integrability hypothesis on $\mathcal{K}$ and Cauchy-Schwarz guarantees the $L^1$-integrability condition 
				\[
				\E[ | r_{\Sigma}(X,Y)\overline{  r_{\Sigma}(X',Y')}  \mathcal{K}'(|X-X'|_{\Sigma}^2) (X-X')\otimes (X-X')|]<+\infty, 
				\]
				so the law of large numbers applies pointwise in $\Sigma$. To improve this to uniform convergence in $\Sigma$, for any given $\epsilon$,  we take a $\delta$-dense subset of $\Sigma_i$ with $\delta$ to be determined, so that for $m\geq m_0(\epsilon)$, almost surely 
				\[
				-\frac{1}{2\lambda}   \E_m[  r_{\Sigma_i}(X,Y) \overline{  r_{\Sigma_i}(X',Y')}  \mathcal{K}'(|X-X'|_{\Sigma_i}^2) (X-X')\otimes (X-X')] 
				\]
				deviates from the expectation version by less than $\epsilon$. We then need an equicontinuity estimate for the difference between 
				\[
				-\frac{1}{2\lambda}   \E_m[  r_{\Sigma_i}(X,Y) \overline{  r_{\Sigma_i}(X',Y')}  \mathcal{K}'(|X-X'|_{\Sigma_i}^2) (X-X')\otimes (X-X')]
				\]
				and
				\[
				-\frac{1}{2\lambda}   \E_m[  r_{\Sigma}(X,Y) \overline{  r_{\Sigma}(X',Y')}  \mathcal{K}'(|X-X'|_{\Sigma}^2) (X-X')\otimes (X-X')] 
				\]
				for $|\Sigma-\Sigma_i|\leq \delta$. 
				To replace $r_{\Sigma_i}(X,Y)$ by $r_\Sigma(X,Y)$, we apply the quantitative continuity of $F_\Sigma$ on $\Sigma$ (\cf Proposition  \ref{lem:continuityJ}), to see that, almost surely
				\[
				\sup_m \sup_{|\Sigma-\Sigma_i|\leq \delta} \E_m[ |r_\Sigma(X,Y)- r_{\Sigma_i}(X,Y)|^2   ] \leq C(\delta), \quad \lim_{\delta\to 0} C(\delta)=0.
				\]
				Combined with the fact that almost surely
				\[
				\limsup_{m\to +\infty} \sup_\Sigma \E_m [ |\mathcal{K}'(|X-X'|_\Sigma ^2) |^2 |X-X'|^4] \leq C,
				\]
				and applying Cauchy-Schwarz, we see that if $\delta$ is chosen small enough, then in the $m\to +\infty$ limit,
				\[
				\begin{split}
					& -\frac{1}{2\lambda}   \E_m[  r_{\Sigma_i}(X,Y) \overline{  r_{\Sigma_i}(X',Y')}  \mathcal{K}'(|X-X'|_{\Sigma_i}^2) (X-X')\otimes (X-X')]
					\\
					= & -\frac{1}{2\lambda}   \E_m[  r_{\Sigma}(X,Y) \overline{  r_{\Sigma}(X',Y')}  \mathcal{K}'(|X-X'|_{\Sigma_i}^2) (X-X')\otimes (X-X')]
					+ O(C(\delta)) 
				\end{split}
				\]
				for some uniform modulus of continuity $C(\delta)$. 
			}
			
			\red{
				Using that 
				$
				\E_m[  |r_{\Sigma}(X,Y) |^2] \leq \E_m[|Y|^2]$ is almost surely bounded in the $m\to +\infty$ limit, and applying Cauchy-Schwarz, we can replace $\mathcal{K}'(|X-X'|_{\Sigma_i}^2) $ by $\mathcal{K}'(|X-X'|_{\Sigma}^2)$ up to $C(\delta)$ error. By taking $\delta$ small enough, we can make $C(\delta)<\epsilon$, which shows the uniform in $\Sigma$ convergence.
			}		
		\end{proof}

		\subsection{Continuous extension of the first variation}

		\begin{Notation}
			We denote $Sym^2_{\geq 0}\subset Sym^2 V^*$ as the closed subset of positive semi-definite forms on $V$. This admits a stratification, with open stratum
			\[
			\mathcal{S}_l= \{  \text{positive semi-definite forms on $V$ with rank $l$} \}, 
			\]
			for $l=0,1,\ldots d-1$. Each $\mathcal{S}_l$ is a smooth manifold, and 
			\[
			Sym^2_{\geq 0}=Sym^2_+\cup \bigcup_{l\leq d-1} \mathcal{S}_l.
			\]
		\end{Notation}

		\red{In this section we study the extension of $\mathcal{J}$ and its first variation to the boundary of $Sym^2_{\geq 0}$. These points represent the degenerate matrices, corresponding to dimensionally reduced configurations, and to test the local minimizing property of $\mathcal{J}$, one needs a well defined first variation as one moves into the interior.}	
		\begin{lem}\label{lem:Jcontinuity2}
			The functional $\mathcal{J}(\Sigma;\lambda)$ extends continuously to the boundary of $Sym^2_{\geq 0}$. Moreover, the minimizer $F_\Sigma$ can be defined for $\Sigma$ on the boundary of $Sym^2_{\geq 0}$, such that $F_\Sigma$ depends continuously on $\Sigma$ in the $\norm{\cdot}_\E$-topology.

		\end{lem}
		
		\begin{proof}
			Motivated by Lemma  \ref{lem:Jrotationalinv},
			for $\Sigma$ on the boundary of $Sym^2_{\geq 0}$, we can define $\mathcal{J}(\Sigma;\lambda):= J(U,\lambda)$ by picking $U$ such that $\Sigma=U^TU$. By Proposition  \ref{lem:continuityJ}, this does not depend on the choice of $U$, and 
			the functional  $\mathcal{J}(\Sigma;\lambda)=J(U,\lambda)$ depends continuously on $\Sigma$, with the uniform modulus of continuity estimate
			\[
			|\mathcal{J}(\Sigma;\lambda)-\mathcal{J}(\Sigma';\lambda)|\leq \frac{1}{2\lambda}\E[|Y|^2 ]  \E[ |\mathcal{K}(|X'-X|_\Sigma^2)- \mathcal{K}(|X'-X|_{\Sigma'}^2) |^2 ]^{1/2}.
			\]

			For positive semi-definite $\Sigma$, we can still define $F_\Sigma:= f_U\circ U$. This depends only on $U$ through $\Sigma=U^T U$, by the integral equation characterization in Proposition  \ref{prop:integraleqn1},
			\red{\[
				F_\Sigma(x)= \frac{1}{\lambda} \E[ (Y-F_\Sigma(X))\mathcal{K}(|x-X|_\Sigma^2)  ].
				\]}
			Lemma \ref{lem:Jminimizercontinuity1}  implies that $F_\Sigma$ depends continuously on $\Sigma$ in the $\norm{\cdot}_\E$-topology. 
		\end{proof}

		\begin{Def}
			We say $\mathcal{J}(\Sigma;\lambda)$ extends to a $C^1$-function on a convex subset of $Sym^2_{\geq 0}$, if there is a differential $D\mathcal{J}$ on $Sym^2_{\geq 0}$ which is a $Sym^2 V$-valued  continuous function on this subset, such that at any $\Sigma,\Sigma'$ in the subset, we have
			\[
			\mathcal{J}(\Sigma';\lambda)= \mathcal{J}(\Sigma; \lambda)+ \langle D\mathcal{J}(\Sigma;\lambda),\Sigma'-\Sigma\rangle +o(|\Sigma'-\Sigma|).
			\]
		\end{Def}

		\begin{prop}\label{prop:continuousextensionJ}
			In the setup of Theorem \ref{thm:firstvariation}, suppose in addition that on a given compact convex subset of $\Sigma,\Sigma'\in Sym^2_{\geq 0}$, 
			\begin{itemize}
				\item (Uniform integrability) 
				\[
				\sup_\Sigma \E[ |\mathcal{K}'(|X-X'|^2_\Sigma)|^2 |X-X'|^4   ] <+\infty.
				\]
				\item (Uniform modulus of continuity) As $\delta\to 0,$
				\[
				\sup_{|\Sigma-\Sigma'|\leq \delta  } \E[  |\mathcal{K}'(|X-X'|^2_\Sigma) -\mathcal{K}'(|X-X'|^2_{\Sigma'}) |^2 |X-X'|^4    ] \to 0.
				\]
			\end{itemize}
			Then
			$\mathcal{J}(\Sigma;\lambda)$ extends to a $C^1$-function on this compact subset of $Sym^2_{\geq 0}$.
		\end{prop}

		\begin{proof}
			First, we claim the formula 
			\[
			D_\Sigma \mathcal{J}(\Sigma; \lambda)
			:=  -\frac{1}{2\lambda}   \E[  r_\Sigma(X,Y) \overline{ r_\Sigma(X',Y') } \mathcal{K}'(|X-X'|_{\Sigma}^2) (X-X')\otimes (X-X') ] 
			\]
			extends continuously to the subset of the boundary of $Sym^2_{\geq 0}$. 
			By Cauchy-Schwarz, the trivial upper bound Lemma \ref{lem:trivialupperbound},  and the uniform integrability hypothesis, we can bound the difference between $D_\Sigma\mathcal{J}(\Sigma'; \lambda)$ and 
			\begin{equation}\label{eqn:gradientcontinuousextension1}
				-\frac{1}{2\lambda}   \E[  r_\Sigma(X,Y)  \overline{r_\Sigma(X',Y')}  \mathcal{K}'(|X-X'|_{\Sigma'}^2) (X-X')\otimes (X-X') ] 
			\end{equation}
			by the estimable quantity
			\[
			\begin{split}
				& \lambda^{-1} \E[|F_{\Sigma'}(X)- F_\Sigma(X)|^2  ]^{1/2} \E[|r_\Sigma(X,Y)|^2 ] \E[ |\mathcal{K}'(|X-X'|^2_\Sigma)|^2|X-X'|^4  ]^{1/2} 
				\\
				\leq & C\lambda^{-1} \E[|F_{\Sigma'}(X)- F_\Sigma(X)|^2  ]^{1/2} \E[|Y|^2] .
			\end{split}
			\]
			By Lemma  \ref{lem:Jminimizercontinuity1},
			\[
			\E[ |F_{\Sigma'}(X)- F_{\Sigma}(X)|^2 ] \leq  \frac{1}{\lambda^2} \E[|Y|^2 ]  \E[ |\mathcal{K}(|X'-X|_{\Sigma'}^2)-\mathcal{K}(|X'-X|_{\Sigma}^2)  |^2 ] ,  
			\]
			so this quantity is $o(|\Sigma'-\Sigma|)$. On the other hand, the difference between the expression (\ref{eqn:gradientcontinuousextension1}) and $D_\Sigma\mathcal{J}(\Sigma, \lambda)$ can be bounded using Cauchy-Schwarz and the trivial upper bound Lemma \ref{lem:trivialupperbound}, by the quantity
			\[
			\frac{1}{2\lambda} \E[|Y|^2] \E[  |\mathcal{K}'(|X-X'|^2_\Sigma) -\mathcal{K}'(|X-X'|^2_{\Sigma'}) |^2 |X-X'|^4   ]^{1/2},
			\]
			which is $o(|\Sigma'-\Sigma|)$ by the uniform modulus of continuity hypothesis. 
			This proves that $\lim_{\Sigma'\to \Sigma}D_\Sigma\mathcal{J}(\Sigma', \lambda) = D_\Sigma\mathcal{J}(\Sigma, \lambda) $.

			For $\Sigma, \Sigma'\in Sym^2_+$, the fundamental theorem of calculus implies
			\[
			\mathcal{J}(\Sigma')-\mathcal{J}(\Sigma)=\int_0^1 \langle D\mathcal{J}( s\Sigma'+ (1-s) \Sigma), \Sigma'-\Sigma\rangle  ds.
			\]
			By a limiting argument, using the convexity of the subset of $Sym^2_{\geq 0}$, the formula extends to $\Sigma,\Sigma'$ on the boundary, so 
			\[
			\mathcal{J}(\Sigma';\lambda)= \mathcal{J}(\Sigma; \lambda)+ \langle D\mathcal{J}(\Sigma;\lambda),\Sigma'-\Sigma\rangle +o(|\Sigma'-\Sigma|),\quad \forall \Sigma'\in Sym^2_{\geq 0}
			\]
			holds for boundary points as well. 
		\end{proof}

		\red{We now discuss some settings where the extension criterion applies.}

		\begin{eg}
			The hypothesis of Proposition  \ref{prop:continuousextensionJ} is satisfied, if the function $\mathcal{K}':\R_{\geq 0}\to \R$ is a bounded continuous function, and $\E[|X|^4]<+\infty$. In particular, the Sobolev kernel with $\gamma>1$ (Example \ref{eg:Sobolevkernel}), and the Gaussian kernel (Example \ref{eg:Gaussiankernel}) satisfy this condition.

		\end{eg}

		Moreover, Proposition  \ref{prop:continuousextensionJ} applies also in some cases where $\mathcal{K}'$ is unbounded near the origin, such as Sobolev kernels with $0<\gamma<1$ in Example \ref{eg:Sobolevkernel}, where $|\mathcal{K}'(r)|= O(r^{\gamma-1})$ as $r\to 0$.

		\begin{lem}\label{lem:boundeddensitygradientbound}
			Given a point in the boundary stratum 
			$\mathcal{S}_l\subset Sym^2_{\geq 0}$, there is a small compact convex neighbourhood in $Sym^2_{\geq 0}$ such that the following holds.

			Suppose $X$ is a continuous random variable $X$ with bounded probability density in $V$, and compact support in $V$. Suppose the radial function $\mathcal{K}(r)$ is $C^1$ on $\R_+$, and $\mathcal{K}'$ remains bounded near infinity, and has a power law bound $|\mathcal{K}'(r)|= O(r^{-\alpha})$ near zero, where $0<\alpha< l/4$ .  Then the hypothesis of Proposition  \ref{prop:continuousextensionJ} is satisfied on the convex neighbourhood.  
		\end{lem}

		\begin{proof}
			By the assumptions $\mathcal{K}'(|X-X'|_\Sigma^2)$ is only large if $|X-X'|_\Sigma^2$ is small, and $|\mathcal{K}'(r)|=O(r^{-\alpha})$, 
			so for $t\gg 1$,
			\[
			\mathbb{P}( |\mathcal{K}'(|X-X'|_\Sigma^2)|^2 \geq t) \leq \mathbb{P}(|X-X'|_\Sigma^2 \leq Ct^{-1/2\alpha}).
			\]
			For any $\Sigma$ in the small neighbourhood of the boundary point in $\mathcal{S}_l$, there are $l$ coordinate components of $X-X'$, such that $\sum_1^l |(X-X')_i|^2\leq C|X-X'|_\Sigma^2$. 
			Thus
			\[
			\mathbb{P}(|X-X'|_\Sigma^2 \leq Ct^{-1/2\alpha}) \leq \mathbb{P}( \sum_1^l |(X-X')_i|^2 \leq C't^{-1/2\alpha})
			\]
			The assumptions on $X$ implies that $X-X'$ has bounded density and compact support, so
			\[
			\mathbb{P}( \sum_1^l|(X-X')_i|^2 \leq Ct^{-1/2\alpha}) \leq C't^{-l/4\alpha},
			\]
			whence
			\[
			\begin{split}
				& \E[  |\mathcal{K}'(|X-X'|^2_\Sigma)|^2 |X-X'|^4   ]\leq C\E[  |\mathcal{K}'(|X-X'|^2_\Sigma)|^2   ] 
				\\
				& \leq C(1+ \int_1^\infty \mathbb{P}( |\mathcal{K}'(|X-X'|_\Sigma^2)|^2 \geq t)  dt)
				\\
				&  \leq C(1+ \int_1^\infty C' t^{-l/4\alpha}  dt).
			\end{split}
			\]
			Since $l/4\alpha>1$ by assumption, this integral converges. The estimate is uniform for all $\Sigma$ in the small neighbourhood, so the uniform integrability hypothesis is verified.

			The above discussion also shows that the contribution to the integral from 
			the region with $ |X-X'|_\Sigma^2\ll 1$ is small. Since $\mathcal{K}'$ is continuous away from the origin, and $X-X'$ is bounded, we see the uniform modulus of continuity in $\Sigma$. 
		\end{proof}

		\begin{rmk}
			The lesson from the proof of Lemma \ref{lem:boundeddensitygradientbound} is that in order for the first variation to be very big near some boundary stratum $\mathcal{S}_l$ with $l>4\alpha+1$, then $X-X'$ should have local concentration along some subspace $W\subset V$, where $W$ is the null subspace of $\Sigma$. 
		\end{rmk}


		For applications to our companion paper, we also need the following criterion for the Lipschitz continuity of $D\mathcal{J}$.

		\begin{cor}\label{cor:Lipextension}
			Suppose that $\mathcal{K}':\R_{\geq 0}\to \R$ is a bounded continuous function, $r\mathcal{K}'(r)$ is bounded for large $r$, and $\mathcal{K}''$ is bounded. Suppose $\E[|Y|^2]+ \E[|X|^8]<+\infty$. Then  $D\mathcal{J}$ extends to a Lipschitz continuous tensor valued function on $Sym^2 V^*$.
		\end{cor}

		\begin{proof}
			The conditions of the first variation formula Thm. \ref{thm:firstvariation} and the $C^1$-extension Proposition \ref{prop:continuousextensionJ} hold in our setting. It suffices to prove $D\mathcal{J}$ is Lipschitz continuous as a tensor field on $Sym^2_{\geq 0}$.

			We apply the first variation formula to estimate $|D \mathcal{J}(\Sigma;\lambda)- D\mathcal{J}(\Sigma';\lambda)|$. There are two sources of errors. The first error is
			\[
			\begin{split}
				& \E[  |r_\Sigma(X,Y)- r_{\Sigma'}(X,Y)| | \overline{r_\Sigma(X',Y')} | |\mathcal{K}'(|X-X'|_{\Sigma}^2) (X-X')\otimes (X-X') |]
				\\
				\leq &  \E[  |r_\Sigma(X,Y)- r_{\Sigma'}(X,Y)|^2   ]^{1/2}  \E[ |r_\Sigma(X',Y')|^2  ]^{1/2} \E[ |X-X'|^4   ]^{1/2} \norm{\mathcal{K}'}_{C^0} 
				\\
				\leq  &  \E[ |F_\Sigma(X)- F_{\Sigma'}(X)|^2 ]^{1/2}  \E[|Y|^2]^{1/2} \E[|X-X'|^4]^{1/2} \norm{\mathcal{K}'}_{C^0}
				\\
				\leq & C \lambda^{-1}\E[   |\mathcal{K}(|X-X'|^2_\Sigma)- \mathcal{K}(|X-X'|^2_{\Sigma'})|^2 ]^{1/2}  \E[|Y|^2] \E[|X-X'|^4]^{1/2} \norm{\mathcal{K}'}_{C^0}
				\\
				\leq & C \lambda^{-1}|\Sigma-\Sigma'|  \E[|Y|^2] \E[|X|^4] \norm{\mathcal{K}'}_{C^0}^2
			\end{split}
			\]
			Here the second line uses Cauchy-Schwarz and the independence of $(X,Y)$ with $(X',Y')$, the third line uses the trivial upper bound Lem. \ref{lem:trivialupperbound}, while the fourth line uses the quantitative continuity Lemma \ref{lem:Jminimizercontinuity1}, and the last line uses mean value inequality. This gives a Lipschitz bound on the first error.

			The second error is
			\[
			\begin{split}
				& \E[  |r_\Sigma(X,Y)| | \overline{r_\Sigma(X',Y')} | |(\mathcal{K}'(|X-X'|_{\Sigma}^2)- \mathcal{K}'(|X-X'|_{\Sigma}^2)) (X-X')\otimes (X-X') |]
				\\
				\leq &  C\norm{\mathcal{K}''}_{C^0} |\Sigma-\Sigma'| \E[      |r_\Sigma(X,Y)| | \overline{r_\Sigma(X',Y')} | |X-X' |^4     ]
				\\
				\leq & C\norm{\mathcal{K}''}_{C^0} |\Sigma-\Sigma'| \E[      |r_\Sigma(X,Y) |^2]^{1/2}   \E[|r_\Sigma(X',Y')|^2]^{1/2}  \E[ | X-X' |^8     ]^{1/2}
				\\
				\leq & C\norm{\mathcal{K}''}_{C^0} |\Sigma-\Sigma'| \E[   |Y|^2]  \E[ | X|^8     ]^{1/2}
			\end{split}
			\]
			Here the second line uses the mean value inequality, the third line uses Cauchy-Schwarz and the independence of $(X,Y)$ with $(X',Y')$, and the fourth line uses the trivial upper bound Lem. \ref{lem:trivialupperbound}. We conclude that this error is also Lipschitz bounded, hence $|D \mathcal{J}(\Sigma;\lambda)- D\mathcal{J}(\Sigma';\lambda)| \leq C|\Sigma-\Sigma'|$ with constant depending on the various moment bounds on $X,Y$.
		\end{proof}

		\subsection{Partial compactification}\label{sect:partialcompactification}

		\begin{Def}
			The partial compactification $\overline{Sym ^2_{\geq 0} } $ consists of $Sym^2_{\geq 0}$ and the points at infinity. Here the points at infinity correspond to a pair $(W,\Sigma_W)$, where $W\subset V$ is a proper subspace of $V$ which is allowed to be zero, and $\Sigma_W$ is a positive semidefinite form on $W$.

		\end{Def}

		\begin{Def}
			(Topology at infinity) A sequence of $\Sigma_i\in Sym^2_+$ is said to converge to a point at infinity $(W,\Sigma_W)$, if the inner products on $W$ induced by the restriction of $\Sigma_i$ converge to $\Sigma_W$, and the quotient inner products on $V/W$ induced by $\Sigma_i$ are bounded below by $\Lambda_i^2 I$, where $\Lambda_i\to +\infty$. 
		\end{Def}

		\begin{prop}
			The function $\mathcal{J}$ extends to a continuous function on the partial compactification  $\overline{Sym ^2_{\geq 0} } $.
		\end{prop}

		\begin{proof}
			This is a reinterpretation of Theorem \ref{thm:asymptoticvalue}. To match the setting, given a sequence $\Sigma$ tending to $(W,\Sigma_W)$ (we suppress the subscripts), by the Gram-Schmidt process, we can find $U\in \End{V}$ such that $U$ preserves the subspace $W$, and $\Sigma=U^TU$. Writing out the matrix form,
			\[
			U= \begin{bmatrix}
				U_{11}, & U_{12}\\
				0  , &U_{22}
			\end{bmatrix} \in 
			\begin{bmatrix}
				\End(W), & \Hom(W', W)\\
				\Hom(W, W'), & \End(W')
			\end{bmatrix}. 
			\]
			By assumption, $\Sigma|_W= U_{11}^T U_{11}$ tends to $\Sigma_W$, so we can arrange that $U_{11}$ also converges to some $U_{11}^\infty$. The quotient inner product  induced by  $\Sigma$ on $W'\simeq V/W$ is given by $U_{22}^T U_{22}$, so $|U_{22}x_{W'}|\geq \Lambda_i |x_{W'}|$ for any $x_{W'}\in W'$, with $\Lambda_i\to +\infty. $ This is precisely the setting of Section \ref{Translationinvkernel}. Theorem \ref{thm:asymptoticvalue} then implies that $\mathcal{J}(\Sigma;\lambda)$ has a unique limit.
		\end{proof}

		\red{Intuitively, the points at infinity in the partial compactification models scale separation and cluster decoupling, while the lower rank boundary strata of $Sym^2_{\geq 0}$ model variable selection.}

		\subsection{Completely monotone kernels}\label{sect:completelymontonekernel}

		A distinguished class of rotationally symmetric kernels consists of kernels representable by a weighted sum of Gaussians,
		\begin{equation}\label{eqn:completelymonotonekernel}
			K(x)= \mathcal{K}(|x|^2_V)= \int_0^{\infty} e^{-t |x|_V^2} d
			\nu(t),
		\end{equation}
		for some finite measure $\mu$.  The celebrated theorem of Schoenberg \cite{Schoenberg} says that given a radial profile function $\mathcal{K}$, then $K(x)=\mathcal{K}(|x|^2)$ defines a RKHS for all dimensions of $V$, if and only if this Gaussian sum representation formula holds, if and only if the profile function $\mathcal{K}$ is  a \emph{completely monotone function}, meaning that $(-1)^m \frac{d^m \mathcal{K}}{dr^m} $ has the sign of $(-1)^m$.

		In our normalization $
		K(0)=\int_{V^*} k_Vd\omega=1$. Upon differentiation,
		\begin{equation}\label{eqn:GaussiansumK'}
			\mathcal{K}'(r)= -\int_0^\infty e^{-tr} td\nu(t).
		\end{equation}
		In particular, the dimensional constants do not explicitly arise in the first variation formula \textit{}(\ref{eqn:Jfirstvariation3}). The function $k_V$ can be obtained by Fourier transform,
		\[
		k_V(\omega)= \int_0^\infty (\pi t^{-1})^{d/2} e^{-\pi^2 t^{-1} |\omega|^2} d\nu(t).
		\]
		Observe also that if $\mathcal{K}$ is completely monotone, then so is $\mathcal{K}^2$. In fact, 
		\[
		\mathcal{K}^2(r)= \int_0^\infty e^{-tr} d\tilde{\nu}(t),
		\]
		where the measure $\tilde{\nu}$ satisfies
		\[
		\int_0^\infty f(t) d\tilde{\nu}(t)= \int_0^\infty\int_0^\infty f(s+t) d\nu(t)d\nu(s).
		\]

		\begin{eg}\label{eg:Gaussiankernel}
			The Gaussian kernel corresponds to the case that $K(x)=e^{-\beta |x|_V^2}$ for some $\beta>0$, so the measure $\mu$ is a Dirac measure. By taking the Fourier transform,
			\[
			k_V(\omega)= (\pi \beta^{-1})^{d/2} e^{- \pi^2 
				\beta^{-1} |
				\omega|^2}.
			\]
		\end{eg}

		\begin{eg}
			The kernel
			$K(x)= \frac{1}{(1+ |x|^2_V)^\alpha}$ for $\alpha>0$ corresponds to the radial profile function $\mathcal{K}(r)= \frac{1}{(1+r)^\alpha}$, which is completely monotone.
		\end{eg}

		\begin{eg}\label{eg:Sobolevkernel}
			The Sobolev kernels can be defined by
			\begin{equation}\label{eqn:Sobolevkernel1}
				k_V(\omega) =k_{d,\gamma}(|\omega|^2)=\frac{(4\pi)^{d/2}\Gamma(\gamma+d/2 )}{  \Gamma(\gamma) } (1+|2\pi \omega|^2)^{-\gamma-d/2}.
			\end{equation}
			Here $\gamma> 0$ to ensure the $L^1$ finiteness. The resulting RKHS are the standard Sobolev spaces on the $d$-dimensional Euclidean space $V$, and the extra parameter $\gamma$ controls the smoothness of functions in $\mathcal{H}$.  The choice of $2\pi$ is aimed at reducing the appearance of dimensional constants. The coefficient is designed to normalize the $L^1$ integral to one.

			\begin{lem}\label{lem:Gaussianrep}
				(Gaussian sum representation) Suppose $\gamma>0$, then
				\[
				k_V(\omega)= \frac{ (4\pi)^{d/2} }{\Gamma(\gamma)} \int_0^{+\infty} y^{\gamma+d/2-1} e^{- (1+|2\pi \omega|^2) y} dy.
				\]
			\end{lem}

			\begin{proof}
				We start from 
				\[
				\int_0^{+\infty} y^{\gamma+d/2-1} e^{- (s+1) y} dy= (s+1)^{-\gamma-d/2}\int_0^{+\infty} y^{\gamma+d/2-1} e^{-  y} dy= (1+s)^{-\gamma-d/2 } \Gamma(\gamma+d/2).
				\]
				Now we plug in $s= |2\pi \omega|^2$, to see
				\[
				\Gamma(\gamma+d/2)(1+|2\pi \omega|^2)^{-\gamma-d/2}=  \int_0^{+\infty} y^{\gamma+d/2-1} e^{- (1+|2\pi \omega|^2) y} dy,
				\]
				as required.
			\end{proof}

			By taking the Fourier transform, one deduces that $K(x)$ also admits a Gaussian sum representation:
			\begin{equation}\label{eqn:GaussianrepK}
				K(x)= K_\gamma(|x|^2):= \frac{1}{\Gamma(\gamma)} \int_0^{+\infty} y^{\gamma-1} e^{-y} e^{- \frac{ |x|^2}{4y} }dy .
			\end{equation}
			Notice that the dimensional dependence $d$ has disappeared in the formula. Suppose $\gamma>0$, then the derivative of $K_\gamma(r)$ for $r>0$ is
			\begin{equation}\label{eqn:Kgammaderivativeformula}
				K_\gamma'(r)= -\frac{1}{4} \frac{1}{\Gamma(\gamma)} \int_0^{+\infty} y^{\gamma-2} e^{-y} e^{- \frac{ r}{4y} }dy = -\frac{1}{4(\gamma-1) } K_{\gamma-1}(r).
			\end{equation}
			In particular, the one-variable function $K_\gamma(r)$ is differentiable at the origin to any order $0\leq m<\gamma$, and the $m$-th derivative is $ \frac{(-1)^m}{4^m (\gamma-1)\ldots (\gamma-m)  } $.

			The Sobolev kernel function is intensely studied classically, and admit formulae in terms of modified Bessel functions \cite{Bessel}.
			Using the method of steepest descent, one can deduce from the Gaussian sum representation formula that $K_\gamma(|x|^2)= O( |x|^\gamma e^{-|x|}) $ for large $x$. For $\gamma>1$, then $\mathcal{K}'(r)$ is bounded near zero, while for $0<\gamma<1$, then   $\mathcal{K}'(r)=O(r^{\gamma-1})$ blows up near the origin.

		\end{eg}

		\section{Landscape of vacua: limiting configurations}\label{sect:landscapeofvacualimitingconfigurations}

		While the minimization of $I(f,U,\lambda)$ for fixed $U,\lambda$ reduces to solving a \emph{linear equation} in an infinite dimensional function space, the functional $\mathcal{J}(\Sigma, \lambda)$ depends on $\Sigma$ in a highly \emph{nonlinear} fashion. 
		We now consider the problem of locally minimizing the functional $\mathcal{J}(\Sigma;\lambda)$ with respect to $\Sigma$ for fixed moderately small $\lambda$, and to avoid confusion, we shall refer to the local minima of $\mathcal{J}(\Sigma;\lambda)$ as \emph{vacua}. \red{This terminology comes from physics, where it refers to the locally energy minimizing field configurations, and the distinct vacua are separated by some energy barriers, often with distinct characteristic physical behavior.}

		Formally speaking,

		\begin{Def}
			A vacuum in $Sym^2_{\geq 0}$ is a local minimizer of $\mathcal{J}$ in $Sym^2_{\geq 0}$. 
			A vacuum at infinity is a point at infinity $(W,\Sigma_W)\in \overline{Sym ^2_{\geq 0} } $, such that for any sequence $\Sigma_i\in Sym^2_+$ converging to $ (W,\Sigma_W) $, whenever $i$ is large enough, then
			\[
			\mathcal{J}(\Sigma_i;\lambda)\geq  \mathcal{J}( 
			(W,\Sigma_W)  ,\lambda).
			\]
		\end{Def}
		
		In this section we focus on the question
		
		\begin{Question}
			In the partial compactification $\overline{Sym^2_{\geq 0}}$, how can we test that $\Sigma$ is locally minimizing? What is the subleading effect of $\mathcal{J}$ near a given vacuum?

		\end{Question}

		\subsection{Boundary vacua in  $Sym^2_{\geq 0}$}\label{sect:boundaryofSym}

		Recall that $Sym^2_{\geq 0}= \bigcup \mathcal{S}_l\cup Sym^2_+$, where $\mathcal{S}_l$ is the open stratum consisting of all rank $l$ positive semi-definite forms on $V$. The $\mathcal{S}_l$ is a submanifold of dimension $\frac{l(l+1)}{2}+ l(d-l)$ inside $Sym^2 V^*$. Every $\Sigma\in \mathcal{S}_l$ is equivalent to the data of a null subspace $W\subset V$ of dimension $d-l$ (whose choice accounts for $l(d-l)$ parameters), and a non-degenerate inner product on the quotient space $V/W$ (whose choice accounts for $\frac{l(l+1)}{2}$ parameters).

		In this section, we consider a boundary point $\Sigma\in \mathcal{S}_l\subset Sym^2_{\geq 0}$, and assume that the hypothesis of Proposition  \ref{prop:continuousextensionJ} holds in a neighbourhood of $\Sigma$, so $\mathcal{J}$ is locally a $C^1$-function on $Sym^2_{\geq 0}$. We write $W$ as the null subspace of $\Sigma$, and $W'$ as its orthogonal complement with respect to $|\cdot |_V^2$, which induces an orthogonal direct sum decomposition
		$
		V=W\oplus W'$, $V^*= W^*\oplus W'^*, 
		$
		so that
		\[
		\begin{cases}
			Sym^2 V\simeq Sym^2 W\oplus Sym^2 W'\oplus W\otimes W',
			\\
			Sym^2 V^*\simeq Sym^2 W^*\oplus Sym^2 W'^*\oplus W^*\otimes W'^*.
		\end{cases}
		\]
		The tangent space of the submanifold $\mathcal{S}_l\subset Sym^2V^*$ at the point $\Sigma$ is identified with $Sym^2 W'^*\oplus  W^*\otimes W'^*$.

		\begin{lem}
			(Necessary condition for boundary vacuum).
			If $\Sigma$ is a vacuum on $\mathcal{S}_l$, then it is a local stationary point of $\mathcal{J}$ restricted to the boundary stratum $\mathcal{S}_l$, so that
			\[
			\langle D_\Sigma \mathcal{J}(\Sigma), A \rangle= 0,\quad \text{for all $A\in Sym^2 V^*$ tangent to $\mathcal{S}_l$ at $\Sigma$}.
			\] 
			Moreover, 
			\[
			\langle D_\Sigma \mathcal{J}(\Sigma), A \rangle \geq 0, \quad \text{for all positive semi-definite $A\in Sym^2 V^*$},
			\]
			where $\langle, \rangle$ is the natural pairing between $Sym^2 V$ and $Sym^2 V^*$.

		\end{lem}
		
		\begin{proof}
			For any positive semi-definite $A$, 
			the directional derivative
			\[
			\langle D_\Sigma\mathcal{J}(\Sigma), A\rangle= \lim_{t\to 0} t^{-1} (\mathcal{J}(\Sigma+tA)-\mathcal{J}(\Sigma))\geq 0
			\]
			by the local minimizing property of $\Sigma$.
		\end{proof}

		\begin{lem}
			(Sufficient condition for boundary stratum) Suppose $\Sigma$ is a local minimizer of $\mathcal{J}$ restricted to the stratum $\mathcal{S}_l$, and suppose
			\[
			\langle D_\Sigma \mathcal{J}(\Sigma), \omega\otimes \omega \rangle>0 ,\quad \forall \omega\neq 0\in W^*\subset V^*.
			\]
			Then $\Sigma$ is a vacuum.
		\end{lem}

		\begin{proof}
			By the continuity of the first variation, in a neighbourhood of $\Sigma\in Sym^2_{\geq 0}$, we have some constant $c>0$ such that
			\[
			\langle D_\Sigma \mathcal{J}(\Sigma'), \omega\otimes \omega \rangle> c|\omega|^2,\quad \forall \omega\in W^*\subset V^*,
			\]
			whence
			\[
			\langle D_\Sigma \mathcal{J}(\Sigma'), A\rangle \geq 0,\quad \text{for all $A\in Sym^2W^*\subset Sym^2V^*$ positive semi-definite}. 
			\]
			Each $\Sigma'$ in a small neighbourhood of $\Sigma$, can be decomposed  uniquely as $\Sigma'=\Sigma''+A$, where $\Sigma''\in \mathcal{S}_l$ is a rank $l$ positive semi-definite form near $\Sigma$, and $A\in Sym^2 W^*\subset Sym^2 V^*$ is a small positive semi-definite form whose null space contains $W'$. Concretely,
			\[
			\Sigma'= \begin{bmatrix}
				\Sigma_{11}, \Sigma_{12}
				\\
				\Sigma_{12}^T, \Sigma_{22}
			\end{bmatrix} \in \begin{bmatrix}
				Sym^2 W^*,  W^*\otimes W'^*
				\\
				W'^*\otimes W^*,  Sym^2 W'^*
			\end{bmatrix}    ,
			\Sigma''=  \begin{bmatrix}
				\Sigma_{12} \Sigma_{22}^{-1} \Sigma_{12}^T, &\Sigma_{12}
				\\
				\Sigma_{12}^T, & \Sigma_{22}
			\end{bmatrix}, A=\Sigma'-\Sigma''.
			\]
			Thus
			\[
			\mathcal{J}(\Sigma')= \mathcal{J}(\Sigma'') + \int_0^1 \langle D_\Sigma \mathcal{J}(\Sigma''+ tA), A\rangle dt \geq \mathcal{J}(\Sigma'') \geq \mathcal{J}(\Sigma).
			\]
			The last inequality uses that $\Sigma$ is a local minimizer for $\mathcal{J}$ restricted to $\mathcal{S}_l$.
		\end{proof}

		\subsubsection{Boundary vacua and variable selection}\label{sect:boundaryvacuavariableselection}

		Variable selection means identifying the essential degrees of freedom in $X$ for the prediction of $Y$. \blue{
			The following observation, motivated by the multi-index setting of Example~\ref{eg:variable-selection}, suggests that the vacua of $\mathcal{J}$ is relevant for the variable selection problem.}

		\begin{lem}\cite[Lemma  11]{CLLR}
			\label{lem:solution-low-rank}
			Given a direct sum decomposition $V=W\oplus W'$, so that $X=(X_W, X_{W'}) $, we suppose that $X_{W}$ is independent of both $Y$ and $X_{W'}$. Then for any $U\in \End(V)$ as in (\ref{eqn:Umatrix}), we have
			\[
			J( U=\begin{bmatrix}
				U_{11}, & U_{12}\\
				U_{12} , & U_{22}
			\end{bmatrix}  ,\lambda) \geq J(\tilde{U}=\begin{bmatrix}
				0, & U_{12}\\
				0, & U_{22}
			\end{bmatrix}  ,\lambda).
			\]
		\end{lem}

		\begin{proof}
			By the translation invariance of the RKHS $\mathcal{H}$, and the minimizing property of $J(\tilde{U}
			,\lambda   )$, we deduce for any $\xi\in V$ and $f\in \mathcal{H}$ that
			\[
			\frac{1}{2}  \E[|Y- f(\xi +(U_{12}X_{W'}, U_{22} X_{W'}))|^2] + \frac{\lambda}{2} \norm{f}^2_\mathcal{H} \geq J (\tilde{U},\lambda   ). 
			\]
			Since $X_W$ is independent of $Y$ and $X_{W'}$ by assumption, for any fixed value of $X_W$, we can choose $\xi=(U_{11}X_W, U_{21}X_W)$, to deduce 
			\[
			\frac{1}{2}  \E[|Y- f(U_{11}X_W+U_{12}X_{W'}, U_{22} X_{W'})|^2|X_W ] + \frac{\lambda}{2} \norm{f}^2_\mathcal{H}\geq J (\tilde{U},\lambda   ).
			\]
			Now we write the definition of $I(f,U,\lambda)$ in terms of conditional expectations, 
			\[
			\begin{split}
				& I(f,U,\lambda) 
				\\
				=&  \frac{1}{2} \E_{X_W} [\E[|Y- f(U_{11}X_W+U_{12}X_{W'}, U_{12}X_W+ U_{22} X_{W'})|^2|X_W ] ]+ \frac{\lambda}{2} \norm{f}^2_\mathcal{H}
				\\
				=&   \E_{X_W} [ \frac{1}{2} \E[|Y- f(U_{11}X_W+U_{12}X_{W'}, U_{12}X_W+ U_{22} X_{W'})|^2   |X_W ] + \frac{\lambda}{2} \norm{f}^2_\mathcal{H}]
				\\
				\geq &  \E_{X_W} [ J (\tilde{U},\lambda   ) ]= J (\tilde{U},\lambda   ). 
			\end{split}
			\]
			By minimizing over $f\in \mathcal{H}$ we deduce the Lemma.
		\end{proof}

		\begin{cor}
			If $\mathcal{J}$ achieves its global minimum in $Sym^2_{\geq 0}$, then there is some boundary vacuum $\Sigma \in Sym^2_{\geq 0}$ achieving this minimum, whose null subspace contains $W$.
		\end{cor}

		\begin{proof}
			If $\Sigma'=U^T U$ is a minimizer of $\mathcal{J}$, then so is $\Sigma= \tilde{U}^T\tilde{U} $, whose null subspace contains $W$.
		\end{proof}

		\subsubsection{First variation at  boundary points}

		The first variation at a boundary point affords some simplifications. As before, $\Sigma\in \mathcal{S}_l$, and $W$ is the null subspace of $\Sigma$. We denote $X=(X_W, X_{W'})$ according to the decomposition $V=W\oplus W'$, and let $Y_0=\E[Y|X_{W'}]$ be the conditional expectation, so $Y=Y_0+Y_1$ with $\E[Y_1|X_{W'}]=0$. We can write $\Sigma=U^T U$, and we recall that  $F_\Sigma(x)=f_U(Ux)$ (\cf Lemma \ref{lem:Jcontinuity2}).

		\begin{lem}
			The function $F_\Sigma(x)$ descends to a function on $W'\simeq V/W$, and is the unique $C^0$-solution to the integral equation
			\red{\[
				f(x)= \frac{1}{\lambda} \E[ (Y_0-f(X))\mathcal{K}(|x-X|_\Sigma^2)  ].
				\]}
			In particular it depends on $Y$ only through $Y_0$. 
		\end{lem}

		\begin{proof}
			We write $\Sigma=U^T U$. Since $\Sigma$ has null space $W$, we know $U(x_W, x_{W'})$ depends only on $x_{W'}$, and so is the function $F_\Sigma= f_U(Ux)$. Since $f_U\in \mathcal{H}\subset C^0$, this $F_\Sigma$ is a continuous function on $V$. 
			By the Euler-Lagrange equation (\ref{eqn:EulerLagrange1}),
			\[
			F_\Sigma(x)+ \frac{1}{\lambda} \E[ (Y-F_\Sigma(x))\mathcal{K}(|x-X|_\Sigma^2)  ]=0.
			\]
			However $ \mathcal{K}(|x-X|_\Sigma^2)$
			depends on $x-X$ only through the $W'$ component, so we can replace $Y$ by $Y_0$ in the integral equation. The uniqueness of the solution follows from the coercivity of the integral equation, similar to Proposition  
			\ref{prop:differentiabilityminimizer}.
		\end{proof}

		The first variation 
		\[
		D_\Sigma \mathcal{J}(\Sigma)= -\frac{1}{2\lambda}   \E[  r_\Sigma(X,Y) \overline{ r_\Sigma(X',Y')}  \mathcal{K}'(|X-X'|_{\Sigma}^2) (X-X')\otimes (X-X') ] 
		\]
		is valued in 
		\[
		Sym^2 V\simeq Sym^2 W\oplus Sym^2 W'\oplus W\otimes W'. 
		\]
		We wish to identify these three components. As a computational trick, we write $r_\Sigma(X,Y)=(Y_0- F_\Sigma(X))+Y_1$, and expand the above formula, and then take the conditional expectation for fixed $X_{W'}, X_{W'}'$. The formula simplifies because  $F_\Sigma(X)$ and $\mathcal{K}'(|X-X'|^2_\Sigma)$ depend only  on the $W'$ component but not $X_W, X_W'$, and only \emph{up to second order conditional moments} for $Y$ enter into the formula.

		\begin{enumerate}
			\item The $Sym^2 W'$ component of $D_\Sigma \mathcal{J}(\Sigma)$ is
			\[
			-\frac{1}{2\lambda}   \E[  (Y_0-F_\Sigma(X)) \overline{ (Y_0'-F_\Sigma(X'))}  \mathcal{K}'(|X-X'|_{\Sigma}^2) (X_{W'}-X'_{W'})\otimes (X_{W'}-X'_{W'})].
			\]
			At a vacuum $\Sigma$, this should vanish.
			
			\item 
			The $W\otimes W'$ component of $D_\Sigma \mathcal{J}(\Sigma)$ is
			\[
			-\frac{1}{2\lambda}   \E[  (Y-F_\Sigma(X)) \overline{ (Y'-F_\Sigma(X'))}  \mathcal{K}'(|X-X'|_{\Sigma}^2) (X_{W}-X'_{W})\otimes (X_{W'}-X'_{W'})],
			\]
			which is equal to
			\[
			\begin{split}
				& \frac{1}{2\lambda}  \E[ \mathcal{K}'(|X-X'|_{\Sigma}^2)
				\\
				&
				\{    (Y_0-F_\Sigma(X)) \overline{ \E[Y_1' X_W'|X_{W'}'] } - \E[Y_1 X_W|X_{W'}]     \overline{ (Y_0'-F_\Sigma(X'))} \} 
				\otimes (X_{W'}-X'_{W'})]. 
			\end{split}
			\]
			At a vacuum $\Sigma$, this should vanish; a sufficient condition is that the first conditional moment $\E[Y_1X_W|X_{W'}]=0$.

			\item 
			The $Sym^2 W$ component of $D_\Sigma \mathcal{J}(\Sigma)$ is the sum of three contributions:
			\begin{equation}\label{eqn:normalgradient1}
				\begin{split}
					-\frac{1}{2\lambda}   \E[  (Y_0-F_\Sigma(X)) \overline{ (Y_0'-F_\Sigma(X'))}  \mathcal{K}'(|X-X'|_{\Sigma}^2) (X_{W}-X'_{W})\otimes (X_{W}-X'_{W})]  
				\end{split}
			\end{equation}
			and
			\begin{equation}\label{eqn:normalgradient2}
				\begin{split}
					&  -\frac{1}{2\lambda}   \E[  Y_1 \overline{ Y_1'}  \mathcal{K}'(|X-X'|_{\Sigma}^2) (X_{W}-X'_{W})\otimes (X_{W}-X'_{W})]
					\\
					=& \frac{1}{\lambda} \text{Re}  \E[ \mathcal{K}'(|X-X'|_{\Sigma}^2)  \E[Y_1 X_W|X_{W'}] \otimes \overline{\E[Y_1' X'_W|X'_{W'}] }  ],
				\end{split}
			\end{equation}
			and a cross term
			\[
			-\frac{1}{\lambda}  \text{Re} \E[  (Y_0-F_\Sigma(X)) \overline{ Y_1'}  \mathcal{K}'(|X-X'|_{\Sigma}^2) (X_{W}-X'_{W})\otimes (X_{W}-X'_{W})]. 
			\]
			If $\Sigma$ is a vacuum, then the sum of these three terms should pair non-negatively with any positive semi-definite form in $Sym^2 W^*$. We notice that (\ref{eqn:normalgradient1}) depends on $Y$ only through $Y_0$, and vanishes when $Y_0=0$ identically; the second term (\ref{eqn:normalgradient2}) depends on $Y$ only through the first conditional moment, and vanishes when 
			$\E[Y_1 X_W|X_{W'}]=0$. The cross term depends on $Y$ only up to second conditional moment, and vanishes if either $Y_0=0$, or the first and second conditional moment of $Y_1$ vanishes.

		\end{enumerate}

		\begin{lem}\label{lem:linearsignal}
			Suppose $\mathcal{K}$ is a completely monotone kernel function. Then the term (\ref{eqn:normalgradient2}) pairs non-positively with any positive semi-definite form in $Sym^2 W^*$. 
		\end{lem}

		\begin{proof}
			It suffices to prove that for any $w\in W^*$, the pairing
			\[
			\begin{split}
				& \langle \frac{1}{\lambda}   \E[ \mathcal{K}'(|X-X'|_{\Sigma}^2)  \E[Y_1 X_W|X_{W'}]  \otimes \overline{  \E[Y_1' X'_W|X'_{W'}] }, w\otimes w\rangle
				\\
				&  = \frac{1}{\lambda}   \E[ \mathcal{K}'(|X-X'|_{\Sigma}^2) \langle \E[Y_1 X_W|X_{W'}] , w\rangle \overline{ \langle \E[Y_1' X'_W|X'_{W'}], w\rangle  } ] \leq 0.  
			\end{split}
			\]
			By the Gaussian sum representation formula (\ref{eqn:GaussiansumK'}) and Fourier transform,
			\[
			\begin{split}
				& \mathcal{K}'(|x|_\Sigma^2)= -\int_0^\infty e^{-t|Ux|^2} td\nu(t)
				\\
				=& - \int_0^\infty \int_{V^*} (\pi t^{-1})^{d/2} e^{-\pi^2|\omega|^2 /t} e^{i\langle Ux, \omega\rangle} td\nu(t).
			\end{split}
			\]
			We denote $f(X_W)= \langle \E[Y_1 X_W|X_{W'}] , w\rangle$, and compute
			\[
			\begin{split}
				& \E[ \mathcal{K}'(|X-X'|_{\Sigma}^2) f(X_{W'}) \otimes \overline{ f(X_{W'}') }  ]  
				\\
				=&  -\int_0^\infty\int_{V^*} (\pi t^{-1})^{d/2}\exp(-\pi^2|\omega|^2/t) |\E[  f(X_{W'}) e^{i\langle UX_{W'}, \omega\rangle}   ]|^2 td\nu(t)
				\leq 0.
			\end{split}
			\]
			This verifies the non-positivity claim. 
		\end{proof}

		\begin{lem}\label{lem:boundarystability}(Compare \cite[Theorem 5]{CLLR})
			Suppose that $\E[X_W|X_{W'}]$ and $\E[X_W\otimes X_W|X_{W'}]$ are both independent of $X_{W'}$, and suppose $\mathcal{K}$ is a completely monotone kernel function. Then the term (\ref{eqn:normalgradient1}) pairs non-negatively with any positive semi-definite form in $Sym^2 W^*$. 
		\end{lem}

		\begin{proof}
			Under the independence assumption, (\ref{eqn:normalgradient1}) is equal to
			\[
			-\frac{1}{\lambda}   \E[  (Y_0-F_\Sigma(X) \overline{ (Y_0'-F_\Sigma(X'))} \mathcal{K}'(|X-X'|_{\Sigma}^2) ]Cov(X_W). 
			\]
			It suffices to prove that for any $w\in W^*$,
			\[
			-\frac{1}{\lambda}   \E[  (Y_0-F_\Sigma(X) \overline{ (Y_0'-F_\Sigma(X'))} \mathcal{K}'(|X-X'|_{\Sigma}^2) ]Cov(X_W)(w,w)\geq 0.
			\]
			This follows from a similar computation as in Lemma \ref{lem:linearsignal}.
		\end{proof}

		\subsubsection{Discussions}\label{Discussion:boundary}


		\begin{enumerate}
			\item  A boundary point $\Sigma\in \mathcal{S}_l$ specifies some signal variables $X_{W'}$, and $Y_0$ measures the information of $Y$ predictable from $X_{W'}$, while $Y_1$ detects how $Y$ depends on the secondary variables $X_W$ conditional on $X_{W'}$. Under the assumptions of Proposition  \ref{prop:continuousextensionJ}, the first variation at $\Sigma$ only depends on $Y_1$ through the up to second order conditional moments $\E[Y_1 X_W|X_{W'}]$ (the `linear signal'), and $\E[Y_1 X_W\otimes X_W|X_{W'}]$ (the `quadratic signal').

			\item  
			Lemma \ref{lem:boundarystability} suggests that if both the linear and quadratic signals are weak compared to $Y_0$, and the covariance matrix of $X_W$ conditional on $X_{W'}$ is independent of $X_{W'}$, then $\mathcal{J}$ has the tendency to increase as $\Sigma$ moves into the interior of $Sym^2_+$. If $\Sigma$ locally minimizes $\mathcal{J}$ restricted to $\mathcal{S}_l$, then we expect $\Sigma$ is a boundary vacuum.  The interpretation is that the information of the secondary variables $X_W$ is suppressed by the signal variables $X_{W'}$, and the prediction function $F_\Sigma$ depends only on $X_{W'}$, which has lower dimension compared to $X$, so we have achieved \emph{variable selection}.

			\item A competing effect is that by Lemma \ref{lem:linearsignal}, the linear signal drives $\mathcal{J}$ to decrease as $\Sigma$ moves into the interior of $Sym^2_+$.

			\item 
			If one wishes to design the algorithm so that the linear signal effect dominates, then one can first try to learn $Y_0=\E[Y|X_{W'}]$ to sufficient accuracy,  which would be much more efficient than the higher dimensional problem of learning $Y$. 
			Then one replaces $Y$ with $Y-Y_0$, to reduce to the situation where $Y_0=0$. In the special case $W'=0$, then $Y_0=\E[Y]$ is simply the average value of $Y$, and the above discussion says that subtracting this average is beneficial for promoting the linear signal. This  subtraction has a similar effect as the common practice of inserting an `intercept', namely to introduce an extra parameter $c$ into the kernel ridge regression loss function
			\[
			\frac{1}{2}\E[ |Y-f(UX)- c|^2 ]+ \frac{\lambda}{2} \norm{f}_{\mathcal{H}}^2,
			\]
			and minimize $f$ together with $c$. 
			
		\end{enumerate}

		\subsection{Vacua at infinity}\label{sect:DiscreteI}

		\begin{Question}
			Does the functional $\mathcal{J}(\Sigma;\lambda)$ always achieve its minimum in $Sym^2_{\geq 0}$? 
		\end{Question}

		By Theorem \ref{thm:asymptoticvalue}, this question is sensitive to whether or not the marginal distribution of any component of $X$ contains a discrete part. In the \emph{continuous variable} case, we have
		\[
		\lim_{\Sigma\to \infty} \mathcal{J}(\Sigma;\lambda) =
		\frac{1}{2} \E[|Y|^2].  
		\]
		In particular no vacua can sit at infinity in $\overline{Sym^2_{\geq 0}}$.

		We now revisit the setting of \emph{discrete variables} as in Section \ref{sect:representer}, where $X$ takes value in a finite number of points $\{ a_1,\ldots a_m\} \subset V $, with $\mathbb{P}(X=a_i)=p_i$, and $Y=y_i$ conditional on $X=a_i$. The representer theorem Proposition  \ref{prop:representer} says
		\[
		\mathcal{J}(\Sigma;\lambda)= \frac{1}{2}\sum_1^m p_i y_i^2-\frac{1}{2} \sum_{i,j=1}^m(\lambda M^{-1}+P)^{-1}_{ij} p_i y_i p_j\bar{y}_j,
		\]
		where $M=(M_{ij})$ is the $m\times m$ inner product matrix with entries
		\[
		M_{ij}=K(Ua_i, Ua_j)= K( U(a_i-a_j)) = \mathcal{K}( |a_i-a_j|_{\Sigma}^2),
		\]
		and $P=\text{diag}(p_1,\ldots p_m)$ encodes the probability of the atoms at $a_i$.

		\begin{cor}\label{cor:largeSigmaasymptote1}
			(Asymptotic for large $\Sigma$) Suppose that $\min_{i,j} |a_i-a_j|_{\Sigma}$ is large. Then
			\[
			\mathcal{J}(\Sigma;\lambda)= \frac{1}{2} \sum_{i=1}^m \frac{\lambda p_i }{\lambda +p_i} |y_i|^2- \frac{\lambda}{2} \sum_{i\neq j,1\leq i,j\leq m} \frac{p_i y_i}{ \lambda+p_i  } \frac{p_j \bar{y}_j}{ \lambda+p_j  } \mathcal{K}( |a_i-a_j|_{\Sigma}^2) +O(|M-I|^2).
			\]
		\end{cor}

		\begin{proof}
			By the decay of the kernel function, $M=I+o(1)$ when $\min_{i,j} |a_i-a_j|_{\Sigma} \to +\infty$. Thus we can make the expansion
			\[
			\begin{split}
				& (P+\lambda M^{-1})^{-1}= ( P+ \lambda I-\lambda(M-I) +O(|M-I|^2) )^{-1}
				\\
				= & (P+\lambda I)^{-1} + \lambda (P+\lambda I)^{-1}(M-I)(P+\lambda I)^{-1} + O( |M-I|^2).
			\end{split}
			\]
			Using the representer theorem,
			\[
			\begin{split}
				& \mathcal{J}(\Sigma;\lambda)=  \frac{1}{2}\sum_1^m p_i y_i^2-\frac{1}{2} \sum_{i,j=1}^m(\lambda M^{-1}+P)^{-1}_{ij} p_i y_i p_j\bar{y}_j
				\\
				=& \frac{1}{2} \sum_{i=1}^m \frac{\lambda p_i}{\lambda +p_i} |y_i|^2 -\frac{\lambda}{2} \sum_{i,j=1}^m \frac{p_i y_i}{ \lambda+p_i  } \frac{p_j \bar{y}_j}{ \lambda+p_j  } (M-I)_{ij} +O(|M-I|^2)
				\\
				=&  \frac{1}{2} \sum_{i=1}^m \frac{ \lambda p_i}{\lambda +p_i} |y_i|^2 -\frac{\lambda}{2} \sum_{i\neq j}  \frac{p_i y_i}{ \lambda+p_i  } \frac{p_j \bar{y}_j}{ \lambda+p_j  } \mathcal{K}( |a_i-a_j|_{\Sigma}^2) +O(|M-I|^2).
			\end{split}
			\]
			as required.
		\end{proof}
		
		\blue{
			\begin{rmk}
				The asymptotic expansion in Corollary~\ref{cor:largeSigmaasymptote1} is a \emph{fixed-$\lambda$ statement}. Besides the explicit $\lambda$-dependence of the leading and subleading terms, the constant implicit in the $O(|M-I|^2)$ remainder also depends on $\lambda$. Consequently, the expansion is not uniform in the limit $\lambda\to 0^+$. Extending the analysis to a singular regime in which $\lambda \to 0^+$ would require controlling the joint behavior of the limits $\Sigma\to\infty$ and $\lambda\to 0^+$, which lies beyond the scope of the present paper.
			\end{rmk}
		}

		\subsubsection{Discussions}

		\begin{enumerate}
			\item  The limiting value is
			\begin{equation}\label{eqn:Jlimitingvalue}
				\lim_{\Sigma\to +\infty} \mathcal{J}(\Sigma;\lambda)= \frac{1}{2} \sum_{i=1}^m \frac{\lambda p_i}{\lambda +p_i} |y_i|^2.
			\end{equation}
			Notice that the contributions of the different atoms completely decouple, and we recover Theorem \ref{thm:asymptoticvalue}.

			\item 
			We notice a few features of the limiting value (\ref{eqn:Jlimitingvalue}). First, this formula is an increasing function of $\lambda$, and goes to zero as $\lambda\to 0$. The interpretation is that there is a tradeoff between regularization effect and approximation, and perfect approximation is only achieved when the regularization goes to zero. Second, for fixed $\lambda$,  if $m$ is large and $\max p_i$ is small compared to $\lambda$ (e.g. if $\max p_i$  comparable to $m^{-1}$), then (\ref{eqn:Jlimitingvalue}) is close to $\frac{1}{2}\E|Y|^2$. The interpretation is that discrete variables behave like continuous variables when the probability density is spread out throughout many atoms.
			
			\item Corollary \ref{cor:largeSigmaasymptote1} admits a nice `physical' interpretation: in the large $\Sigma$ regime, the $m$ particles interact via a potential of the form
			\[
			-\frac{\lambda}{2} \sum_{i,j=1}^m \frac{p_i y_i}{ \lambda+p_i  } \frac{p_j \bar{y}_j}{ \lambda+p_j  } \mathcal{K}( |a_i-a_j|_{\Sigma}^2).
			\]
			The higher term $O(|M-I|^2)$ is negligible. For completely monotone kernels $\mathcal{K}$ is positive, so the \emph{sign of the potential} depends on the signs of $\text{Re}(y_i\bar{y}_j)$, which physically corresponds to \emph{attractive vs. repulsive forces} between the particles.

			\item 
			
			Suppose $m=2$, and $\text{Re}(y_1\bar{y}_2)<0$.
			For all $\Sigma$ such that $\min_{i,j} |a_i-a_j|_{\Sigma}$ is sufficiently large, we have
			\[
			\mathcal{J}(\Sigma;\lambda)>  \frac{1}{2} \sum_{i=1}^m \frac{ \lambda p_i}{\lambda +p_i} |y_i|^2.
			\]
			The interpretation is that $\frac{1}{2} \sum_{i=1}^m \frac{p_i^2}{\lambda +p_i} |y_i|^2$ should be viewed as a local minimum value for $\mathcal{J}$, achieved at \emph{a vacuum at infinity}. Consequently, \emph{vacua at infinity can occur sometimes in the presence of discrete structures}.

			The converse is not true: the presence of discrete structure does not imply the existence of any vacua at infinity. The easiest example is when $m=2$ and $\text{Re}(y_1\bar{y}_2)>0$.

			\item

			For Sobolev and Gaussian kernels, the potential decays exponentially at infinity. If we would like to detect the vacua at infinity by running some gradient flow on $Sym^2_{\geq 0}$, then we expect to require infinite time to reach such vacua. To speed up the flow, one should use a kernel function $\mathcal{K}$ which decays more slowly at infinity.

			\item More generally $\Sigma$ can converge to infinity, but remain bounded on some proper subspace $W\subset V$. The limiting value for $\mathcal{J}$ is sensitive to the choice of $W$. 
			
		\end{enumerate}


		\section{Cluster interactions}\label{sect:clusterinteractions}

		In the kernel ridge regression problem, we now assume there is a hidden discrete random variable $X_0$ valued in $\{1, 2,\dots m\}$, which is not directly accessible to the observer, but controls how the data $(X,Y)$ is generated. One should imagine that for different $X_0=i$, the conditional probability distributions of $X$ have distinct characteristic behaviors, e.g., their supports may cluster around far separated points, or they may have very different scale parameters. Typical examples are Example \ref{eg:twoscale} and \ref{eg:multiscale}. 
		\blue{We shall develop more functional analytic tools to analyze the interactions between clusters of  probability distributions. The main result is a collection of operator norm estimates for `cluster operators', showing that interactions between clusters are weak when these clusters are sufficiently separated in space or scale. As a consequence, the kernel ridge regression problem approximately decomposes into simpler cluster-wise problems.}
		
		\blue{These operator estimates have two main applications. First, they provide a quantitative, non-asymptotic counterpart to the asymptotic decoupling phenomena discussed in Section~\ref{Translationinvkernel}; see Section~\ref{sect:noninteraction}. Second, they form a key technical ingredient in the analysis of multi-scale data distributions; specifically, with the dimensional reduction mechanism developed in Section~\ref{sect:dimreductionmechanism}, they allow us to isolate the contributions of individual clusters and construct vacua associated with different scales, see Section~\ref{sect:Landscape2}.
		}

		\subsection{Operator theoretic framework}\label{sect:operator}


			The Euler-Lagrange equation (\ref{EulerLagrange}) can be cast in the operator theoretic framework. Following Cucker and Smale's work on standard kernel ridge regression~\cite{Smale},
			we define the bounded self-adjoint operator $T$ by
			\[
			\E[ f(UX) \bar{g}(UX)]= (Tf, g)_{\mathcal{H}},\quad \forall g\in \mathcal{H}. 
			\]
			By substituting $g=K(x,\cdot)$, we obtain the explicit formula
			\[
			Tf(x)= (Tf, K(x,\cdot))_{\mathcal{H}}= \E[ f(UX) \overline{K(x, UX)}] .
			\]
			We define $\mathcal{Y}\in \mathcal{H}$ by Riesz representation,
			\[
			\E[ Y \bar{g}(UX)  ]= (\mathcal{Y}, g)_{\mathcal{H}},\quad \forall g\in \mathcal{H}. 
			\]
			This admits the explicit formula $\mathcal{Y}= \E[  Y\overline{K( \cdot, UX)}   ].$
			Then the Euler-Lagrange equation (\ref{EulerLagrange}) becomes a linear equation in the Hilbert space $\mathcal{H}$,
			\begin{equation}\label{eqn:EulerLag2}
				(\lambda I+  T ) f_U = \mathcal{Y}. 
			\end{equation}

			\begin{lem}\label{lem:normY}
				Let $(X',Y')$ be an independent copy of $(X,Y)$. Then
				\[
				\norm{\mathcal{Y}}^2_{\mathcal{H}} =   \E[ Y\bar{Y}' K( UX,UX')  ].
				\]
				
			\end{lem}
			
			\begin{proof}
				We  apply the reproducing kernel property, and the fact that $K(x,y)=\overline{K(y,x)}$, to compute
				\[
				\begin{split}
					\norm{\mathcal{Y}}^2_{\mathcal{H}} = & (\E[  YK( UX, \cdot)  ], \E[  Y'K( UX',\cdot)   ])_{\mathcal{H}}
					\\
					= &   \E[ Y\bar{Y}' (K( UX,\cdot), K(UX',\cdot) )_{\mathcal{H}}   ]
					\\
					= &   \E[ Y\bar{Y}' K( UX,UX')  ]
				\end{split}
				\]
				as required.
			\end{proof}

			\red{\begin{cor}\label{Cor:quantitativenonfitting} 
					(Decay estimate for the minimizer)
					We have
					\[
					\begin{cases}
						\lambda\norm{f_U}_{\mathcal{H} }^2 + \E[ |f_U(UX)|^2]  \leq \frac{1}{\lambda}\norm{\mathcal{Y}}_{\mathcal{H}}^2 = \frac{1}{\lambda}\E[ Y\bar{Y}' K( UX, UX')  ],
						\\
						J(U,\lambda)\geq  \frac{1}{2} \E[|Y|^2]-(2\lambda)^{-1} \E[ Y\bar{Y}' K( UX,UX'))  ].
					\end{cases}
					\]
					In particular, if $X$ is a continuous variable, then
					\[
					\lambda \norm{f_U}_{\mathcal{H}}^2+\E[ |f_U(UX)|^2]\to 0, \quad U\to \infty.
					\]
				\end{cor}
			}	
			
			\begin{proof}
				Using the Euler-Lagrange equation,
				\[
				\lambda \norm{f_U}_{\mathcal{H}}^2 \leq ((\lambda I+T) f_U, f_U)_{\mathcal{H}} = (\mathcal{Y}, f_U)_{\mathcal{H}} \leq \norm{f_U}_{\mathcal{H} } \norm{\mathcal{Y}}_{\mathcal{H}},
				\]
				hence $\norm{f_U}_{\mathcal{H} } \leq \lambda^{-1} \norm{\mathcal{Y}}_{\mathcal{H}} $, so
				\[
				\lambda\norm{f_U}_{\mathcal{H} }^2 + \E[ |f_U(UX)|^2]\leq \frac{1}{\lambda}\norm{\mathcal{Y}}_{\mathcal{H}}^2 = \frac{1}{\lambda}\E[ Y\bar{Y}' K( UX,UX')  ].
				\]
				Therefore
				\[
				\begin{split}
					J(U,\lambda)=&  \frac{1}{2} \E[|Y|^2]- \frac{1}{2} E[Y \overline{f_U(UX)} ]= \frac{1}{2} \E[|Y|^2]-\frac{1}{2} (\mathcal{Y}, f_U)_{\mathcal{H}}
					\\
					\geq & \frac{1}{2} \E[|Y|^2]-(2\lambda)^{-1} \norm{\mathcal{Y}}_{\mathcal{H}}^2 \geq \frac{1}{2} \E[|Y|^2]-(2\lambda)^{-1} \E[ Y\bar{Y}' K( UX,UX')  ].
				\end{split}
				\]

				By Cauchy-Schwarz, 
				\[
				\E[ Y\bar{Y}' K( UX,UX')  ] \leq \E[|Y|^2] \E[ |K( UX,UX')|^2  ]^{1/2}.
				\]
				In particular when $X$ is a continuous variable, then $X-X'$ is a continuous variable, so by dominated convergence $\E[ |K( U(X-X'))|^2  ] \to 0 $ as $U\to \infty$, whence $\lambda \norm{f_U}_{\mathcal{H}}^2+\E[ |f_U(UX)|^2]\to 0$.
			\end{proof}

			\subsubsection{Operator norm estimates}
			
			\begin{prop}\label{prop:operatorT}
				(Operator norm estimate) Let $X'$ denote an independent copy of $X$. Then 
				\[
				\norm{Tf}_{\mathcal{H}}^2\leq   \norm{f}_{C^0}^2 \E[ |K(U(X-X'))|  ] \leq \norm{f}_{\mathcal{H}}^2 \E[ |K(U(X-X'))|  ] ,
				\]
				\[
				\norm{Tf}_{\mathcal{H}}^2\leq \E[|f(UX)|^2 ] \E[ |K(U(X-X'))|^2  ]^{1/2} ,
				\]
				\[
				\norm{T\mathcal{Y} }_{\mathcal{H}}^2 \leq  \E[|Y|^2] \E[ |K(U(X-X')|^2  ].
				\]
				Moreover, the Hilbert-Schmidt norm of $T$ acting on $\mathcal{H}$ is 
				\[
				\norm{T}_{HS}^2= \E[ |K(U(X-X')|^2  ].
				\]

			\end{prop}
			
			\begin{proof}
				We use the $T^*T$-argument. For any $g\in \mathcal{H}$, we compute
				\[
				(T^*Tf,g)_{\mathcal{H}}= (Tf, Tg)_{\mathcal{H}}= \E[ f(UX') \overline{Tg(UX')}  ].
				\]
				But $Tg(x)= \E[   g(UX) \overline{K(x, UX)}  ]$, so 
				\[
				\begin{split}
					(T^*Tf, g)_{\mathcal{H}}  
					=  \E[  f(UX') \overline{K(U(X-X'))}  \overline{g(UX) }  ] .
				\end{split}
				\]
				In particular
				\[
				\norm{Tf}_{\mathcal{H}}^2= (T^*Tf,f)\leq   \norm{f}_{C^0}^2 \E[ |K(U(X-X'))|  ] \leq \norm{f}_{\mathcal{H}}^2 \E[ |K(U(X-X'))|  ] .
				\]
				Moreover by Cauchy-Schwarz,
				\[
				\norm{Tf}_{\mathcal{H}}^2= \E[  f(UX') \overline{K(U(X-X'))}  \overline{f(UX) }  ]\leq \E[|f(UX)|^2 ] \E[ |K(U(X-X'))|^2  ]^{1/2} .
				\]

				Similarly,  for any $g\in \mathcal{H}$,
				\[
				|(T\mathcal{Y}, g)|= |\E[  Y' \overline{K(UX, UX')} \overline{g(UX)}   ] |\leq \E[|Y|^2]^{1/2} \E[ |K(U(X-X'))|^2  ]^{1/2} \norm{g}_{\mathcal{H}},   
				\]
				whence $\norm{T\mathcal{Y} }_{\mathcal{H}} \leq  \E[|Y|^2]^{1/2} \E[ |K(U(X-X')|^2  ]^{1/2}   $.
				
				We now compute the Hilbert-Schmidt norm, by taking an orthonormal basis $\psi_i$ of $\mathcal{H}$, and repeatedly applying the reproducing kernel property:
				\[
				\begin{split}
					& \norm{T}_{HS}^2= \sum_i \norm{ T\psi_i}_{\mathcal{H}}^2= \E[ \sum_i \psi_i(UX') \overline{K(U(X-X'))}  \overline{\psi_i(UX) }  ] 
					\\
					= & \sum_i \E[  (\psi_i, K(UX', \cdot) )_{\mathcal{H}} \overline{K(U(X-X'))}  \overline{   (\psi_i, K(UX, \cdot) )_{\mathcal{H}}   }  ]
					\\
					=&   \E[  (K(UX, \cdot), K(UX',\cdot) )_{\mathcal{H}} \overline{K(U(X-X'))}   ]
					\\
					=&  \E[  K(UX, UX') \overline{K(U(X-X'))}   ]= \E[ |K(U(X-X')|^2  ].
				\end{split}
				\]
				as required. We note that similar computational ideas appeared in~\cite{GrettonBoSmSc05}.
			\end{proof}

			We now derive some sufficient conditions for the operator norm of $T$ to be small, in the case of completely monotone kernels. 
			
			\begin{lem}\label{lem:Kdoubleintegralbound}
				Let $W\subset V$ be a proper subspace, and the probability density of the marginal distribution of $X$ on $V/W$ has $L^\infty$-bound by $p_0$ with respect to the Lebesgue measure on $W'\simeq V/W$. Suppose $K(x)=\mathcal{K}(|x|_V^2)$ is the completely monotone kernel (\ref{eqn:completelymonotonekernel}). The positive semi-definite $\Sigma$ induces the quotient inner product $\Sigma_{V/W}$ on $V/W$, namely
				\[
				\norm{x}_{\Sigma_{V/W}}^2= \min\{ |x+y|_\Sigma^2: y\in W \}. 
				\]
				Then
				\[
				\E[ \mathcal{K}(|X-X'|_{\Sigma}^2)   ] \leq  \frac{p_0}{  \sqrt{\det \Sigma_{V/W}} } \int_0^\infty (\pi t^{-1})^{d_{W'}/2}  d\nu(t)  .
				\]
				\[
				\E[ \mathcal{K}^2(|X-X'|_{\Sigma}^2)   ] \leq 
				\frac{p_0}{  \sqrt{\det \Sigma_{V/W}} } \iint (\pi (t+s)^{-1})^{d_{W'}/2}  d\nu(t) d\nu(s).
				\]

			\end{lem}

			\begin{proof}
				For the completely monotone kernel, $\mathcal{K}$ is a decreasing function non-negative function, so
				\[
				\begin{split}
					\E[ \mathcal{K}(|X-X'|_{\Sigma}^2)   ] & \leq \E[ \mathcal{K}(|X-X'|_{\Sigma_{V/W}}^2)   ]
					\\
					&= \E[ \E[ \mathcal{K}(|X_{W'}-X'_{W'}|_{\Sigma_{V/W}}^2) || X_{W'}, X_{W'}']  ]
					\\
					& \leq \E_{X_{W'}}[ p_0\int_{W'} \mathcal{K}(|y-X_{W'}|_{\Sigma_{V/W}}^2) dy ]
					\\
					& = p_0\int_{W'} \mathcal{K}(|y|_{\Sigma_{V/W}}^2) dy
					\\
					& =  \frac{p_0}{  \sqrt{\det \Sigma_{V/W}} } \int_{W'} \mathcal{K}(|y|^2) dy.
				\end{split}
				\]
				where the second line uses conditional expectation, and the third line uses the $L^\infty$-bound on the marginal density of $X_{W'}'$. By the Gaussian sum representation (\ref{eqn:completelymonotonekernel}),
				\[
				\int_{W'} \mathcal{K}(|y|^2) dy= \int_0^\infty e^{-t|y|^2 } d\nu(t)= \int_0^\infty (\pi t^{-1})^{d_{W'}/2}  d\nu(t)  .
				\]
				hence the first statement.

				If $\mathcal{K}$ is completely monotone, then so is $\mathcal{K}^2$ (See Section \ref{sect:completelymontonekernel}). The same argument then gives
				\[
				\begin{split}
					&\E[ \mathcal{K}^2(|X-X'|_{\Sigma}^2   ] \leq  \frac{p_0}{  \sqrt{\det \Sigma_{V/W}} } \int_0^\infty (\pi t^{-1})^{d_{W'}/2}  d\tilde{\nu}(t)
					\\
					= &  \frac{p_0}{  \sqrt{\det \Sigma_{V/W}} } \iint (\pi (t+s)^{-1})^{d_{W'}/2}  d\nu(t) d\nu(s).
				\end{split}
				\]
				as required.
			\end{proof}

			The upshot is that under quantitative bounds on the probability density of the marginal distribution on $V/W$, then the operator norm of $T$ is small when $\det \Sigma_{V/W}$ is very large. Notice also that while $V$ may be very high dimensional, $W'$ may be low dimensional, and to apply the result we only need to know the marginal distribution.

			\begin{eg}
				(Gaussian example) Suppose $X$ is a Gaussian variable with covariance matrix $Cov(X)= (C^{ij})$ in the coordinates on $V$, so the inverse covariance matrix is $Cov(X)^{-1}=(C_{ij})$, and 
				$X-X'$ has covariance matrix $2Cov(X)$. Suppose $K(x)=\mathcal{K}(|x|_V^2)$ is the completely monotone kernel (\ref{eqn:completelymonotonekernel}), then
				\[
				\begin{split}
					& \E[ \mathcal{K}(|X-X'|_{\Sigma}^2   ]
					= \int_V   \frac{ 1 }{ (4\pi)^{d/2}\sqrt{ \det( C^{ij}) }  } \exp( -\frac{1}{4}  C_{ij} x^ix^j  ) \mathcal{K}(|x|_{\Sigma}^2) dx
					\\
					=&  \int_0^\infty\int_V   \frac{ 1 }{ (4\pi)^{d/2}\sqrt{ \det( C^{ij}) }  } \exp( -\frac{1}{4}  C_{ij} x^ix^j - t|x|_{\Sigma}^2) dx d\nu(t)
					\\
					=& \int_0^\infty   \frac{ 1 }{ \sqrt{ \det( C^{ij}) } \sqrt{ 
							\det(  C_{ij}+ 4t\Sigma_{ij})   } }  d\nu(t)
					\\
					= &  \int_0^\infty   \frac{ 1 }{  \sqrt{ 
							\det(  \delta_{ij}+ 4 tC^{ik}\Sigma_{kj})   } }  d\nu(t).
				\end{split}
				\]
				This is small when $\Sigma$ is large compared to $Cov^{-1}(X)$ in at least one direction.

			\end{eg}
			
			\begin{eg}\label{rmk:manyatoms}
				Suppose $X$ is a discrete variable consisting of a large number of atoms at $a_i\in V$, each with very small probability $p_i$, and suppose all the atoms are far separated with respect to the inner product $\Sigma$. Then 
				\[
				\E[ \mathcal{K}(|X-X'|_{\Sigma}^2 )  ]=\sum_{i,j} p_ip_j \mathcal{K}(|a_i-a_j|_{\Sigma}^2), 
				\]
				which will be a small quantity of an order comparable to $O(\sum p_i^2)$. The intuition is that a very dispersed discrete variable behaves like a continuous variable.
				
			\end{eg}

			\subsection{Cluster operators}

			We let $\chi_i= \1_{X_0=i}$, and define  a self-adjoint operator $T_i$ by
			\begin{equation*}
				\E[ \chi_i f(UX) \bar{g}(UX)]= (T_if, g)_{\mathcal{H}},\quad \forall g\in \mathcal{H},
			\end{equation*}
			This induces the decomposition $T=\sum T_i$, with each $T_i$ describing one `cluster'. As in Section \ref{sect:operator}, the explicit formula for $T_i$ is 
			\begin{equation}\label{eqn:operatorTi}
				T_if(x)=  \E[ \chi_i   f(UX) \overline{K(x, UX)}  ].
			\end{equation}
			Similarly, we define $\mathcal{Y}_i\in \mathcal{H}$ by 
			\[
			\E[ \chi_i Y \bar{g}(UX)  ]= (\mathcal{Y}_i, g)_{\mathcal{H}},\quad \forall g\in \mathcal{H},
			\]
			so $\mathcal{Y}_i= \E[\chi_i  Y\overline{K( \cdot, UX)}   ]$ gives a decomposition 
			$\mathcal{Y}=\sum_i \mathcal{Y}_i$. The following Proposition can be proved by the same arguments as Proposition  \ref{prop:operatorT}.

			\begin{prop}\label{prop:operatorT2}
				Let $X'$ denote an independent copy of $X$, and let $\chi_i=\1_{X_0=i}$ and $\chi_i'= \1_{X_0'=i}$. Then 
				\[
				\norm{T_i f}_{\mathcal{H}}^2\leq   \norm{f}_{C^0}^2 \E[ \chi_i \chi_i' |K(U(X-X'))|  ] \leq \norm{f}_{\mathcal{H}}^2 \E[ \chi_i \chi_i' |K(U(X-X'))|  ] ,
				\]
				\[
				\norm{T_i f}_{\mathcal{H}}^2\leq \E[\chi_i|f(UX)|^2 ] \E[ \chi_i\chi_i' |K(U(X-X'))|^2  ]^{1/2},
				\]
				The Hilbert-Schmidt norm of $T_i$ acting on $\mathcal{H}$ is 
				\[
				\norm{T_i}_{HS}^2= \E[ \chi_i \chi_i'|K(U(X-X')|^2  ].
				\]
				Moreover the operator $T_iT_j$ satisfies the bound
				\begin{equation*}\label{eqn:TiTj}
					\norm{T_i T_j f}_{\mathcal{H}} \leq \norm{f}_{C^0} \E[ \chi_i \chi_j' |K(U(X-X'))|  ]
				\end{equation*}
				\[
				\norm{T_i T_j f}_{\mathcal{H}}^2 \leq \E[\chi_j|f(UX)|^2  ] \E[ \chi_i \chi_j' |K(U(X-X'))|^2  ]. 
				\]
			\end{prop}

			The estimates on the operator norms of $T$ apply also to $T_i$ with cosmetic changes. Moreover, there are two basic mechanisms for $\norm{T_iT_j}$ to be small: the clusters may be far separated, or have very different scale parameters.

			\begin{lem}
				Suppose $K(x)=\mathcal{K}(|x|_V^2)$ is a completely monotone kernel function, and suppose
				\[
				\text{dist}_\Sigma( \text{supp}(\chi_i X), \text{supp}(\chi_j X) ) \geq a.
				\]
				Then
				\[
				\E[\chi_i \chi_j' |\mathcal{K}(|X-X'|_\Sigma^2) |   ] \leq  \mathcal{K}(a^2) \E[ \chi_i] \E[\chi_j].  
				\]
			\end{lem}

			In particular, for the Sobolev kernel or the Gaussian kernel, the function $\mathcal{K}$ decays exponentially, so the operator norm of $T_i T_j$ is exponentially small. The interpretation is that \emph{spatially far separated clusters have very weak interactions}.

			\begin{lem}\label{lem:Kdoubleintegralbound2}
				Let $W\subset V$ be a proper subspace, and the probability density of the marginal distribution of $\chi_i X$ on $V/W$ has $L^\infty$-bound by $p_0$ with respect to the Lebesgue measure on $W'\simeq V/W$. Suppose $K(x)=\mathcal{K}(|x|_V^2)$ is the completely monotone kernel (\ref{eqn:completelymonotonekernel}). The positive semi-definite form $\Sigma$ induces the quotient inner product $\Sigma_{V/W}$ on $V/W$. 
				Then
				\[
				\E[ \chi_i \chi_j' \mathcal{K}(|X-X'|_{\Sigma}^2 )  ] \leq  \frac{p_0}{  \sqrt{\det \Sigma_{V/W}} } \int_0^\infty (\pi t^{-1})^{d_{W'}/2}  d\nu(t)  .
				\]

			\end{lem}

			This situation allows the support of $\chi_i X$ and $\chi_j X$ to have overlaps, but when the marginal distribution of $\chi_i X$ is very shallow and spread out, while $\chi_j X$ is very concentrated, then $\frac{p_0}{  \sqrt{\det \Sigma_{V/W}} }$ would be small. The interpretation is that \emph{scale separation between two clusters can lead to weakness of interactions}.

			\subsection{Non-interaction between far separated clusters}\label{sect:noninteraction}

			We now construct an approximate solution to the Euler-Lagrange equation (\ref{eqn:EulerLag2}). Each cluster leads to a decoupled Euler-Lagrange equation
			\[
			(\lambda I+  T_i) f_i= \mathcal{Y}_i,
			\]
			whose solution $f_i$ is the minimizer of the functional
			\[
			I_i(f,U,\lambda)=	\frac{1}{2} \E[ \chi_i |Y- f(UX)|^2   ] +\frac{\lambda}{2} \norm{f}_\mathcal{H}^2.
			\]
			We set $\tilde{f}_U= \sum_1^m f_i$. The intuition is that $\tilde{f}_U$ is a good approximation of $f_U$ when the cluster interaction is weak; this gives a more quantitative perspective on Theorem \ref{thm:asymptoticapproximateminimizer}.

			\begin{thm}\label{thm:noninteraction1}
				Let $\mathcal{Z}=\mathcal{Y}- (\lambda I+ T) \tilde{f}_U$. Then there exists some random variable $Z$ with 
				\begin{equation}\label{eqn:HahnBanach}
					\begin{cases}
						(\mathcal{Z}, g)_{\mathcal{H}}= \E[Z \overline{g(UX)}], \quad \forall g\in \mathcal{H},
						\\
						\E[|Z|^2] \leq \frac{1}{\lambda^2}  \E[ |Y|^2] (\sum_{1\leq i\neq j\leq m}\E[  \chi_i \chi_j'|K( UX,UX')|^2  ]).  
					\end{cases}
				\end{equation}
				Moreover, the deviation between $f_U$ and $\tilde{f}_U$ satisfies
				\begin{equation}\label{eqn:trivialupperboundZ}
					\lambda \norm{f_U- \tilde{f}_U}_{\mathcal{H}}^2+ \E[ |Z-(f_U-\tilde{f}_U)(UX)|^2 ] \leq  \E[|Z|^2]. 
				\end{equation}
			\end{thm}
			
			\begin{proof}
				We compute
				\[
				(\lambda I+ T) \tilde{f}_U= \sum_j (\lambda I+ \sum_i T_i) f_j = \sum_j \mathcal{Y}_j+  \sum_{1\leq i\neq j\leq m} T_i f_j=\mathcal{Y}+\frac{1}{\lambda} \sum_{1\leq i\neq j\leq m} T_i ( \mathcal{Y}_j- T_j f_j  ). 
				\]
				Hence
				\[
				\mathcal{Z}=
				- \frac{1}{\lambda} \sum_{1\leq i\neq j\leq m} T_i ( \mathcal{Y}_j- T_j f_j  ) .
				\]
				For any $g\in \mathcal{H}$, 
				using the definition of $T_i$ and $\mathcal{Y}_i$, we compute
				\[
				\begin{split}
					& (\mathcal{Z}, g)_{\mathcal{H}}= - \frac{1}{\lambda} (\sum_{1\leq i\neq j\leq m} T_i ( \mathcal{Y}_j- T_j f_j  ), g)_{\mathcal{H}} 
					\\
					=& - \frac{1}{\lambda} \sum_{1\leq i\neq j\leq m}\E[ \chi_i (\mathcal{Y}_j- T_j f_j  )(UX)  \overline{g(UX)}  ]
					\\
					= & -\frac{1}{\lambda} \sum_{1\leq i\neq j\leq m}\E[ \chi_i \chi_j' (Y'- f_j(UX'))\overline{K( UX,UX')} 
					\overline{g(UX)}  ]
					\\
					=&  \E[ Z 
					\overline{g(UX)}  ]
				\end{split}
				\]
				where the random variable $Z$ is obtained by taking expectation over the $(X',Y')$ variables, but not the $(X,Y)$ variables:
				\[
				Z= -\frac{1}{\lambda} \E_{X',Y'}[\sum_{1\leq i\neq j\leq m} \chi_i \chi_j' (Y'- f_j(UX'))\overline{K( UX,UX')}] .
				\]
				By Cauchy-Schwarz, 
				\[
				\begin{split}
					\E[|Z|^2]\leq & \frac{1}{\lambda^2} (\sum_{j} \E[ \chi_j |Y- f_j(UX)|^2] \sum_{1\leq i\neq j\leq m}\E[  \chi_i \chi_j'|K( UX,UX')|^2  ]
					\\
					\leq &  \frac{1}{\lambda^2} (\sum_{j} \E[ \chi_j |Y|^2] \sum_{1\leq i\neq j\leq m}\E[  \chi_i \chi_j'|K( UX,UX')|^2  ]
					\\
					= &  \frac{1}{\lambda^2}  \E[ |Y|^2] \sum_{1\leq i\neq j\leq m}\E[  \chi_i \chi_j'|K( UX,UX')|^2  ] 
				\end{split}
				\]
				Here the second line uses the trivial upper bound Lemma \ref{lem:trivialupperbound}.

				Now $f_U-\tilde{f}_U$ satisfies the Euler-Lagrange equation
				\begin{equation}\label{eqn:EulerLagrangeZ}
					(\lambda I+T) (f_U- \tilde{f}_U)= \mathcal{Z}.
				\end{equation}
				This has the same form as the Euler-Lagrange equation (\ref{eqn:EulerLag2}), with $f_U$ replaced by $f_U-\tilde{f}_U$, and $Y$ replaced by $Z$. Thus the trivial upper bound Lemma \ref{lem:trivialupperbound} applied to $Z$ instead of $Y$ implies (\ref{eqn:trivialupperboundZ}). 
			\end{proof}

			The point is that if the interaction between clusters satisfies 
			\[
			\frac{1}{\lambda^2}  (\sum_{1\leq i\neq j\leq m}\E[  \chi_i \chi_j'|K( UX,UX')|^2  ]) < 1,
			\]
			then $\E[ |Z|^2]$ is smaller than $\E[|Y|^2]$ by a factor less than one. We can replace $Y$ by $Z$, and $f_U$ by $f_U-\tilde{f}_U$, and the resulting equation (\ref{eqn:EulerLagrangeZ}) has the same form as (\ref{eqn:EulerLag2}). Banach iteration of this process then yield the exact solution to (\ref{eqn:EulerLag2}).

			\begin{rmk}
				In applications, it may happen that not all the clusters are far separated. We have the flexibility to 
				group together several cases of $X_0$ to form bigger clusters, and apply the Theorem to study the interaction of these bigger clusters. 
			\end{rmk}

			\begin{cor}\label{cor:noninteraction2}
				The deviation between $f_U$ and $\tilde{f}_U$ 
				satisfies the bound
				\[
				\lambda \norm{f_U- \tilde{f}_U}_{\mathcal{H}}^2+ \E[ |(f_U-\tilde{f}_U)(UX)|^2 ]  \leq \frac{4}{\lambda^2}  \E[ |Y|^2] \sum_{1\leq i\neq j\leq m}\E[  \chi_i \chi_j'|K( UX,UX')|^2  ] .
				\]

			\end{cor}

			\begin{proof}
				By (\ref{eqn:trivialupperboundZ}), 
				\[
				\lambda \norm{f_U- \tilde{f}_U}_{\mathcal{H}}^2+ \E[ |(f_U-\tilde{f}_U)(UX)|^2 ] \leq   2\text{Re} \E[\bar{Z}  (f_U-\tilde{f}_U)(UX) ]. 
				\]
				By Cauchy-Schwarz,
				\[
				\lambda \norm{f_U- \tilde{f}_U}_{\mathcal{H}}^2+ \E[ |(f_U-\tilde{f}_U)(UX)|^2 ] \leq   2 \E[|(f_U-\tilde{f}_U)(UX) |^2]^{1/2} \E[|Z|^2]^{1/2},
				\]
				so by rearranging, 
				\[
				\begin{split}
					&	\lambda \norm{f_U- \tilde{f}_U}_{\mathcal{H}}^2+ \E[ |(f_U-\tilde{f}_U)(UX)|^2 ] \leq 4 \E[|Z|^2] 
					\\
					\leq &\frac{4}{\lambda^2}  \E[ |Y|^2] \sum_{1\leq i\neq j\leq m}\E[  \chi_i \chi_j'|K( UX,UX')|^2  ] 
				\end{split}
				\]
				where the last inequality uses (\ref{eqn:HahnBanach}).
			\end{proof}

			\begin{thm}(Minimum value)\label{thm:noninteraction3}
				Let $J_i= I_i(f_i,U,\lambda)$ be the minimum of the decoupled minimization problems. Then
				\begin{equation}
					\begin{cases}
						J(U,\lambda)- \sum_1^m J_i \leq \frac{1}{\lambda^2}   \E[|Y|^2] \sum_{1\leq i\neq j\leq m} \E[  \chi_i'\chi_j |K(UX, UX')|^2   ],
						\\
						J(U,\lambda)- \sum_1^m J_i \geq  -\frac{2}{\lambda^2} \E[|Y|^2] \sum_{1\leq i\neq j\leq m} \E[  \chi_i'\chi_j |K(UX, UX')|^2   ].
					\end{cases}	
				\end{equation}
			\end{thm}

			\begin{proof}
				For an upper bound on $J(U,\lambda)$, we plug in $\tilde{f}_U$ as a test function, so 
				\[
				\begin{split}
					&	J(U,\lambda) \leq I(\tilde{f}_U, U,\lambda) = \sum_1^m   \frac{1}{2} \E[\chi_j |Y- \tilde{f}_U(UX)|^2] + \frac{\lambda}{2} \norm{  \tilde{f}_U}_{\mathcal{H}}^2.
					\\
					=& \sum_j   \frac{1}{2} \E[\chi_j |Y- f_j(UX)-  \sum_{j\neq i} f_i(UX) | ^2] + \frac{\lambda}{2} \norm{  \sum f_j }_{\mathcal{H}}^2.
				\end{split}
				\]
				By the Euler-Lagrange equation for the minimization $I_j$, 
				\[
				\E[  \chi_j (Y- f_j(UX))   \overline{ (\sum_{i\neq j} f_i(UX))  }  ]= \lambda ( f_j, \sum_{i\neq j} f_i )_\mathcal{H},
				\]
				so after cancelling many cross terms, the  above RHS simplifies to
				\[
				\begin{split}
					& \sum_j   \frac{1}{2} \E[\chi_j |Y- f_j(UX) | ^2] + \frac{\lambda}{2} \sum\norm{   f_j }_{\mathcal{H}}^2 + \frac{1}{2} \sum_j    \E[\chi_j |  \sum_{j\neq i} f_i(UX) | ^2]
					\\
					=& \sum J_j+ \frac{1}{2}  \sum_j   \E[\chi_j |  \sum_{j\neq i} f_i(UX) | ^2].
				\end{split}
				\]
				Hence
				\begin{equation}\label{eqn:noninteractiontestfunction}
					J(U,\lambda) \leq  I(\tilde{f}_U, U,\lambda) = \sum J_j+\frac{1}{2}  \sum_j    \E[\chi_j |  \sum_{j\neq i} f_i(UX) | ^2].
				\end{equation}

				By the integral equation on $f_i$ (\cf (\ref{eqn:EulerLagrange1})),
				\[
				f_i(x) = \frac{1}{\lambda}  \E[    \chi_i'   (Y'- f_i(UX'))   \overline{K(x,UX') }]  ,
				\]
				so by Cauchy-Schwarz, for each fixed $j$, 
				\[
				\begin{split}
					&	|\sum_{i\neq j} f_i(x) |^2 \leq  \frac{1}{\lambda^2}   \sum_{i\neq j}\E[  \chi_i' |Y'- f_i(UX')|^2    ] \sum_{i\neq j} \E_{X'}[  \chi_i' |K(x, UX')|^2   ]
					\\
					\leq &  \frac{1}{\lambda^2}  \sum_{i\neq j}\E[  \chi_i' |Y'|^2    ] \sum_{i\neq j} \E_{X'}[  \chi_i' |K(x, UX')|^2   ]
					\\
					\leq & \frac{1}{\lambda^2}  \E[  |Y|^2    ] \sum_{i\neq j} \E_{X'}[  \chi_i' |K(x, UX')|^2   ]
				\end{split}
				\]
				where the second line uses the trivial upper bound for the minimization problem defining $f_i$. We substitute $x=UX$ into this pointwise estimate, multiply by $\chi_j$, and take expectation, to deduce
				\[
				\sum_j \E[\chi_j |  \sum_{j\neq i} f_i(UX) | ^2   ] \leq \frac{1}{\lambda^2}  \E[|Y|^2] \sum_{1\leq i\neq j\leq m} \E[  \chi_i'\chi_j |K(UX, UX')|^2   ].
				\]
				Combining this with (\ref{eqn:noninteractiontestfunction}),
				\[
				J(U,\lambda)\leq I(\tilde{f}_U,U,\lambda) \leq \sum_j J_j +\frac{1}{\lambda^2}   \E[|Y|^2] \sum_{1\leq i\neq j\leq m} \E[  \chi_i'\chi_j |K(UX, UX')|^2   ].
				\]

				Now we derive the lower bound on $J(U,\lambda)$. Since $f_U$ is the minimizer of the quadratic functional $I(f, U,\lambda)$, 
				\[
				I(\tilde{f}_U ,  U, \lambda)= J(U,\lambda)+ \frac{1}{2} \E[  | (f_U- \tilde{f}_U)(UX)|^2 ]+ \frac{\lambda}{2} \norm{  f_U- \tilde{f}_U }_\mathcal{H}^2.
				\]
				We observe that $ I(\tilde{f}_U ,  U, \lambda)\geq \sum J_j$ by 
				(\ref{eqn:noninteractiontestfunction}), 
				so 
				\[
				\begin{split}
					& J(U,\lambda)- \sum J_j \geq -  \frac{1}{2} \E[  | (f_U- \tilde{f}_U)(UX)|^2 ]-   \frac{\lambda}{2} \norm{  f_U- \tilde{f}_U }_\mathcal{H}^2
					\\
					\geq & -   \frac{2}{\lambda^2}  \E[ |Y|^2] \sum_{1\leq i\neq j\leq m}\E[  \chi_i \chi_j'|K( UX,UX')|^2  ] .
				\end{split}
				\]
				where the second line uses Corollary \ref{cor:noninteraction2}.
			\end{proof}

			\subsection{Dimensional reduction mechanism}\label{sect:dimreductionmechanism}

			The moral of this Section is that if some component of $X$ concentrates near its conditional expectation value with respect to the inner product $\Sigma$, then the minimizer $f_U$ is well approximated by the minimizer for the kernel ridge regression problem associated to a simpler probability distribution of $(X,Y)$, depending on fewer degrees of freedom.

			\subsubsection{Changing probability distributions}

			

		We now consider the effect of changing the probability distribution. Let $(\tilde{X}, \tilde{Y})$ be another random pair, \red{defined on a common probability space as $(X,Y)$,} so we can define $\tilde{\mathcal{Y}}\in \mathcal{H}$ and self-adjoint operator $\tilde{T}$ analogous to $\mathcal{Y}$ and $T$. Let \red{$(X', Y', \tilde{X}', \tilde{Y}')$} be an independent copy of \red{$(X,Y, \tilde{X}, \tilde{Y})$}.

		\begin{lem}\label{lem:HSnormTTtilde}
			The Hilbert-Schmidt norm
			\[
			\norm{T-\tilde{T}}_{HS}^2 =	\E[ |K(UX, UX')|^2- |K(U\tilde{X}, UX')|^2   - |K(UX, U\tilde{X'})|^2+|K(U\tilde{X}, U\tilde{X'})|^2 ] .
			\]
		\end{lem}

		\begin{proof}
			By Proposition  \ref{prop:operatorT},
			\[
			\norm{T}_{HS}^2= \E[ |K(U(X-X'))|^2  ]= \int_{V\times V} |K(U(x-x'))|^2 d\mu_X d\mu_{X'},
			\]
			where $d\mu_X$ is the probability measure of $X$ in $V$. 
			The Hilbert-Schmidt norm square is an inner product on the space of operators, hence is bilinear in $d\mu_X$ and $d\mu_{X'}$. Thus  
			\[
			\begin{split}
				& \norm{T-\tilde{T}}_{HS}^2= \int_{V\times V} |K(U(x-x'))|^2 (d\mu_X- d\mu_{\tilde{X}}) (d\mu_{X'}- d\mu_{\tilde{X}'})
				\\
				=& 
				\E[ |K(UX, UX')|^2- |K(U\tilde{X}, UX')|^2   - |K(UX, U\tilde{X'})|^2+|K(U\tilde{X}, U\tilde{X'})|^2 ]
				.
			\end{split}
			\]
			as required.
		\end{proof}


		\begin{lem}\label{lem:dimreductionYequalcase}
			Suppose \red{$Y=\tilde{Y}$}. Then
			\[
			\norm{\mathcal{Y}- \red{\tilde{\mathcal{Y} } }}_{\mathcal{H}}^2 = \E[ Y\overline{Y'}\{ 
			K(UX, UX')- K(U\tilde{X}, UX')   - K(UX, U\tilde{X}')+K(U\tilde{X}, U\tilde{X}')
			\}]
			\]
			which is bounded by
			\[
			\E[|Y|^2]	\E[ |K(UX, UX')- K(U\tilde{X}, UX')   - K(UX, U\tilde{X'})+K(U\tilde{X}, U\tilde{X'}) |^2 ]^{1/2}.
			\]
		\end{lem}

		\begin{proof}
			We apply Lemma  \ref{lem:normY} and expand the inner product on $\mathcal{H}$. 
		\end{proof}

		\subsubsection{Decomposition of $Y$}

		For each value $i=1,2,\ldots m$ of $X_0$, suppose we are given some direct sum decomposition $V=W_i\oplus W_i'$. We let $Y_i^0=\E[\chi_i Y| X_{W'}]$, and decompose $Y$ by
		\[
		\chi_i Y= Y_i^0+ Y_i^1,\quad  \tilde{Y}= Y^0= \sum Y_i^0,\quad  Y^1= \sum_1^m Y_i^1,\quad Y=Y^0+ Y^1. 
		\]
		We can then take $\tilde{X}$ to be a random variable valued in $V$, such that conditional on $X_0=i$, 
		\[
		\tilde{X}= ( \E[ X_{W_i}| X_{W_i}'],  X_{W_i'}) \in W_i\oplus W_i'= V.
		\]
		We view $\tilde{X}$ as a kind of dimensional reduction of $X$.

		Then the minimizer $f_U$ can be decomposed into $f_U=f_U^0+ f_U^1$, 
		\[
		\begin{cases}
			(\lambda I+ T) f_U^i= \mathcal{Y}^i,
			\\
			\mathcal{Y}^i= \E[Y^i \overline{K(\cdot, UX)}] , \quad i=0,1.
		\end{cases}
		\]
		We wish to find conditions for the smallness of $f_U^1$, so that we can replace $Y$ by $Y^0$, and approximate $f_U$ by $f_U^0$.

		\begin{lem}\label{lem:K4terms}
			Suppose  $K(x)= \mathcal{K}(|x|_V^2)$, and $\Sigma=U^T U$.  Then
			\[
			\begin{split}
				& \lambda\norm{f_U^1}_{\mathcal{H} }^2 + \E[ |f_U^1(UX)|^2]  \leq \frac{1}{\lambda} \E[|Y^1|^2] \times
				\\
				& \E[ | \mathcal{K}(|X-X'|_\Sigma^2)-  \mathcal{K}(|\tilde{X}-X'|_\Sigma^2)
				- \mathcal{K}(|X-\tilde{X'}|_\Sigma^2)
				+ \mathcal{K}(|\tilde{X}-\tilde{X'}|_\Sigma^2)|^2   ]^{1/2}.
			\end{split}
			\]
		\end{lem}

		\begin{proof}
			Applying Corollary \ref{Cor:quantitativenonfitting}, 
			\[
			\lambda\norm{f_U^1}_{\mathcal{H} }^2 + \E[ |f_U^1(UX)|^2]  \leq  \frac{1}{\lambda}\E[ Y^1  \bar{Y}^{1' }K( U(X-X'))  ].
			\]
			Conditional on each value of $X_0$,
			by construction $\E[ Y^1|X_{W'_i}]=0$, so the term $\E[ Y^1\bar{Y}^{1'} K( U(X-X'))  ]$ is equal to
			\[
			\begin{split}
				& \E[ Y^1\bar{Y}^{1' }\{ K( U(X-X')) - K(U(\tilde{X}- X')) - K(U(X- \tilde{X}')) + K(U(\tilde{X}-\tilde{X}')   \} ]
				\\
				\leq &
				\E[|Y^1|^2] \E[ | K( U(X'-X)) - K(U(X'- \tilde{X})) - K(U(X'- \tilde{X})) + K(U(\tilde{X'}-\tilde{X})|^2   ]^{1/2} 
				.
			\end{split}
			\]
			Hence the result.
		\end{proof}

		The point here is that if for each value of $X_0$, the random variable $X_{W_i}$ concentrates near its conditional average, then
		\[
		\E[ | \mathcal{K}(|X-X'|_\Sigma^2)-  \mathcal{K}(|\tilde{X}-X'|_\Sigma^2)
		- \mathcal{K}(|X-\tilde{X'}|_\Sigma^2)
		+ \mathcal{K}(|\tilde{X}-\tilde{X'}|_\Sigma^2)|^2   ]
		\]
		would be small, 
		whence $\lambda\norm{f_U^1}_{\mathcal{H} }^2 + \E[ |f_U^1(UX)|^2] $ is small.

		\subsubsection{Simplifying the distribution of $X$}\label{sect:simplifydistribution}

		Next we estimate the difference between $f_U^0$ versus the solution $\tilde{f}_U$ to the equation
		\[
		(\lambda I+ \tilde{T}) \tilde{f}_U= \tilde{\mathcal{Y}}, \quad \tilde{\mathcal{Y}}= \E[Y^0 \overline{K(\cdot, U\tilde{X})}],
		\]
		where $\tilde{T}$ is the operator associated to the distribution of $\tilde{X}$, namely
		\[
		\tilde{T}f(x)=\E[   f(U\tilde{X}) \overline{K(x,U\tilde{X})}   ].
		\]
		We compare $\tilde{f}_U$ and $f_U^0$ to the intermediate object $f_U^2$, defined by
		\[
		(\lambda I+ T) f_U^2= \tilde{\mathcal{Y}}. 
		\]

		\begin{itemize}
			\item  Step 1: Estimate the difference between $f_U^2$ and $f_U^0$.
		\end{itemize}

		We notice
		\[
		(\lambda I + T) (f_U^2- f_U^0) = \tilde{\mathcal{Y}}- \mathcal{Y}^0. 
		\]
		Applying Lemma \ref{lem:dimreductionYequalcase} and Corollary
		\ref{Cor:quantitativenonfitting},
		\[
		\begin{split}
			&\lambda \norm{f_U^2- f_U^0 }_{\mathcal{H}}^2 + \E[ |(f_U^2-f_U^0)(UX)|^2  ] \leq \frac{1}{\lambda} \norm{\tilde{\mathcal{Y}}- \mathcal{Y}^0}_{\mathcal{H}}^2 
			\\
			\leq & 
			\frac{1}{\lambda} 
			\E[|Y^0|^2] \E[ |\mathcal{K}(|X-X'|^2_\Sigma)
			-
			\mathcal{K}(|\tilde{X}- X'|_\Sigma^2)  
			-
			\mathcal{K}(|X- \tilde{X}'|_\Sigma^2) 
			+
			\mathcal{K}(|\tilde{X}- \tilde{X}'|_\Sigma^2)|^2   ]^{1/2}.
		\end{split}
		\]

		\begin{itemize}
			\item Step 2: Estimate the difference between $\tilde{f}_U$ and $f_U^2$.
		\end{itemize}
		
		This difference comes from changing $T$ to $\tilde{T}$. 
		\red{We have
			\[
			(\lambda I+ T)(\tilde{f}_U- f_U^2)=- (\tilde{T}-T) \tilde{f}_U,
			\]
			and the RHS can be estimated by
			\[
			\begin{split}
				& \norm{(\tilde{T}-T) \tilde{f}_U}_{\mathcal{H}}^2 \leq \norm{T- \tilde{T}}_{HS}^2 \norm{\tilde{f}_U}_\mathcal{H}^2 \leq\norm{T- \tilde{T}}_{HS}^2 \lambda^{-2} \norm{ \mathcal{Y}}_{\mathcal{H}}^2
				\\
				\leq &  \lambda^{-2} \norm{ \mathcal{Y}}_{\mathcal{H}}^2\E[ |\mathcal{K}^2(|X-X'|^2_\Sigma)
				-
				\mathcal{K}^2(|\tilde{X}- X'|_\Sigma^2)  
				-
				\mathcal{K}^2(|X- \tilde{X}'|_\Sigma^2) 
				+
				\mathcal{K}^2(|\tilde{X}- \tilde{X}'|_\Sigma^2)|   ]
				\\
				\leq &
				\lambda^{-3} \E[|Y^0|^2]\E[ |\mathcal{K}^2(|X-X'|^2_\Sigma)
				-
				\mathcal{K}^2(|\tilde{X}- X'|_\Sigma^2)  
				-
				\mathcal{K}^2(|X- \tilde{X}'|_\Sigma^2) 
				+
				\mathcal{K}^2(|\tilde{X}- \tilde{X}'|_\Sigma^2)|   ]
			\end{split}\]
			Here the second line applies Lemma \ref{lem:HSnormTTtilde} to bound the Hilbert-Schmidt norm (hence the operator norm) of $\tilde{T}-T$, while the last line applies Lemma \ref{lem:normY}.
			By Lemma \ref{Cor:quantitativenonfitting}, we have
			\[
			\begin{split}
				&	\lambda \norm{\tilde{f}_U- f_U^2 }_\mathcal{H}^2 + \E[  |( \tilde{f}_U- f_U^2   ) (UX)|^2  ]
				\\
				\leq &  \lambda^{-1}  \norm{  (T-\tilde{T})\tilde{f}_U }_{\mathcal{H}}^2 
				\\
				\leq &   \lambda^{-4} \E[|Y^0|^2]\E[ |\mathcal{K}^2(|X-X'|^2_\Sigma)
				-
				\mathcal{K}^2(|\tilde{X}- X'|_\Sigma^2)  
				-
				\mathcal{K}^2(|X- \tilde{X}'|_\Sigma^2) 
				+
				\mathcal{K}^2(|\tilde{X}- \tilde{X}'|_\Sigma^2)|   ].
			\end{split}
			\]
		}

		\subsection{Discussions}

		We now comment on the philosophical meaning of the setup of Section \ref{sect:clusterinteractions}.

		\begin{itemize}
			\item The \emph{discrete} random variable $X_0$ gives a \emph{qualitative classification} of the learning task into a finite number of cases. Distinct values of $X_0$ signify what we would interpret as different \emph{conceptual subcategories}, which is often not directly accessible to empirical observers.

			\item Instead, only the \emph{clustering} in the distribution of $X$ could in principle be observed empirically, and this is how empirical observers may infer the existence of $X_0$ in the model.

			\item It is intuitively clear that if an observer may gain access to $X_0$ (prior knowledge of qualitative classification), it would lead to more efficient learning. The motivation of kernel learning is that the information of $X_0$ is as good as
			the scale parameters and the essential degree of freedom in the learning problem, and this information is hopefully reflected in the distribution of the vacua.

		\end{itemize}

		\section{Landscape of vacua: scale detectors}\label{sect:Landscape2}

		\subsection{What is a significant vacuum?}\label{sect:scaledetector}

		As discussed in Section \ref{sect:Motivation}, the main motivation for the kernel learning problem is to find the best parameters $U$ or $\Sigma=U^T U$, which detect the inherent scale parameters, and select out the essential degrees of freedom in $X$, in order for the kernel ridge regression minimizer $F_\Sigma$ to efficiently predict $Y$. This lead us to study the vacua of $\mathcal{J}(\Sigma;\lambda)$, namely the local minimizers of $\mathcal{J} $ in the partial compactification $\overline{Sym^2_{\geq 0}}$.

		However, not necessarily all the vacua are equally important for the task; for instance, if $\mathcal{J}(\Sigma;\lambda)$ is close to the trivial upper bound $\frac{1}{2}\E[|Y|^2]$, then the minimizer has small norm $\E[|F_\Sigma|^2]\ll 1$, so $\Sigma$ does not carry meaningful information of $Y$. We wish to formulate a notion of significant vacuum, which intuitively excludes such `shallow local minima'.

		\subsubsection{Scale detector}
		
		We will now formalize the idea that the most interesting vacua are those which identify the scale parameters in the learning problem; in general dimensions, scale parameters are specified by inner products rather than numbers. We begin with the interior of $Sym^2_+$. 

		\begin{Def}
			Given $\Sigma\in Sym^2_+ $, and a parameter $\Lambda>1$, the \emph{$\Lambda$-neighbourhood} of $\Sigma$ denotes the subset of inner products which are uniformly equivalent to $\Sigma$ with constant $\Lambda$,
			\[
			\mathcal{N}(\Sigma, \Lambda):=\{  \Sigma'\in Sym^2_+: \Lambda^{-1} \Sigma \leq \Sigma'\leq  \Lambda \Sigma \}. 
			\]
			Here $\leq $ is meant in the sense of the partial order on the inner products. 
		\end{Def}
		
		\blue{Informally, a scale detector $\Sigma$ is a vacuum that identifies the intrinsic scales of the data up to a multiplicative factor $\Lambda$, with a quantitative energy gap $\epsilon \E|Y|^2$ separating it from competing scale configurations.}

		\begin{Def}\label{Def:scaledetector}
			Let $\epsilon>0$ and $\Lambda>1$ be given parameters.
			A vacuum $\Sigma\in Sym^2_+$ is called an $(\epsilon,\Lambda)$-\emph{scale detector}, if it has the local minimizing property
			\[
			\mathcal{J}(\Sigma', \lambda)\geq \mathcal{J}(\Sigma, \lambda),\quad \forall \Sigma' \in \mathcal{N}(\Sigma, \Lambda), 
			\]
			and the \emph{strict gap} property
			\[
			\mathcal{J}(\Sigma', \lambda)\geq \mathcal{J}(\Sigma, \lambda)+\epsilon \E|Y|^2 ,\quad \forall \Sigma' \in \partial \mathcal{N}(\Sigma, \Lambda).
			\]

		\end{Def}

		\begin{rmk}
			One should think that $\epsilon$ is moderately small (say around $10^{-1}$), and $\Lambda$ is moderately large (say around $5$). 
			The motivation is that the eigenvalues of $\Sigma$ detects the correct scale parameters in the problem, and what really matters is not the precise value of $\Sigma$, but its uniform equivalence neighbourhood. We point out the obvious fact that for scale detectors
			\[
			\mathcal{J}(\Sigma;\lambda)\leq \frac{1}{2}\E[|Y|^2]-\epsilon \E[|Y|^2],
			\]
			which means $\mathcal{J}(\Sigma;\lambda)$ is bounded away from the trivial upper bound $\frac{1}{2}\E[|Y|^2]$, and therefore carries some information about $Y$. 
		\end{rmk}

		\begin{rmk}
			As another motivation, the scale detector is \emph{stable under small perturbation}, in the sense that due to the strict gap property, the $\mathcal{J}$-minimizers within 
			$\mathcal{N}(\Sigma, \Lambda)$ are forced to lie in the interior of $\mathcal{N}(\Sigma, \Lambda)$, even if $\mathcal{J}$ is slightly perturbed. On the other hand, if the strict gap property fails, that would suggest the existence of some \emph{escape directions}, along which $\mathcal{J}$ is almost constant or could further decrease, so the vacuum $\Sigma $ would not be very stable.
		\end{rmk}

		The following is an attempt to formulate the notion of scale detectors on the boundary of $Sym^2_{\geq 0}$.

		\begin{Def}
			Given a positive semi-definite $\Sigma$ with null subspace $W$, we denote $W'$ as the orthogonal complement of $W$ with respect to the ambient inner product $|\cdot|_V^2$ on $V$. Given parameters $\Lambda>1$ and $a>0$, the $(\Lambda, a)$-\emph{neighbourhood} of $\Sigma$ is the subset  $\mathcal{N}(\Sigma, \Lambda, a)$, consisting of all $\Sigma'\in Sym^2_{\geq 0}$ satisfying the following conditions:
			\begin{itemize}
				\item The restriction of $\Sigma'$ to $W$ is bounded above by $a |\cdot|_V^2$.
				
				\item  The restriction of $\Sigma'$ to $W'$ satisfies the uniform equivalence condition
				\[
				\Lambda^{-1} \Sigma|_{W'} \leq \Sigma'|_{W'}\leq \Lambda \Sigma|_{W'} .
				\]
			\end{itemize}
		\end{Def}

		\begin{Def}
			Let $\epsilon>0$ and $\Lambda>1, a>0$ be given parameters.
			A boundary vacuum $\Sigma\in Sym^2_{\geq 0}\setminus Sym^2_+$ is called an $(\epsilon,\Lambda, a)$-\emph{scale detector}, if it has the local minimizing property
			\[
			\mathcal{J}(\Sigma', \lambda)\geq \mathcal{J}(\Sigma, \lambda),\quad \forall \Sigma' \in \mathcal{N}(\Sigma, \Lambda, a), 
			\]
			and the \emph{strict gap} property
			\[
			\mathcal{J}(\Sigma', \lambda)\geq \mathcal{J}(\Sigma, \lambda)+\epsilon \E|Y|^2 ,\quad \forall \Sigma' \in \partial \mathcal{N}(\Sigma, \Lambda, a).
			\]
			
		\end{Def}

		Given a boundary scale detector, we can regard $W'\simeq V/W$-component of $X$ as the primary variables, and the $W$-component as the secondary variables. The problem of scale detection is mixed with the problem of variable selection.

		\subsection{Where can the scale detectors appear?}

		We shall focus on the scale detectors in the interior of $Sym^2_+$, since the boundary vacua reduce to a lower dimensional problem, as studied in Section \ref{sect:boundaryofSym}. We showcase some sample results to illustrate conceptual features.

		\subsubsection{A priori restrictions}

		First, we show that interior $(\epsilon,\Lambda)$-scale detectors \emph{cannot be too close to the boundary} of $Sym^2_{\geq 0}$.

		\begin{lem}\label{lem:scaledetectorlowerbound}
			Suppose $K(x)=\mathcal{K}(|x|_V^2)$ where $\mathcal{K}'$ is a bounded continuous function, and 
			\[
			\beta:= \sup_{ \omega\in V^*: |\omega|_V=1  } \E[ |\langle X, \omega\rangle |^4 ]^{1/2} <+\infty.
			\]
			Then any $(\epsilon, \Lambda)$-scale detector $\Sigma$ must be bounded below by $\frac{\epsilon \lambda}{2\beta\norm{\mathcal{K}'}_{C^0}} |\cdot|_V^2$.
		\end{lem}

		\begin{rmk}
			Here the finiteness of $\beta$ is of course equivalent to $\E[|X|^4]<+\infty$, but the constant $\beta$ is far more efficient if the dimension is high.
		\end{rmk}

		\begin{proof}
			For any $\Sigma$, we let $\Sigma'=\Sigma+ t\omega\otimes \omega$ for some $\omega\in V^*$ with $|\omega|_V=1$, then
			\[
			|\mathcal{K}(|X'-X|_\Sigma^2)- \mathcal{K}(|X'-X|_{\Sigma'}^2) |\leq \norm{\mathcal{K}'}_{C^0} |t| |\langle \omega, X-X'\rangle|^2.
			\]
			By the quantitative continuity estimate of Proposition  \ref{lem:continuityJ},
			\[
			\begin{split}
				|\mathcal{J}(\Sigma;\lambda)-\mathcal{J}(\Sigma';\lambda)|\leq 
				& \frac{1}{2\lambda}\E[|Y|^2 ] \E[ |\mathcal{K}(|X'-X|_\Sigma^2)- \mathcal{K}(|X'-X|_{\Sigma'}^2) |^2 ]^{1/2}.
				\\
				\leq & \frac{1}{2\lambda}\E[|Y|^2 ]|t| \norm{\mathcal{K}'}_{C^0} \E[|\langle X'-X, \omega\rangle|^4]^{1/2}
			\end{split}
			\]
			Using the Minkowski inequality
			\[
			\E[|\langle X'-X, \omega\rangle|^4]^{1/4} \leq \E[\langle X', \omega\rangle|^4]^{1/4}  + \E[\langle X, \omega\rangle|^4]^{1/4} \leq 2\beta^{1/2}. 
			\]
			Consequently,
			\begin{equation}\label{eqn:Jcontinuity3}
				|\mathcal{J}(\Sigma;\lambda)-\mathcal{J}(\Sigma+ t\omega\otimes \omega,\lambda)|\leq \frac{2\beta}{\lambda}|t|\norm{\mathcal{K}'}_{C^0} \E[|Y|^2].
			\end{equation}

			We now suppose $\Sigma$ is an $(\epsilon,\Lambda)$-scale detector in $Sym^2_+$. 
			We find  an eigenbasis $\{e_i\}$ of $\Sigma$ with respect to the ambient inner product $|\cdot|_V^2$, with eigenvalues $0<\lambda_1\leq \ldots \leq \lambda_d$. 
			By the strict $\epsilon$-gap property, if we choose 
			\[
			\Sigma'=\Sigma+t\omega\otimes \omega= \Lambda^{-1} \lambda_1 e_1^*\otimes e_1^*+ \sum_2^d \lambda_i e_i^*\otimes e_i^*,\quad t= -(1-\Lambda^{-1})\lambda_1, \quad \omega=e_1^*,
			\]
			then
			$
			\mathcal{J}(\Sigma';\lambda)- \mathcal{J}(\Sigma;\lambda) \geq \epsilon \E[|Y|^2]. 
			$
			Thus $\frac{2\beta}{\lambda}|t|\norm{\mathcal{K}'}_{C^0} \geq \epsilon$, whence
			\[
			\lambda_1 \geq |t| \geq   \frac{\epsilon \lambda}{2\beta\norm{\mathcal{K}'}_{C^0}}.
			\]
			Since $\lambda_1$ is the smallest eigenvalue, this implies $\Sigma \geq \frac{\epsilon \lambda}{2\beta\norm{\mathcal{K}'}_{C^0}} |\cdot|_V^2$.
		\end{proof}

		\begin{rmk}
			This argument of course does not rule out scale detectors on the boundary. For instance, in dimension one, to test scale detection at an interior point $\Sigma\in Sym^2_+(\R)=\R_+$, we only need to consider the value of $\mathcal{J}$ on the dyadic annulus $(\Lambda^{-1} \Sigma, \Lambda \Sigma)$, while at the boundary point $\Sigma=0$, we would need the value of $\mathcal{J}$ on some ball in $Sym^2_{\geq 0} (\R)$ around the origin. 
		\end{rmk}

		Next we prove that under \emph{quantitative bounds on the probability density} of $X$, then the scale detectors in $Sym^2_{\geq 0}$ (including the boundary scale detectors) remain in a bounded region.

		\begin{lem}\label{lem:scaledetectorupperbound}
			Suppose $K(x)= \mathcal{K}(|x|_V^2)$ 
			is a completely monotone kernel (\ref{eqn:completelymonotonekernel}), such that
			\[
			C_\mathcal{K}= \iint (\pi(t+s)^{-1})^{1/2} d\nu(t)d\nu(s) <+\infty.
			\]
			Suppose also that along any unit direction $\xi\in V$, the marginal distribution of $X$ projected to $\R\xi$ has probability density $L^\infty$-bound less than $p_0$. Then any $(\epsilon,\Lambda)$-scale detector $\Sigma$ in $Sym^2_{\geq 0}$ is bounded above by
			$
			\Sigma \leq (C_\mathcal{K}  p_0)^2 (\lambda\epsilon)^{-4} |\cdot |_V^2.
			$
			
		\end{lem}

		\begin{proof}
			By Corollary \ref{Cor:quantitativenonfitting},
			\[
			\begin{split}
				\mathcal{J}(\Sigma;\lambda)\geq  & \frac{1}{2} \E[|Y|^2]-\lambda^{-1} \E[ Y\bar{Y}' \mathcal{K}( |X-X'|_\Sigma^2)  ]
				\\
				\geq & \frac{1}{2} \E[|Y|^2]-\lambda^{-1} \E[ |Y|^2] \E[ \mathcal{K}^2( |X-X'|_\Sigma^2)  ]^{1/2}. 
			\end{split}
			\]
			We take an eigenbasis $\{e_i\}$ of $\Sigma$ with respect to $|\cdot|_V^2$, with eigenvalues $0\leq \lambda_1\leq \ldots \leq \lambda_d$. 
			By Lemma \ref{lem:Kdoubleintegralbound}, choosing $\xi=e_d$ the highest eigenvector, $W'= \R\xi$, and $W$ its orthogonal complement, we deduce
			\[
			\E[ \mathcal{K}^2(|X-X'|_{\Sigma}^2   ] \leq    \frac{p_0}{  \sqrt{\lambda_d} } \iint (\pi (t+s)^{-1})^{1/2}  d\nu(t) d\nu(s). 
			\] Thus
			\begin{equation}\label{eqn:scaledetectorupperbound} 
				\begin{split}
					\mathcal{J}(\Sigma;\lambda)
					& \geq \frac{1}{2} \E[|Y|^2]-\lambda^{-1} \E[ |Y|^2] (\frac{C_\mathcal{K} p_0}{  \sqrt{\lambda_d} })^{1/2} .
				\end{split}
			\end{equation}
			If $\Sigma$ is a scale detector, then $\mathcal{J}(\Sigma;\lambda) \leq  \frac{1}{2} \E[|Y|^2]-\epsilon \E[|Y|^2]$, so
			$
			\lambda^{-1} (\frac{C_{\mathcal{K}}p_0}{  \sqrt{\lambda_d} } )^{1/2}  \geq \epsilon,
			$
			whence $\lambda_d \leq (C_{\mathcal{K}} p_0)^2 (\lambda\epsilon)^{-4} $, which implies the upper bound on $\Sigma$. 
		\end{proof}
		
		\begin{rmk}
			(Improved estimates) In the above proof,
			we have only used the $\dim W'=1$ case of Proposition  \ref{lem:Kdoubleintegralbound}. If $\Sigma_{V/W}$ is large in $\dim {W'}>1$  directions, then the suppression factor  $\det(\Sigma_{V/W})^{-1/2}$ from
			Proposition  \ref{lem:Kdoubleintegralbound}  gives much stronger estimates than  $\lambda_d^{-1/2}$. 
		\end{rmk}

		\begin{rmk}
			The Lemmas are applicable to the Gaussian kernels, or the Sobolev kernels with $\gamma>1$ (See Example \ref{eg:Gaussiankernel}, \ref{eg:Sobolevkernel}).
			The upshot is that if the conditions of both Lemmas are satisfied, then any interior $(\epsilon,\Lambda)$-scale detector must be uniformly equivalent to $|\cdot|_V^2$ with estimable constants. The intuitive explanation is that the condition 
			$\E[|X|^4]<+\infty$ rules out divergence of $X$ to infinity, while
			the quantitative $L^\infty$-bound on the probability density rules out concentration of $X$, so the random variable $X$ has only one inherent scale parameter, equivalent to the inverse covariance matrix of $X$ up to controlled uniform constants.

			On the other hand, if we drop the $L^\infty$-bound on the probability density, then $X$ can have several far separated scale parameters, and scale detectors can be far from $|\cdot|_V^2$. 
			
		\end{rmk}

		\subsubsection{Global minimum and scale detector}

		We now find a sufficient condition to find an interior scale detector. We shall need the following:

		\begin{itemize}
			\item Essential Dependence Assumption: 
			\[
			\inf_{Sym^2_+} \mathcal{J}(\Sigma;\lambda)  \leq \inf_{Sym^2_{\geq 0} \setminus  Sym^2_+ }  \mathcal{J}(\Sigma;\lambda)- 2\epsilon \E[|Y|^2]. 
			\]
		\end{itemize}
		The intuitive meaning is that $Y$ should depend on all the degrees of freedom in $X$, and missing any degree of freedom would result in worse predictive power, as we now explain. We suppose $Y$ is a function of $X$, so $\text{Var}(Y|X)=0$, but for any decomposition $V=W\oplus W'$, $Y$ cannot be expressed as a function of $X_{W'}$ alone, with failure measured by
		\[
		\E[ \text{Var}(Y|X_{W'})  ] >4 \epsilon \E[|Y|^2]. 
		\]
		Suppose $\Sigma$ is positive semi-definite with null subspace $W$, and $X_{W'}$ is the marginal of $X$ on $V/W$. Then by the trivial lower bound in Lemma \ref{lem:triviallowerbound},
		\[
		\mathcal{J}(\Sigma;\lambda)\geq \frac{1}{2}\E[ \text{Var}(Y|X_{W'})  ] >2 \epsilon \E[|Y|^2]. 
		\]
		When $\lambda$ is small enough, $Y$ should be well approximated by some function in $\mathcal{H}$, so $\inf_{\Sigma\in Sym^2_+} \mathcal{J}(\Sigma;\lambda)$ is close to zero, and the essential dependence assumption is satisfied.

		\begin{thm}\label{thm:scaledetectorglobalmin}
			We suppose the essential dependence assumption, and the hypotheses of Lemma \ref{lem:scaledetectorlowerbound}, \ref{lem:scaledetectorupperbound}. 
			Let
			\[
			\Lambda=  \max(  \frac{\beta\norm{\mathcal{K}'}_{C^0}}{2\epsilon \lambda}, (C_\mathcal{K}  p_0)^2 (\lambda\epsilon)^{-4}).
			\]
			Then any global minimizing vacuum $\Sigma_0$ lies within the
			neighbourhood $\mathcal{N}(|\cdot|_V^2, \Lambda)$, and such $\Sigma_0$ is an $(\epsilon,\Lambda^2)$-scale detector.

		\end{thm}

		\begin{proof}
			Since $X$ is a continuous variable, we know $\lim_{\Sigma \to \infty} \mathcal{J}(\Sigma;\lambda)= \frac{1}{2} \E[|Y|^2]$, so the global minimum is achieved in $Sym^2_{\geq 0}$. By the essential dependence assumption, any global minimizer $\Sigma_0$ must be inside the interior $Sym^2_+$. The main idea is to prove that $\mathcal{J}$ has a quantitative gap from its minimum, for any $\Sigma$ close to the boundary or close to infinity.

			For any $\Sigma\in Sym^2_+$, 
			we find  an eigenbasis $\{e_i\}$ of $\Sigma$ with respect to the ambient inner product $|\cdot|_V^2$, with eigenvalues $0<\lambda_1\leq \ldots \leq \lambda_d$. Choosing
			\[
			t=-\lambda_1, \quad \omega= e_1^*, \quad \Sigma'= \sum_2^d \lambda_i e_i^*\otimes e_i^*,
			\]
			then upon applying (\ref{eqn:Jcontinuity3}), and the essential dependence assumption, 
			\[
			\begin{split}
				\mathcal{J}(\Sigma) & \geq \mathcal{J}(\Sigma')- \lambda_1 \frac{2\beta}{\lambda}\norm{\mathcal{K}'}_{C^0} \E[|Y|^2] 
				\\
				& \geq \inf_{Sym^2_{\geq 0} \setminus  Sym^2_+ }  \mathcal{J}(\Sigma;\lambda) - \lambda_1 \frac{2\beta}{\lambda}\norm{\mathcal{K}'}_{C^0} \E[|Y|^2] 
				\\
				& \geq \mathcal{J}(\Sigma_0) + 2\epsilon \E[|Y|^2] -  \lambda_1 \frac{2\beta}{\lambda}\norm{\mathcal{K}'}_{C^0} \E[|Y|^2] ,
			\end{split}
			\]
			Thus  as long as 
			$
			\lambda_1 \leq \Lambda^{-1}\leq \frac{\lambda\epsilon}{  2\beta  \norm{\mathcal{K}'}_{C^0} } 
			$, we have $\mathcal{J}(\Sigma) \geq  \mathcal{J}(\Sigma_0) + \epsilon \E[|Y|^2] $. On the other hand, by (\ref{eqn:scaledetectorupperbound}) and the essential dependence assumption, as long as $\lambda_d \geq \Lambda\geq  (C_\mathcal{K}  p_0)^2 (\lambda\epsilon)^{-4}$, we also have
			\[
			\mathcal{J}(\Sigma;\lambda)
			\geq \frac{1}{2} \E[|Y|^2]-\epsilon \E[ |Y|^2] \geq \mathcal{J}(\Sigma_0)+ \epsilon \E[ |Y|^2].
			\]
			Consequently, at $\Sigma=\Sigma_0$, we have $\lambda_1\geq \Lambda^{-1}$ and $\lambda_d\leq \Lambda$, so $\Sigma\in \mathcal{N}(|\cdot|_V^2, \Lambda)$.

			Moreover, any $\Sigma\in \partial \mathcal{N}(\Sigma_0, \Lambda^2)$ lies outside of $\mathcal{N}(|\cdot|_V^2, \Lambda)$, so either $\lambda_1\leq \Lambda^{-1}$ or $\lambda_d \geq \Lambda$, hence the strict gap property for the $(\epsilon, \Lambda^2)$-scale detector $\Sigma_0$ follows.
		\end{proof}

		\begin{rmk}
			We \emph{do not} claim that the global minimizing vacuum is unique, nor that the interior scale detectors are unique. On the other hand, in the context of the above theorem all these are uniformly equivalent to each other with controllable constants, so from the viewpoint of the scale detection philosophy, their roles are more or less interchangeable. 
		\end{rmk}
		
		\blue{
			\begin{rmk}
				The localization estimates for scale detectors in Lemmas~\ref{lem:scaledetectorlowerbound}--\ref{lem:scaledetectorupperbound} and Theorem~\ref{thm:scaledetectorglobalmin} deteriorate as $\lambda \to 0^+$. They do not preclude whether the scale detectors remain bounded, approach the boundary, or escape to infinity as $\lambda \to 0^+$. Understanding the vacua's limiting behavior as $\lambda \to 0^+$ would require a separate analysis and lies beyond the scope of the paper.
			\end{rmk}
		}

		\subsection{Multiple vacua in dimension one}

		A significant phenomenon is that when $Y$ is the superposition of several functions of $X$, with very different sets of inherent scale parameters, one can sometimes observe \emph{multiple vacua}, each sensitive to one set of scale parameters. As a corollary, the vacua that are \emph{not necessarily the global minimizer} of $\mathcal{J}(\Sigma;\lambda)$, can carry important information which the global minimizer does not detect. 
		
		\red{
			Our one-dimensional examples should be viewed as a proof of concept, which illustrates how the vacua detects the scale parameters. One can hope to explore the higher dimensional examples using the tools developed in this paper, such as the dimensional reduction  near the boundary of $Sym^2_{\geq 0}$, and the asymptotic behavior for large $\Sigma$, but the full picture remains an interesting question.
		}
		
		\subsubsection{Setup}

		We work in a special case of Section \ref{sect:clusterinteractions}; a typical example to keep in mind is Example \ref{eg:multiscale}.
		Let $X_0$ be a discrete random variable valued in  $ \{ 1,2,\ldots m\}$, which is not directly accessible to the observer, and let $p_i= \mathbb{P}(X_0=i)>0$. Let $Z_1,\ldots Z_m$ be continuous random variables on $\R$, such that $\E[|Z_i|^2 ]$ is finite, $\E[Z_i]=0$, and all the $Z_i$ and $X_0$ are independent. Let $\sigma_1,\ldots \sigma_m\in \R$,  and $a_1,\ldots a_m\in \R$, such that 
		\[
		X= \sum_1^m \chi_i(a_i+\sigma_i Z_i),\quad \chi_i=  \1_{X_0=i}.
		\]
		Thus conditional on $X_0=i$, then $X$ has the same law as $a_i+\sigma_i Z_i$. One should view $\sigma_i$ as specifying the scale parameters of the clusters, the $a_i$ as specifying the cluster centres, while $Z_i$ gives the rescaled version of the conditional distribution of $X$ (e.g. $Z_i$ could be a standard Gaussian). Let 
		\[
		Y= \sum_1^m  g_i( \frac{X- a_i}{\sigma_i} ),
		\]
		such that $\E[|Y|^2]<+\infty$, and $\E[\chi_i |Y|^2]>0$. \red{Here $g_i$ are fixed functions independent of the choice of $a_i, \sigma_i$, such that $Y$ is a nonzero function of $X$ conditional on $X_0=i$.}

		We want to study the multi-scale phenomenon in these examples, so 
		we view $p_i, Z_i, g_i$ as \emph{fixed} choices, while $\sigma_i$ are scale parameters that we shall vary. We shall focus on the case
		\[
		0<\sigma_1\ll \sigma_2\ll \ldots \\ \sigma_m ,\quad \min_{i\neq j} |a_i-a_j| \gtrsim \sigma_m.
		\]
		Qualitatively, we wish to understand the vacua of $\mathcal{J}$ as $\sigma_i/\sigma_{i+1}\to 0$. \red{In the following arguments, the constants are allowed to depend on $p_i, Z_i, g_i$, and we will use $o(1)$ to denote quantities that converges to zero in the limit $\sigma_i/\sigma_{i+1}\to 0$. The explicit dependence on $\sigma_i/\sigma_{i+1}$ is computable in principle, but we will mostly suppress this for simplicity.}

		We shall make the following additional assumption:
		
		\begin{itemize}
			\item  (Average value in each cluster)   $\E[Y|X_0=i]=0$ for each value of $X_0$.

			\begin{rmk}
				Here the conditional expectation of $Y$ can be viewed as the zero dimensional reduction of the subproblem of learning $Y$ conditional on $X_0=i$. If the average zero condition is not satisfied on the nose, then one should first try to learn these average numbers and then subtract them off, as we remarked in Section \ref{Discussion:boundary}, item 4.
			\end{rmk}

		\end{itemize}

		As usual, we write the kernel function as $K(x)= \mathcal{K}(|x|^2)$.

		\subsubsection{Multiple vacua}

		Under the above setting, we shall prove the existence of many vacua which reflect the scale parameters in the distribution of $X$.

		\begin{thm}\label{thm:multiplevacua}
			(Multiple vacua existence)  
			There are a small constant $\epsilon$ and a large constant $\Lambda$, such that if $\sigma_i/\sigma_{i+1}$ is sufficiently \red{small} for all $i$, and $\min_{i\neq j}|a_i-a_j|\geq C^{-1} \sigma_m$, then the following statements hold:
			\begin{enumerate}

				\item Any $(\epsilon, \Lambda)$-scale detector $\Sigma$ satisfies 
				\[
				\Lambda^{-1} \sigma_i^{-2}<	\Sigma= U^2 < \Lambda\sigma_i^{-2}    
				\]
				for some $i\in \{1,\ldots m\}$. 
				
				\item  For each $i=1,\ldots m$, there is an $(\epsilon,\Lambda^2)$-scale detector $\Sigma$ with
				\[
				\Lambda^{-1} \sigma_i^{-2}<	\Sigma= U^2 < \Lambda\sigma_i^{-2} .
				\]
			\end{enumerate}

		\end{thm}

		\begin{proof}
			We shall focus on the cases $ \sigma_{i+1}^{-1}\sigma_i^{-1}  \leq \Sigma \leq \sigma_i^{-1}\sigma_{i-1}^{-1}$ for some $2\leq i\leq m-1$; the other cases $\Sigma\leq \sigma_m^{-1}\sigma_{m-1}^{-1}$ and $\Sigma\geq \sigma_1^{-1} \sigma_2^{-1} $ can be treated similarly.  The main difficulty  is that we are dealing with a \emph{multi-parameter family} of random variables $(X,Y)$ depending on $\sigma_1,\ldots \sigma_m$.  Our strategy is to simplify the problem by a number of reductions, to arrive at a problem with no extra scale parameters.

			\begin{itemize}
				\item Step 1: Replacing the small scale clusters $X_0<i$ by atoms.
			\end{itemize}
			
			First, conditional on \blue{$X_0=j$} for $j\leq i-1$, we replace $\chi_j Y$ with its conditional mean (which by assumption is zero). The deviation between the minimizer of the original loss function $I(f,U,\lambda)$ versus the new loss function, is the quantity $f_U^1$ in Lemma \ref{lem:K4terms}, which is bounded by
			\[	\begin{split}
				& \lambda\norm{f_U^1}_{\mathcal{H} }^2 + \E[ |f_U^1(UX)|^2]  \leq \frac{1}{\lambda} \E[|Y^1|^2] \times
				\\
				&
				\E[|	\mathcal{K}(|X-X'|_\Sigma^2)-  \mathcal{K}(|\tilde{X}-X'|_\Sigma^2)
				- \mathcal{K}(|X-\tilde{X'}|_\Sigma^2)
				+ \mathcal{K}(|\tilde{X}-\tilde{X'}|_\Sigma^2)|^2   ]^{1/2},
			\end{split}
			\]
			where $\tilde{X}=X$ conditional on $X_0\in \{ i,\ldots m\}$, and $\tilde{X}= \E[X|X_0=j]= a_j$ conditional on $X_0=j$ for $j<i$.  We claim that as $\sigma_{i-1} U\to 0$, the quantity 
			\[
			\E[ | \mathcal{K}(|X-X'|_\Sigma^2)-  \mathcal{K}(|\tilde{X}-X'|_\Sigma^2)
			- \mathcal{K}(|X-\tilde{X'}|_\Sigma^2)
			+ \mathcal{K}(|\tilde{X}-\tilde{X'}|_\Sigma^2)|^2   ]
			\]
			converges to zero.

			To see the claim, recall that  $\mathcal{K}(r)\to 0$ as $r\to +\infty$ by Riemann-Lebesgue, so we only need to consider the case where the distances $|X-X'|_\Sigma$ etc are all bounded.  Now conditional on
			$X_0=j$ for some $j<i$, we have
			\[
			|X-\tilde{X}|_\Sigma = \sigma_j U  |Z_j- \E[Z_j]|\leq  \sigma_{i-1} U  |Z_j- \E[Z_j]| ,
			\]
			so $|X-X'|_\Sigma^2- |\tilde{X}- X'|_\Sigma^2 $ is small when $\sigma_{i-1}U \ll 1$. The claim then follows from the uniform modulus of continuity on the kernel function.

			The upshot is that up to $o(1)$ error in the $\sigma_k/\sigma_{k+1}\to 0$ limit, we can compute $\mathcal{J}$ by the simpler problem where $\chi_j Y$ is replaced by its conditional expectation value zero, for all $j<i$.  Likewise, by the arguments of Section \ref{sect:simplifydistribution}, for the purpose of finding $\mathcal{J}(\Sigma;\lambda)$ and the minimizing function $f_U$ of $I(f,U,\lambda)$ up to $o(1)$ error, we may replace the conditional distribution of $\chi_j X$ by the atomic measure at the average value of $\chi_j X$, which is $a_j$.

			\begin{itemize}
				\item Step 2: Decouple the bigger scale clusters.
			\end{itemize}

			In Section \ref{sect:noninteraction} we analyzed the deviation between the minimizing function $f_U$ versus the sum $\tilde{f}_U=\sum f_k$ of the minimizing functions $f_k$ of the decoupled problems associated to individual clusters.

			By Corollary \ref{cor:noninteraction2}, the deviation error is bounded by
			\begin{equation}
				\lambda \norm{f_U- \tilde{f}_U}_{\mathcal{H}}^2+ \E[ |(f_U-\tilde{f}_U)(UX)|^2 ] \leq  \frac{1}{\lambda^2}  \E[ |Y|^2] (\sum_{1\leq l\neq k\leq m} \E[  \chi_l \chi_k'|\mathcal{K}( |X-X'|_\Sigma^2)|^2  ]).
			\end{equation}
			By Theorem \ref{thm:noninteraction3},
			\begin{equation}
				\begin{cases}
					\mathcal{J}(\Sigma;\lambda)- \sum_1^m J_j \leq \frac{1}{\lambda^2}  \E[ |Y|^2] (\sum_{1\leq l\neq k\leq m} \E[  \chi_l \chi_k'|\mathcal{K}( |X-X'|_\Sigma^2)|^2  ]).
					\\
					\mathcal{J}(\Sigma;\lambda)   - \sum_1^m J_i \geq  -\frac{2}{\lambda^2}  \E[ |Y|^2] (\sum_{1\leq l\neq k\leq m} \E[  \chi_l \chi_k'|\mathcal{K}( |X-X'|_\Sigma^2)|^2  ]).
				\end{cases}	
			\end{equation}
			Here $J_i$ stands for the minimum value of the decoupled problems.
			We now analyze the key quantity $\sum_{1\leq l\neq k\leq m} \E[  \chi_l'\chi_k |K(UX, UX')|^2   ]$.

			Conditional on $X_0=l$ and $X_0'=k$, then $X$ has average $a_l$ with variance $\sigma_l^2 \E[|Z_l|^2]$, while $X'$ has average $a_k$ with variance $\sigma_k^2 \E[|Z_k|^2]$. By assumption
			\[
			|a_l-a_k|\gtrsim \sigma_m\gg \sigma_{m-1}\gg \ldots \gg  \sigma_1.
			\]
			\red{
				Let $R\gg 1$ be any given parameter.
				In the $\sigma_k/\sigma_{k+1}\to 0$ limit, using the assumption that $\Sigma \geq (\sigma_i\sigma_{i+1})^{-1}$, outside the exceptional event
				\[
				|X-X'|<R^{-1} \min\{|a_l-a_k|, \sigma_m\},
				\]
				the separation distance $|X-X'|_\Sigma^2$ is bounded below by 
				\[
				|X-X'|_\Sigma^2 \geq R^{-2} \min\{|a_l-a_k|, \sigma_m\}^2 (\sigma_i\sigma_{i+1})^{-1} \geq C^{-1} R^{-2}\sigma_m^2 (\sigma_i\sigma_{i+1})^{-1}
				\]
				which tends to infinity for any fixed $R$. We now show that the exceptional event has arbitrarily small probability in the limit when $R$ is sufficiently large.  When $l,k\leq m-1$,
				\[
				\begin{split}
					& \mathbb{P}( |X-X'| < R^{-1}|a_l-a_k|  \mid X_0 = l, X_0' = k)
					\\
					= &
					\mathbb{P}( |a_l+\sigma_l Z_l - a_k-\sigma_k Z_k| < R^{-1}|a_l-a_k| )
					\\
					\leq & \mathbb{P}( |\sigma_l Z_l | \geq  \frac{1}{4}|a_l-a_k| )+ \mathbb{P}( |\sigma_k Z_k | \geq  \frac{1}{4}|a_l-a_k| )
					\\
					\leq &    \frac{4\sigma_l}{ |a_l-a_k| }  \E[ |Z_l|^2   ]^{1/2}  +\frac{4\sigma_k}{ |a_l-a_k| }  \E[ |Z_k|^2   ]^{1/2} 
					\\
					\leq &    C(\sigma_l+\sigma_k)\sigma_m^{-1}\to 0,
				\end{split}
				\]
				where we used $|a_l-a_k|\gtrsim \sigma_m$ and $(\sigma_k+\sigma_l)\sigma_m^{-1} = o(1)$, while when $l$ or $k$ is $m$, the probability 
				$\mathbb{P}( |a_l+\sigma_l Z_l - a_k-\sigma_k Z_k| < R^{-1}\sigma_m )$ can be made arbitrarily small by taking $R$ sufficiently large, thanks to the assumption that $Z_m$ is a continuous random variable and hence $\sup_{c \in \mathbb{R}}\mathbb{P}(|Z_m-c|<R^{-1})\to 0$ as $R\to\infty$}. Since the kernel function $\mathcal{K}(r)$ tends to zero as $r\to \infty$, this shows that in the $\sigma_k/\sigma_{k+1}\to 0$ limit,
			\[
			\sum_{1\leq l\neq k\leq m} \E[  \chi_l \chi_k'|\mathcal{K}( |X-X'|_\Sigma^2)|^2  ] \to 0. 
			\]
			The conclusion is that up to $o(1)$ error, we may replace the minimizing function by the sum of the minimizing functions for the decoupled problems.

			Moreover, for each cases $X_0=i+1,\ldots X_0=m$, the norms of the  minimizing functions for the decoupled problems can be bounded by Corollary \ref{Cor:quantitativenonfitting},
			\[
			\begin{cases}
				\lambda\norm{f_k}_{\mathcal{H} }^2 + \E[ |f_k(UX)|^2]  \leq  \frac{1}{\lambda}\E[ \chi_k \chi_k' Y\bar{Y}' K( U(X-X'))  ],
				\\
				\mathcal{J}(\Sigma;\lambda)
				\geq  \frac{1}{2} \E[\chi_k |Y|^2]-(2\lambda)^{-1} \E[ \chi_k \chi_k' Y\bar{Y}' K( U(X-X'))  ].
			\end{cases}
			\]
			By Cauchy-Schwarz, and the fact that conditional on $X_0=k$ then $X= a_k+\sigma_k Z_k$, 
			\[
			\begin{split}
				\E[ \chi_k \chi_k' Y\bar{Y}' K( U(X-X'))  ]
				\leq & \E[\chi_k |Y|^2] \E[  \chi_k \chi_k'\mathcal{K}(|X-X'|_\Sigma^2)   ]^{1/2}
				\\
				= & \E[\chi_k |Y|^2] \E[  \chi_k \chi_k'\mathcal{K}( \sigma_k^2 \Sigma|Z_k-Z_k'|^2)   ]^{1/2}
			\end{split}
			\]
			Since for $k>i$ we have  $\sigma_k^2 \Sigma \geq \sigma_k^2 (\sigma_i\sigma_{i+1})^{-1} \to +\infty$ in the limit, and since $Z_k$ is a continuous variable, 
			\[
			\E[  \chi_k \chi_k'\mathcal{K}( \sigma_k^2 \Sigma|Z_k-Z_k'|^2)   ]\to 0.
			\]
			The upshot is that the contributions from the decoupled problems for $X_0=i+1,\ldots X_0=m$ (the `bigger scale clusters') to the sum has norm $o(1)$ in the limit. \red{It is worth noting that the $o(1)$-estimate here uses only the information on the distribution of $X$ (ultimately due to the separation of the cluster centres in $X$ as $\sigma_k^2\Sigma\to +\infty$), but does not require decay property of the functions $g_i$ that appears in $Y$.}

			On the other hand, conditional on $X_0=1,2,\ldots i-1$, since we have replaced $Y$ by zero, the corresponding decoupled problem becomes trivial.

			\begin{itemize}
				\item  Step 3: Analyzing the simpler kernel ridge regression problem.
			\end{itemize}

			We have reduced to the decoupled minimization problem
			\[
			\frac{1}{2} \E[      \chi_i |Y-f(UX)|^2        ] + \frac{\lambda}{2} \norm{f}_\mathcal{H}^2 = \frac{1}{2} \E[      \chi_i |Y-f(U(a_i+ \sigma_i Z_i))|^2        ] + \frac{\lambda}{2} \norm{f}_\mathcal{H}^2.
			\]
			By the above discussion, the deviation between the minimizer $f_U$ of the original kernel ridge regression problem, versus the minimizer $f_i$ of this decoupled problem, converges to zero in norm.

			By the translation invariance of the RKHS, the minimum value of the functional does not depend on $a_i$. Setting aside $a_i$, and viewing $U\sigma_i$ as the scaling parameter, we have thereby removed the dependence of the distribution $(X,Y)$ on the extra parameters $\sigma_1,\ldots \sigma_m$. We recognize the two limits:
			\begin{enumerate}
				\item  As $U\sigma_i\to +\infty$, namely $\Sigma \sigma_i^2\to +\infty$, the minimizing function $f_i$ satisfies
				\[
				\lambda\norm{f_i}_\mathcal{H}^2+ \E[    \chi_i |f_i(U(a_i+ \sigma_i Z_i))|^2   ]         \to 0.
				\]
				This follows from Corollary \ref{Cor:quantitativenonfitting} and the assumption that $Z_i$ is a continuous random variable.

				\item As $U\sigma_i\to 0$, namely $\Sigma \sigma_i^2\to 0$, then again the minimizing function satisfies
				\[
				\lambda\norm{f_i}_\mathcal{H}^2+ \E[    \chi_i |f_i(U(a_i+ \sigma_i Z_i))|^2   ]         \to 0.
				\]
				To see this, by Corollary \ref{cor:Jminimizercontinuity} and Lemma \ref{lem:Jminimizercontinuity1}, this limit of the minimizing function is given by the minimizing function for $U=0$. But when $U=0$ the minimization depends on $Y$ only though its conditional expectation value, which is zero by assumption.
			\end{enumerate}
			
			\red{
				Moreover, for $U=\sigma_i^{-1}$ (\ie $\Sigma=\sigma_i^{-2})$, the decoupled functional becomes 
				\[
				\frac{1}{2} \E[      \chi_i |Y-f(U a_i+ Z_i))|^2        ] + \frac{\lambda}{2} \norm{f}_\mathcal{H}^2.
				\]
				\blue{By translation invariance, this minimum coincides with the minimum of
					\[
					\frac{1}{2}\E\!\left[\chi_i|Y-f(Z_i)|^2\right]
					+\frac{\lambda}{2}\|f\|_{\mathcal H}^2,
					\]
					which is manifestly independent of the scale parameter $\sigma_i$.} Since $Y$ is a nonzero function of $X$ conditional on $X_0=i$, the minimizer is nontrivial.} Thus there exists some $\epsilon>0$ \red{independent of $\sigma_i$}, such that the min value of this decoupled functional has a quantitative gap from its trivial upper bound:
			\[
			J_i= \blue{\min_{f_i}}\frac{1}{2} \E[      \chi_i |Y-f_i(UX)|^2        ] + \frac{\lambda}{2} \norm{f_i}_\mathcal{H}^2 <\frac{1}{2} \E[      \chi_i |Y|^2        ] -2\epsilon \E[|Y|^2].
			\]

			\begin{itemize}
				\item  Step 4: Quantitative version.
			\end{itemize}

			\red{
				As always we fix $Z_i$, $g_i$, $\lambda$, and the choice of kernel, and assume $\min_{i\neq j} |a_i-a_j|/\sigma_m $ is bounded uniformly from below.
				By Steps 1-2, the norm of the deviation between $f_i$ and the minimizer $f_U$ of the original kernel ridge regression, converges to zero in the limit. As $\sigma_k/\sigma_{k+1}\to 0$, 
				\[
				\mathcal{J}(\Sigma;\lambda)= \sum_1^m J_j +o(1)= J_i+ \sum_{j\neq i} \frac{1}{2} \E[\chi_j |Y|^2] +o(1).
				\]
				Using Step 3, we infer that there exists some $\eta>0$ small enough, such that the following holds when $\sigma_i/\sigma_{i+1}<\eta$ for all $i$, and $ \sigma_{i+1}^{-1}\sigma_i^{-1}  \leq \Sigma \leq \sigma_i^{-1}\sigma_{i-1}^{-1}$,}
			\begin{enumerate}
				\item There is a strict gap 
				$
				\mathcal{J}(\Sigma=\sigma_i^{-2} ;\lambda)\leq \frac{1}{2}\E[|Y|^2] -2\epsilon \E[|Y|^2].
				$
				\item 
				There is some $\Lambda$ large enough depending on $\epsilon$ but not on $\sigma_i$, such that whenever $\Sigma\leq \Lambda^{-1} \sigma_i^{-2}$ or $\Sigma\geq \Lambda \sigma_i^{-2}$, then
				$
				\mathcal{J}(\Sigma;\lambda) \geq \frac{1}{2}\E[|Y|^2] -\epsilon \E[|Y|^2].
				$

			\end{enumerate}

			By the second point, any $(\epsilon, \Lambda)$-scale detector within the range $ \sigma_{i+1}^{-1}\sigma_i^{-1}  \leq \Sigma \leq \sigma_i^{-1}\sigma_{i-1}^{-1}$ must be contained in the smaller range $\Lambda^{-1} \sigma_i^{-2} < \Sigma < \Lambda \sigma_i^{-2} $. Using the first point, the minimizing vacuum of $\mathcal{J}$ within the range $\Lambda^{-1} \sigma_i^{-2} \leq \Sigma \leq \Lambda \sigma_i^{-2} $ must be achieved in the interior, and it is easy to see from Definition \ref{Def:scaledetector} that this minimizing vacuum $\Sigma$ is an $(\epsilon, \Lambda^2)$-scale detector.

			The other two situations $0\leq \Sigma\leq \sigma_m^{-1}\sigma_{m-1}^{-1}$ and $\Sigma\geq \sigma_1^{-1} \sigma_2^{-1} $ can be treated similarly \red{(using Step 1, resp. Step 2)}, and the Theorem is proved.
		\end{proof}


		\newcommand{\rank}{{\rm rank}}

	\appendix	
	\section{Discussions: Representation Learning as a Variational Phenomenon}
	\label{appendix:discussion}
	\blue{In machine learning, representation learning seeks to understand how a learning system can automatically discover meaningful structure in data. This question is often raised in the context of neural networks, whose empirical success is frequently attributed to their ability to learn useful internal representations of the data~\cite{BengioCoVi13}. Yet the nature of such representations and the mechanisms responsible for their emergence remain difficult to formulate and analyze rigorously.  The purpose of this section is not to develop new mathematical results, but rather to reinterpret the variational analysis developed in this paper through a representation-learning lens. To place the present work in this context, we first recall the standard two-layer neural network model, draw connections to the kernel learning problem studied in this paper, and discuss what the mathematical results we established may suggest from a representation-learning perspective.
	}
	
	\blue{
		A standard two-layer neural network model takes the form
		\begin{equation*}
			x\mapsto \sum_{1}^m a_i \sigma(\langle w_i, x \rangle),~~~a_i \in \C, w_i \in V^*
		\end{equation*}
		where both the first-layer weights $(w_i)_1^m$ and the second-layer weights $(a_i)_1^m$ are learned from the data. Here $\sigma:\mathbb{R}\to\mathbb{R}$ is a prescribed nonlinear activation function, say ReLU ($\sigma(x) = \max(x, 0)$). A central question in the theory of neural networks, and one of the motivating questions behind representation learning, is whether the learned first-layer weights $(w_i)_1^m$ encode meaningful `structural information' about the underlying data distribution. To date, even formulating this question precisely remains subtle.  Rigorous evidence has recently been obtained that specific training dynamics can drive learned representations toward the underlying low-dimensional structure in certain highly structured multi-index models~\cite{BaErdogduSuzukiWangWuYang2022,BerthierMontanariZhou2024,MontanariWa26}. Nevertheless, these works focus on particular \emph{optimization dynamics}, and a general understanding of which representations are intrinsically preferred by the corresponding population \emph{variational problem} remains elusive. The difficulty is compounded by the fact that the population variational problem associated with neural network models is typically highly nonlinear, overparameterized, and nonconvex, making its landscape difficult to characterize directly.
	}
	
	\blue{
		The kernel learning problem studied in this paper may be viewed as a tractable analogue of a two-layer neural network. Specifically, the predictor takes the form
		\begin{equation*}
			x \mapsto f(Ux),~~~f \in \mathcal{H}, U \in \End(V)
		\end{equation*}
		where both $f \in \mathcal{H}$ and $U \in \End(V)$ are learned from the data. Here $U$ plays a role analogous to the collection of learned first-layer weights $(w_i)_1^m$ in a two-layer neural network, while $f$ plays the role of a learned second layer. Similar to neural networks, the variational problem associated with the kernel learning model $f(Ux)$ studied in this paper remains nonlinear, overparameterized, and nonconvex. Nevertheless, the two models are mathematically quite different. A principal advantage of the $f(Ux)$ model is that, despite these challenges, its population objective admits a rigorous variational analysis. This makes it possible to study representation-learning phenomena directly through the geometry of the population landscape, rather than through the behavior of a particular optimization algorithm.
	}
	
	\blue{
		From this perspective, the present work provides a concrete example in which favorable local minimizers of a two-layer model induce representations encoding meaningful `structural information' of the data, including intrinsic scales, predictive low-dimensional structure, and clustering structure in the underlying data distribution---phenomena that are rigorously characterized in the present paper. These phenomena arise as properties of the population objective itself. In this sense, the analysis addresses a more foundational question: which representations are intrinsically favored by the learning objective, rather than by a particular optimization algorithm. This contrasts with much of the existing neural-network literature, e.g.,~\cite{JacotGaHo18, MeiMoNg18, BaErdogduSuzukiWangWuYang2022,BerthierMontanariZhou2024,MontanariWa26}, where representation-learning phenomena are often analyzed through the behavior of specific algorithmic schemes. While we do not claim analogous results for neural networks, these findings suggest that certain forms of representation learning may arise from the geometry of the population landscape itself, and suggest investigating whether similar mechanisms arise in neural networks.
	}
	
	\blue{It is important to emphasize that the present paper focuses exclusively on the static, variational properties of the kernel learning problem. The dynamical aspects of the model, including the behavior of optimization algorithms and the evolution of learned representations during training, are not addressed here. These questions are investigated from a gradient-flow perspective in the companion paper \emph{Gradient Flow in the Kernel Learning Problem}~\cite{RuanLi}. Readers interested in the broader connection between the present variational analysis and representation learning in two-layer neural networks are strongly encouraged to refer to the postscript of that work, where these connections and their implications for optimization dynamics are discussed in considerably greater detail.
	}
	
	\vspace{.5em}
	
	\begin{Acknowledgement}
		Yang Li thanks C. Letrouit, B. Geshkovski and E. Cornacchia for discussions related to machine learning. Feng Ruan thanks K. Liu for discussions related to kernel methods and their relevance to contemporary machine learning. This work was partially supported by the National Science Foundation under Grant DMS-2540678. We thank the referees for careful reading and helpful suggestions.
	\end{Acknowledgement}

\end{document}